\documentclass{article}

\usepackage[main, final]{neurips_2026}

\usepackage[utf8]{inputenc}
\usepackage[T1]{fontenc}
\usepackage{amsmath,amssymb,amsthm}
\usepackage{graphicx}
\usepackage{booktabs}
\usepackage{array}
\usepackage{multirow}
\usepackage{hyperref}
\usepackage{url}
\usepackage{xcolor}
\usepackage{microtype}
\usepackage{nicefrac}
\usepackage[capitalize]{cleveref}         % must come after hyperref
\hypersetup{pdftitle={A Communication-Theoretic Framework for LLM Agents: Cost-Aware Adaptive Reliability},
  pdfauthor={Hamed Omidvar, Vahideh Akhlaghi}}
\crefname{figure}{Figure}{Figures}
\Crefname{figure}{Figure}{Figures}
\crefname{table}{Table}{Tables}
\Crefname{table}{Table}{Tables}
\crefname{section}{Sec.}{Secs.}
\Crefname{section}{Sec.}{Secs.}
\crefname{equation}{Eq.}{Eqs.}
\Crefname{equation}{Eq.}{Eqs.}
\crefname{proposition}{Prop.}{Props.}
\Crefname{proposition}{Prop.}{Props.}
\crefname{remark}{Remark}{Remarks}
\Crefname{remark}{Remark}{Remarks}
\usepackage{algorithm}
\usepackage{algpseudocode}
\usepackage{subcaption}
\usepackage{enumitem}
\usepackage{bm}
\usepackage{dsfont} 
\newtheorem{proposition}{Proposition}

\newtheorem{lemma}{Lemma}

\theoremstyle{definition} 

\newtheorem{definition}{Definition} %[section]
\newtheorem{remark}{Remark} %[section]
\newcommand{\agentcodec}{\textsc{AgentCodec}}
\newcommand{\E}{\mathbb{E}}

\newcommand{\prob}{\mathbb{P}}
\newcommand{\calA}{\mathcal{A}}

\newcommand{\calT}{\mathcal{T}}

\newcommand{\calS}{\mathcal{S}}
\newcommand{\calC}{\mathcal{C}}

\DeclareMathOperator*{\argmax}{arg\,max}

\newcommand{\qualityimp}{6.8}
\newcommand{\costimp}{55.7}

\title{%
A Communication-Theoretic Framework for LLM Agents: Cost-Aware Adaptive Reliability
}

\author{\vspace{2mm}
  Hamed Omidvar, Vahideh Akhlaghi \\ 
  INTELLERCE LLC \\
  San Jose, CA, USA \\
  \texttt{\{hamed.omidvar, vahideh.akhlaghi\}@intellerce.com} 
}

\begin{document}
\maketitle

% ===========================================================================
% ABSTRACT
% ===========================================================================

\begin{abstract}
Agents built on large language models (LLMs) rely on reliability techniques such as retry, majority voting, and self-consistency that have been developed in parallel rather than within a common analytical framework. We observe that an LLM sampled at temperature $T$ is a discrete stochastic channel $p(y \mid x)$ in the sense of Shannon's coding theory, and use this identity as the entry point for such a framework. Each of these techniques is a special case of one of six classical reliability operators: diversity combining, hybrid retransmission, iterative generator-critic decoding, rateless sampling, structured redundant verification, and difficulty-adaptive routing. Within the framework we give two closed-form results, a noise-variance threshold above which uniform averaging beats quality-weighted averaging and a contractivity criterion for generator-critic refinement, and turn each into a design rule that can be checked before committing compute. We further introduce a cost-aware semantic-nearest-neighbor router whose single Lagrangian knob traverses the quality-cost frontier without retraining. Across six channel configurations spanning local and cloud models, no fixed model-technique-budget choice dominates. On a 300-item hard split of MMLU, GSM8K, and HumanEval, the router occupies the empirical Pareto frontier: at matched quality its normalized cost is ${\approx}56\%$ lower than the best fixed technique, and at matched cost it improves quality by ${\approx}7\%$ ($26\%$ over single-shot decoding). Grading 6,900 outputs without any LLM judge confirms that the judge tracks correctness at the item level and that the router's advantage over a fixed policy persists at most cost settings; with internal reasoning enabled, the gain of external refinement shrinks but remains significant for our incremental-redundancy operator. These results argue for consolidating reliability techniques into a single tunable layer informed by channel coding. \looseness=-1
\end{abstract}

% ===========================================================================
% 1. INTRODUCTION
% ===========================================================================
\section{Introduction}
\label{sec:intro}
% \vspace{-2mm}
Large language models (LLMs) have become the computational backbone of autonomous AI agents, powering applications from code generation to scientific reasoning~\citep{openai2023gpt4,anthropic2024claude}.
Yet a fundamental tension persists: individual LLM calls are \emph{stochastic and unreliable}.
Output quality varies substantially with prompt formulation, sampling temperature, and model choice.
Hallucination, logical gaps, and omissions occur unpredictably, particularly on tasks near the frontier of a model's capability. \looseness=-1
%~\citep{wang2023selfconsistency}
Current approaches to agent reliability (retry loops, majority voting, and ensemble methods) each address specific failure modes but lack a unified theoretical framework that would provide principled design criteria across techniques.
Communication engineering offers such a framework: from Shannon's channel capacity theorem~\citep{shannon1948} through turbo~\citep{berrou1993turbo} and low-density parity-check (LDPC) codes~\citep{gallager1962}, it has developed a principled toolkit for \emph{extracting reliable information from unreliable channels}.

\textbf{Core insight.}
An autoregressive LLM with weights $\theta$, conditioned on a fixed system prompt with any in-context examples and sampled at temperature $T$, defines a conditional distribution $p_\theta(y \mid x)$ over outputs~$y$ given task input~$x$. The weights $\theta$ and the system context are held fixed across calls and together specify the channel; $T$ is a tunable channel parameter (suppressed in the notation for brevity). Treating each complete generation as a single channel use, $p_\theta(y \mid x)$ satisfies the definition of a channel transition probability $W(y \mid x)$~\citep{shannon1948, cover2006elements}.
Temperature sets how broadly the channel samples, but we do not model it as additive noise on a fixed correct answer: raising $T$ changes \emph{which} output is drawn, and the channel is not noiseless at $T\!=\!0$, since limited capacity, gaps in training data, and task mismatch leave errors that persist under greedy decoding (\Cref{rem:temperature}). Every operator below is therefore defined directly on sampled outputs and their quality scores; the additive-Gaussian picture of classical receivers motivates the operators but no reported number depends on it. Token-level log-probabilities provide an intrinsic soft-decision output.
Mutual information, diversity combining, and coding gain all have well-defined analogs in this setting, which yields both technique transfer and testable statements about operating regimes. \looseness=-1

\textbf{Contributions.} We make four contributions.
\begin{enumerate}[leftmargin=*,nosep]
  \item \textbf{A unifying communication-theoretic framework.} We formalize the LLM as a discrete stochastic channel (\Cref{sec:model}) and show that each widely cited inference-time reliability method (Self-Consistency~\citep{wang2023selfconsistency}, Best-of-$N$~\citep{cobbe2021verifiers,lightman2024verify}, Confidence-Informed Self-Consistency (CISC)~\citep{taubenfeld2025confidence}, Mixture-of-Agents~\citep{wang2025mixture}, Self-Refine~\citep{madaan2023selfrefine}, Multi-Agent Debate~\citep{du2024debate}, Chain-of-Verification~\citep{dhuliawala2024chain}, FrugalGPT/RouteLLM~\citep{chen2024frugalgpt, ong2025routellm}) is a special case of one of six classical reliability operators at a specific hyperparameter setting (\Cref{tab:baseline-reduction-app,tab:operator-view}).\looseness=-1

  \item \textbf{Two analytical results that yield design rules.} In support of the framework, a \emph{channel-state-information (CSI) noise crossover} (\Cref{prop:csi-crossover}) sharpens classical results on maximal-ratio combining with weighting errors~\citep{gans1971effect,tomiuk1999general} into a closed-form threshold, and a \emph{fixed-point reading of iterative refinement} (\Cref{prop:refinement-threshold}) identifies the generator-critic loop with an extrinsic-information-transfer (EXIT) recursion~\citep{ten_brink2001convergence}. In controlled tests, injected score noise reproduces the predicted single crossing, and Self-Refine, which has no quality gate, falls below a single call once reasoning is enabled. Each result yields a rule that can be checked before committing compute (\Cref{sec:analytical-predictions,app:prop-evidence,app:reasoning}).\looseness=-1

  \item \textbf{The \agentcodec{} operators and their evaluation.} \agentcodec{} (\Cref{sec:system}) instantiates the six operators; each variant adds one mechanism that the corresponding prior method omits. We evaluate them across six channel configurations, on 69 curated tasks and on a 300-item hard split of MMLU~\citep{hendrycks2021mmlu}, GSM8K~\citep{cobbe2021verifiers}, and HumanEval~\citep{chen2021humaneval} (\Cref{sec:experiments}), validate the LLM judge against judge-free grading of $6{,}900$ outputs and against a second judge, and repeat the iterative operators with internal reasoning enabled (\Cref{sec:exp-judge,sec:discussion}).

  \item \textbf{A cost-aware semantic-nearest-neighbor router} (\Cref{sec:router}). A non-parametric dispatcher with one Lagrangian knob $\lambda$ traverses the quality-cost frontier without retraining. On the 300-item split it improves mean quality by ${\approx}\qualityimp\%$ over the cross-validated best fixed technique at matched cost and matches it at ${\approx}\costimp\%$ lower normalized cost (\Cref{fig:acm-oracle-pareto}); it beats the best fixed technique on two further configurations and under judge-free grading at most cost settings, and, despite its simplicity, it is tied with or ahead of four learned routing baselines, including a gradient-boosted regressor and a contextual bandit, at 16 of 18 settings (\Cref{sec:exp-acm,sec:exp-judge}).
\end{enumerate}

\begin{figure}[t]
  \centering
  \includegraphics[width=\linewidth]{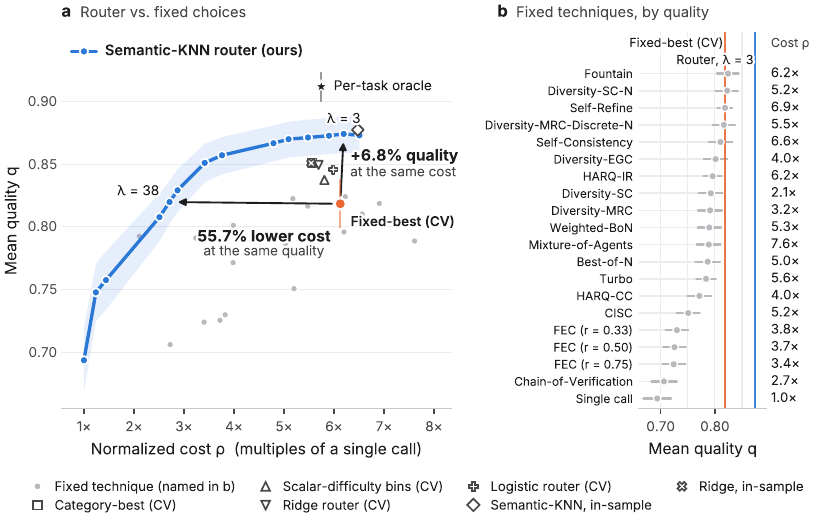}
  \caption{Cost-aware router versus fixed choices on the 300-item hard split (Ollama-cloud trio: Nemotron-Nano-3 30B and Devstral-Small-2 24B as channels, GLM-5.1 as judge; three generation repeats per task; 5-fold stratified cross-validation). \textbf{(a)}~Mean quality versus normalized cost $\rho$, in multiples of one uncoded call. Blue: the router as its knob $\lambda$ is swept, with a $95\%$ bootstrap band. Gray: the 20 fixed techniques. Orange: the cross-validated best fixed technique. Star: per-task oracle. Hollow markers: other routing policies. Arrows mark the gain at matched cost ($\lambda{=}3$) and the cost reduction at matched quality ($\lambda{=}38$). \textbf{(b)}~Every fixed technique by mean quality with its $95\%$ bootstrap interval, and its normalized cost at right; vertical lines mark the cross-validated best fixed technique and the router at $\lambda{=}3$. Quality is the judge score (\Cref{sec:system}); \Cref{sec:exp-judge} repeats the comparison under judge-free grading.}
  \label{fig:acm-oracle-pareto}
\end{figure}

% Classical maximal-ratio combining is information-theoretically optimal under high signal-to-noise ratio with perfect channel-state information~\citep{proakis2008digital},
% Pilot-probe channel-quality estimation from the token log-probabilities of a short probe response (\Cref{sec:system}); soft-output diversity combining and rateless decoding that use the channel's intrinsic log-probability output in place of a judge call (\Cref{sec:system,sec:soft}); a critique-free chase-combining baseline that, at matched call budget, isolates the gain from repeated transmission alone from the gain produced by structured extrinsic feedback (\Cref{sec:exp-harq}); and finally a
% Owing to the page limit, details of our analyses and additional results are deferred to the appendix.

% ===========================================================================
% 2. BACKGROUND AND RELATED WORK
% ===========================================================================
% \vspace{-3mm}
\section{Related work}
\label{sec:background}
% \vspace{-2mm}

We review the communication-theoretic foundations in \Cref{app:comms-primer} and focus here on positioning our work. The core communication-theory primitives we use are: \emph{selection combining} (SC, picking the best of several copies), \emph{equal-gain combining} (EGC, averaging copies with equal weight), \emph{maximal-ratio combining} (MRC, weighting copies by their channel-state-information / SC quality estimate), \emph{hybrid automatic repeat request} (HARQ, with the chase-combining variant HARQ-CC and the incremental-redundancy variant HARQ-IR), \emph{forward error correction} (FEC), and \emph{adaptive coding and modulation} (ACM, selecting among modulation-and-coding schemes (MCS) according to a channel-quality indicator). \looseness=-1

\textbf{Sampling and voting.}
Self-consistency~\citep{wang2023selfconsistency}, its confidence-informed extensions~\citep{taubenfeld2025confidence,feng2026optimal,corderoencinar2025certified}, best-of-$N$~\citep{chow2025inference,lightman2024verify}, and verifier-weighted best-of-$N$~\citep{li2023diverse} are selection combining (SC), with the weighted variants amounting to MRC over the discrete answer space. Test-time compute scaling~\citep{snell2025scaling,roberts2026testtime,zhu2025scaling} identifies parallel and sequential regimes mapping to diversity combining and HARQ / turbo coding.
\textbf{Ensembles and routing.}
Mixture-of-Agents (MoA)~\citep{wang2025mixture} is equal-gain combining (EGC); Self-MoA~\citep{li2026rethinking} is time diversity + EGC; Attention-MoA~\citep{wen2026attention} approaches maximal-ratio combining (MRC). LLM-Blender~\citep{jiang2023llmblender} performs rank-then-fuse~\citep{chen2026harnessing}. FrugalGPT~\citep{chen2024frugalgpt} and RouteLLM~\citep{ong2025routellm} are degenerate adaptive coding and modulation (ACM), adapting only the model without redundancy control. Smoothie~\citep{guha2024smoothie} routes each sample to one of several LLMs without labels, scoring candidates with a latent-variable model over their outputs. Our router is an instance of per-instance algorithm selection~\citep{rice1976algorithm}; it differs from these routers in its decision space, choosing among \emph{reliability operators} and their cost rather than among models.
\textbf{Iterative refinement.}
Self-Refine~\citep{madaan2023selfrefine} and Reflexion~\citep{shinn2023reflexion} resemble hybrid automatic repeat request with incremental redundancy (HARQ-IR) without an extrinsic-information constraint. Multi-agent debate~\citep{du2024debate,ning2025madfact} corresponds to turbo decoding but lacks the extrinsic-only exchange that is central to classical turbo decoders. LLM-as-a-Judge~\citep{zheng2023judging} is a noisy channel estimator~\citep{gu2024survey,schroeder2024trust}.
Each prior method can be understood as an instance of one communication-theoretic technique; our framework makes this mapping explicit and evaluates all six within a single experimental setup (\Cref{tab:mapping,tab:baseline-reduction-app}). \looseness=-1

% ===========================================================================
% 3. THE AGENT CHANNEL MODEL
% ===========================================================================
% \vspace{-2mm}
\section{The agent channel model}
\label{sec:model}
% \vspace{-2mm}
\subsection{Agent as stochastic channel}
% \vspace{-2mm}
\begin{definition}[Agent Channel]
\label{def:agent-channel}
Fix a finite token vocabulary $\mathcal{V}$ and let $\mathcal{V}^{*}$ denote the set of all finite-length token sequences over $\mathcal{V}$. Let the task (prompt) space be $\calT \subseteq \mathcal{V}^{*}$ and the output space $\mathcal{Y} = \mathcal{V}^{*}$. An autoregressive language model with parameters $\theta$ defines, for each input $x \in \calT$ and each sequence $y = (y_{1}, \ldots, y_{L}) \in \mathcal{Y}$ of length $L = |y|$, the conditional probability
% \vspace{-2mm}
\begin{equation}
  p_\theta(y \mid x) \;=\; \prod_{t=1}^{|y|} p_\theta(y_t \mid y_{<t}, x),
  \label{eq:channel}
\end{equation}
where $t \in \{1, \ldots, |y|\}$ indexes token position, $y_{t} \in \mathcal{V}$ is the token at position $t$, $y_{<t} := (y_{1}, \ldots, y_{t-1})$ denotes the prefix preceding position $t$ (with $y_{<1} := \varnothing$, the empty sequence, and an end-of-sequence symbol terminating the product), and $p_\theta(y_t \mid y_{<t}, x)$ is the next-token distribution produced by the model's softmax output. We call $\calA_{\theta} : \calT \to \mathcal{Y}$, $Y \sim p_{\theta}(\cdot \mid x)$, an \emph{agent channel}; it is a discrete stochastic channel that is memoryless across independent invocations, each complete generation constituting a single channel use\footnote{The autoregressive factorization in \eqref{eq:channel} introduces memory \emph{within} a sequence; treating each (prompt, response) pair as one channel use restores memorylessness across uses, the standard block-coding assumption~\citep{cover2006elements}. This applies to independent invocations only: multi-turn chains that condition on prior turns (HARQ and turbo, which feed critic feedback into the generator's prompt) violate it and are analyzed as channels with feedback in \Cref{sec:exp-harq,app:formal-foundations}; the memoryless analysis is used here only for single-turn parallel techniques (diversity combining, fountain, FEC).}.
The conditional $p_\theta(y \mid x)$ satisfies the axioms of a channel transition probability $W(y \mid x)$~\citep{shannon1948, cover2006elements}: it is non-negative for every $(x, y)$ and sums to one over $y \in \mathcal{Y}$ for every fixed $x$. The channel's noise characteristics (temperature-controlled per-token entropy, model-dependent quality) are structured enough for communication-theoretic reliability techniques to transfer, as we demonstrate.
\end{definition}

\begin{remark}[Temperature as the noise parameter]
\label{rem:temperature}
Let $T > 0$ denote the sampling temperature. Temperature sampling rescales the per-token logits and renormalizes, replacing $p_\theta(\cdot \mid y_{<t}, x)$ with $p_T(v \mid y_{<t}, x) \propto p_\theta(v \mid y_{<t}, x)^{1/T}$ for $v \in \mathcal{V}$. Writing $H_T(Y_t \mid y_{<t}, x) := -\sum_{v \in \mathcal{V}} p_T(v \mid y_{<t}, x) \log p_T(v \mid y_{<t}, x)$ for the per-token Shannon entropy under $p_T$, one has $H_T \to 0$ as $T \to 0^{+}$ (deterministic channel) and $H_T \to \log |\mathcal{V}|$ as $T \to \infty$ (uniform-noise channel). Temperature therefore continuously interpolates between a deterministic and a pure-noise channel, controlling the \emph{sampling} component of channel noise. Even at $T \to 0$ (greedy decoding) the channel is not noiseless: errors from limited model capacity, training-data gaps, and task-distribution mismatch persist under greedy decoding. Moreover, raising $T$ changes \emph{which} output is sampled rather than perturbing a fixed correct one, so $T$ is not the variance of an additive noise; we use it only as a control on the sampling distribution, and no quantitative result in this paper depends on an additive-noise model.
\end{remark}

\begin{definition}[Quality Function and Channel Estimation]
\label{def:quality}
A \emph{quality function} is any map $q : \calT \times \mathcal{Y} \to [0, 1]$ that takes a (task, output) pair $(x, y)$ to a scalar quality score, with $q = 1$ a perfect answer and $q = 0$ a wholly wrong one. The agent channel itself supplies an intrinsic, judge-free reliability measure: given a realized output $y \sim p_{\theta}(\cdot \mid x)$ of length $L = |y|$, the \emph{per-token negative log-likelihood} is \looseness=-1
\begin{equation}
  \mathcal{H}(\calA_\theta, x, y) \;=\; -\frac{1}{|y|}\sum_{t=1}^{|y|} \log p_\theta(y_t \mid y_{<t}, x),
  \label{eq:intrinsic}
\end{equation}
where $t$ and $p_{\theta}(y_t \mid y_{<t}, x)$ are as defined above. Equation \eqref{eq:intrinsic} is the response-level analog of the channel's average per-token Shannon entropy and is the standard agent-uncertainty proxy used in the calibration literature~\citep{kadavath2022language, kuhn2023semantic}. Except for variants discussed later that require the availability of log-probabilities, we implement $q$ as a separate large language model (the ``judge''), an external channel estimator that plays the role of pilot-based signal-to-noise ratio (SNR) estimation in classical receivers; this introduces estimator noise that we discuss in \Cref{rem:intrinsic-csi}. \looseness=-1
\end{definition}
% \vspace{-3mm}
\subsection{Channel Quality and Operating Point}
\label{sec:channel-quality}
% \vspace{-2mm}
Let $p_{X}$ be the task distribution over $\calT$. The \emph{expected channel noise rate} is $\bar{\mathcal{H}}(\calA_\theta, T) := \E_{X \sim p_X}\big[\E_{Y \sim p_{T}(\cdot \mid X)}[\mathcal{H}(\calA_\theta, X, Y)]\big]$, the average over both the task distribution and the sampling distribution at temperature $T$, and depends on model parameters $\theta$ (neural scaling laws show that training cross-entropy loss decays as $L \propto N^{-\alpha}$ in the model parameter count $N$~\citep{kaplan2020scaling, hoffmann2022training}, which upper-bounds $\bar{\mathcal{H}}$ on in-distribution tasks), temperature $T$ (\Cref{rem:temperature}), and task difficulty~\citep{yadkori2024believe}.
Model capacity and task difficulty act as \emph{dual} coordinates of one operating point: a hard task on a large model may produce the same $\bar{\mathcal{H}}$ as an easy task on a small one, so adapting the technique to estimated difficulty is the analog of adapting the coding scheme to the channel~\citep{goldsmith1997variable}. We use this duality as an organizing observation rather than a tested claim; it motivates per-task technique selection (\Cref{sec:router}). Full derivations appear in \Cref{app:formal-foundations}. \looseness -1

% \vspace{-3mm}
\subsection{Diversity Order}
% \vspace{-2mm}
\begin{definition}[Agent Diversity Order]
\label{def:diversity-order}
Given $d \geq 1$ agent channels $\calA_{1}, \ldots, \calA_{d}$, each independently generating $Y_{i} \sim p_{\theta_{i}}(\cdot \mid X)$ with quality scores $q_{i} := q(X, Y_{i}) \in [0,1]$, the ensemble has \emph{diversity order} $d$. For a threshold $\tau \in (0, 1)$, if the $q_{i}$ are independent the ensemble outage satisfies $\prob(\max_{i} q_{i} < \tau) = \prod_{i} P_{\mathrm{out}, i}(\tau)$ with $P_{\mathrm{out}, i}(\tau) := \prob(q_{i} < \tau)$, decaying geometrically in $d$. Correlation through shared training data or aligned prompts gives an \emph{effective} diversity $d_{\mathrm{eff}} \leq d$ (\Cref{rem:csi}; empirically $d_{\mathrm{eff}} \in [1.19, 1.71]$ across our configurations, \Cref{tab:branch-correlation,app:formal-foundations}). \looseness=-1
\end{definition}

\begin{remark}[MRC weighting heuristic]
\label{prop:mrc}
By analogy with SNR-proportional MRC in additive white Gaussian noise (AWGN) channels~\citep{proakis2008digital}, we set $w_i^{\text{MRC}} = q_i / \sum_{j} q_j$. Because our combiner is a nonlinear LLM and $q_i$ a noisy judge estimate, we treat this as a heuristic and rely on a cyclic-redundancy-check (CRC) style regression guard (\Cref{sec:system}) to ensure no degradation versus SC. The CSI-noise crossover that quantifies when this heuristic loses to EGC is in \Cref{rem:csi,sec:analytical-predictions}.
\end{remark}

\begin{remark}[CSI quality and the MRC-EGC reversal]
\label{rem:csi}\label{rem:intrinsic-csi}
\Cref{prop:mrc} requires branch independence and accurate CSI. In practice LLM outputs are correlated ($d_{\text{eff}} \leq d$) and numeric judge scores are a noisy estimate; we partially mitigate the latter with a sigma-delta-style binary checklist judge (\Cref{sec:system}). Residual weighting errors can still attenuate MRC's gain over EGC~\citep{gans1971effect, tomiuk1999general}, quantified by the critical variance $\sigma_w^{*2}$ of \Cref{prop:csi-crossover}. Token log-probabilities offer a complementary intrinsic soft-decision channel that realizes classical soft-input soft-output (SISO) MRC~\citep{kadavath2022language} (\Cref{sec:soft,app:formal-foundations}).
\end{remark}

\subsection{Two analytical results}
\label{sec:analytical-predictions}
\vspace{-2mm}
Both results are stated and proved in the appendix; here we separate what each predicts before any technique is run from what it only explains after the fact.

\paragraph{CSI-noise crossover (\Cref{prop:csi-crossover}).}
Weighting branches by an estimated quality score helps only while those scores are reliable. For a linear combiner with Gaussian weighting errors~\citep{gans1971effect,tomiuk1999general}, \Cref{prop:csi-crossover} gives in closed form the critical weighting-noise variance $\sigma_{w}^{*2}$ above which equal-gain combining beats noisy maximal-ratio combining, and shows that the advantage of the latter decays monotonically and changes sign exactly once (\Cref{app:mrc-egc-crossover}). \emph{Prediction.} Scorer noise is measurable without ground truth: two judges from different vendors scoring the same $1{,}980$ outputs disagree with standard deviation $0.19$, which bounds per-judge noise at ${\le}1.7\times$ the within-task score spread. A sweep that holds the channel pair, tasks, and candidates fixed and varies only injected score noise reproduces the single crossing, with the maximal-ratio advantage falling from $+6.5$ to $-2.7$ percentage points (pp; \Cref{app:prop-evidence}). The resulting rule: with a noisy scorer, combine in the discrete answer space, or use equal-gain weights if a synthesis step is required. \emph{Explanation only.} Where the synthesis crossover sits in our main configurations was found by comparing configurations, not by sweeping noise within one.

\paragraph{Iterative-decoding threshold (\Cref{prop:refinement-threshold}).}
Identifying the refinement operator with a one-dimensional quality-update map $f$, \Cref{prop:refinement-threshold} recovers the contraction-versus-expansion dichotomy $|f'(q^{\infty})| \lessgtr 1$ of EXIT-chart analysis~\citep{ten_brink2001convergence, richardson2001capacity}. The dynamical-systems content is classical; our contribution is the modeling identification together with the best-of-sequence guard (\Cref{rem:bos-guard}), which keeps the delivered answer from regressing even when the unguarded map does. \emph{Prediction.} The proposition is local in the initial quality $q_0$: refinement should help most far below the fixed point and least, or harm, near it. Per trajectory, the unguarded gain falls from $+13.2$ pp for $q_0 < 0.4$ to $+1.1$ pp for $0.75 \le q_0 < 0.9$, and Self-Refine, which has no quality gate, loses $2.8$ pp on trajectories that start at $q_0 \ge 0.9$ ($n{=}355$, $p{<}10^{-19}$), whereas our iterative operators stop at their threshold and never refine such answers (\Cref{app:threshold-details}). The rule is to gate refinement on estimated initial quality, and the operators already do. Across generator scales, the share of turbo trajectories that end below their start is about a quarter at 3B and 8B ($26\%$ each) and falls to $5\%$ at 14B, where the mean gain is three to four times larger. \emph{Not established.} No configuration is net-harmful on average (refinement helps significantly at every scale, by $+5.1$ to $+19.7$ pp), so a pair operating strictly below threshold has not been observed and the full dichotomy remains untested; the empirical claim is the graded one above.

% Formal statements, proofs, and discussions for the above results are deferred to \Cref{app:mrc-egc-crossover,app:threshold-details}.
Additional parallels to communication theory can be found in \Cref{app:comms-primer,app:formal-foundations}.

% ===========================================================================
% 4. SYSTEM DESIGN: AGENTFORGE
% ===========================================================================
\section{System design: \agentcodec{}}
\label{sec:system}
\vspace{-2mm}
\agentcodec{} implements the six operators, their variants, and re-implementations of the prior methods in a three-layer architecture: (1)~a \emph{channel layer} wrapping any OpenAI-compatible LLM endpoint as a stochastic channel with cost and latency instrumentation; (2)~a \emph{technique layer} implementing the operators; and (3)~an \emph{orchestration layer} that routes tasks and collects metrics. \Cref{tab:mapping,tab:baseline-reduction-app} summarize the mapping and \Cref{app:hyperparams} the hyperparameters.
\textbf{Quality scorer (channel estimator).} A judge LLM answers 15 weighted yes/no criteria, and their weighted sum is the quality score $q \in [0,1]$ (sigma-delta style), which avoids the coarse quantization of numeric ratings (\Cref{app:exp-methodology}). When a final output is scored against a reference answer, one criterion carrying $12\%$ of the weight asks whether the final answer matches it; the other $88\%$ scores how the answer is reached: supporting details ($23\%$), factual accuracy ($20\%$), completeness and explanation ($20\%$), reasoning ($12\%$), and presentation ($13\%$). Judge scores therefore measure holistic answer quality, of which final-answer correctness is one part; \Cref{sec:exp-judge} compares them with judge-free grading.
\textbf{Where ground truth may enter.} A classical receiver estimates channel state from pilot symbols, never from the data symbol it is decoding. The same rule applies here: during decoding, neither the scorer that techniques call to rank, weight, and select their own candidates nor the critic sees the reference answer, so no technique can select a candidate by checking it against the answer. The reference enters only when final outputs are scored for evaluation. Exact and blended objective scorers are implemented but disabled in all reported runs, so reported qualities are judge scores except the post hoc judge-free grades of \Cref{sec:exp-judge}, and verified outcomes otherwise enter only through historical tasks used to calibrate the router (\Cref{sec:router}). \looseness -1

\begin{table}[tbp]
\centering
\caption{Communication $\to$ agent mapping. SC: selection combining; MIMO: multiple-input multiple-output; MCS: modulation and coding scheme. Full table in \Cref{app:comms-primer}.}
\label{tab:mapping}
\scriptsize
\setlength{\tabcolsep}{4pt}
\setlength{\tabcolsep}{3pt}
\begin{tabular}{@{}p{2.8cm}p{3.5cm}p{3.5cm}@{}}
\toprule
\textbf{Comm.\ Concept} & \textbf{Agent Analog} & \textbf{Implementation} \\
\midrule
Diversity (MIMO) & Multi-model/prompt/temp & Diversity ensemble \\
MRC / SC / EGC & Quality-weighted / best / equal & Combining strategies \\
HARQ-CC / IR & Retry / retry + critic & Chase / incremental \\
Turbo codes & Generator-critic loop & Extrinsic info exchange \\
Fountain codes & Sample until confident & Rateless adaptive \\
FEC block codes & Structured redundancy & Parity sections \\
ACM / MCS table & Difficulty-aware routing & 5-level MCS table \\
\bottomrule
\end{tabular}
\end{table}

\begin{table}[h]
\centering
\scriptsize
\setlength{\tabcolsep}{4pt}
\caption{\textbf{Taxonomic reduction of comparable prior methods.} Every widely cited inference-time reliability method is recovered as a special case of an \agentcodec{} technique by setting a single hyperparameter; the \agentcodec{} variant adds one mechanism the prior method omits. \Cref{tab:baseline-direct} reports the measured head-to-head comparisons at matched budget.\looseness=-1}
\label{tab:baseline-reduction-app}
\begin{tabular}{@{}>{\centering\arraybackslash}p{3.6cm}>{\centering\arraybackslash}p{2.4cm}>{\centering\arraybackslash}p{3.4cm}>{\centering\arraybackslash}p{2.7cm}@{}}
\toprule
\textbf{Prior method} & \textbf{\agentcodec{} variant} & \textbf{Reduction (prior $=$ ours when)} & \textbf{Added mechanism} \\
\midrule
Self-Consistency~\citep{wang2023selfconsistency} & Div.-MRC-Discrete-$N$ & cluster-MRC, $q_j\!\equiv\!1$ & Quality weighting \\
Best-of-$N$~\citep{cobbe2021verifiers,lightman2024verify} & Div.-SC-$N$ & single policy, $N$ temp.\ draws & Cross-model diversity \\
Weighted Best-of-$N$~\citep{li2023diverse} & Div.-MRC-Discrete-$N$ & single policy, judge-CSI & Cross-model + wider pool \\
CISC~\citep{taubenfeld2025confidence} & Div.-MRC-Discrete-$N$ & single policy, self-reported confidence as CSI & Cross-model pool + judge CSI \\
Mixture-of-Agents~\citep{wang2025mixture} & Div.-EGC & single layer, $w_i\!=\!1/d$ & (identity at single-layer EGC) \\
Self-Refine~\citep{madaan2023selfrefine} / Reflexion~\citep{shinn2023reflexion} & HARQ-IR & free-form (or verbal-RL) critique, full rewrite, last iter & Dedup'd structured critique + corrections + best-of-seq \\
Multi-agent debate~\citep{du2024debate} & Turbo & two generators, no extrinsic split & Extrinsic exchange, interleaver \\
Chain-of-Verification~\citep{dhuliawala2024chain} & FEC & fixed parity = verification & Rate-adaptive parity \\
FrugalGPT~\citep{chen2024frugalgpt} / RouteLLM~\citep{ong2025routellm} & Fountain / ACM (degenerate) & cheap-first cascade (FG) / single-shot model swap (RL), no coding & Sample combining + rate adaptation \\
Tree-of-Thoughts~\citep{yao2023tree} & Div.\ MRC (temporal) & per-branch eval., select-best & Fountain-style erasure marking \\
\bottomrule
\end{tabular}
\end{table}

\textbf{Diversity combining.}
The framework supports three diversity axes (model, prompt, temperature), each with SC/MRC/EGC combining; in this paper we focus on model diversity. Each branch generates $Y_i = \calA_i(X)$, is scored $q_i = q(X, Y_i)$, and MRC combines with weights $w_i = q_i / \sum_j q_j$ (pseudocode in \Cref{app:hyperparams}). Compile-time prompt construction~\citep{khattab2024dspy,liu2022makesgood,rubin2022learning,zhou2023ape,yang2024opro} is an orthogonal, amortized-before-inference source of prompt diversity; composition with decode-time channel combining is left to future work. 
\textbf{HARQ.}
HARQ-CC generates $K$ independent attempts and combines equally. HARQ-IR iteratively refines via a critic that identifies specific errors (new parity); both stop at quality $\geq \tau$. 
\textbf{Turbo decoding.}
A generator and critic iteratively exchange \emph{extrinsic} information: each round, the critic provides only \emph{new} feedback (scaled by a damping factor $\alpha < 1$), mirroring the extrinsic-only exchange used in classical turbo decoders. Extrinsic scaling, a severity floor, and adaptive damping prevent the oscillation and over-correction that otherwise plague naive critic-generator loops (\Cref{app:hyperparams}).
\textbf{Fountain codes.}
Samples are generated until a confidence metric exceeds a threshold. The \emph{rateless} property automatically adapts cost to task difficulty.
\textbf{FEC.}
Structured parity sections (reasoning, self-verification, alternative approach, confidence), each from a separate call, accompany the answer at rates $r \in \{3/4, 1/2, 1/3, 1/4\}$. A decoder cross-checks sections (syndrome decoding).
\textbf{ACM routing.}
A difficulty estimator (pilot signal) selects from a modulation-and-coding-scheme (MCS) table (MCS-0 through MCS-4); the router of \Cref{sec:router} generalizes this, jointly adapting model and redundancy strategy.
\textbf{Soft-output variants.}
Three optional soft variants (\textsc{Soft-MRC}, \textsc{Soft-Fountain}, \textsc{Soft-ACM}) replace judge-score CSI with the intrinsic log-likelihood-ratio-like (LLR-like) signal $c(y) = \exp\!\big(\tfrac{1}{L}\sum_{t} \log p(y_t \mid y_{<t}, x)\big)$ from $L$ per-token log-probabilities (\Cref{sec:soft}). \textbf{Judge, critic, voter.} Judge rates an answer (CSI source); critic rewrites/critiques an answer (extrinsic-information feedback in iterative decoders); voter clusters multiple answers into equivalence classes for discrete MRC/majority aggregation. Judge and voter are decoder-side aggregation operators (post-generation); the critic is part of the iterative decoder loop itself.

\paragraph{Cost-aware semantic-nearest-neighbor router.}
\label{sec:router}
The routing layer is a non-parametric dispatcher. Each task $x$ is embedded with BAAI/bge-large-en-v1.5~\citep{xiao2024cpack}; the dispatcher looks up the per-technique mean quality $\bar{q}_{i}(x)$ and cost $\bar{c}_{i}(x)$ over the $k = 20$ nearest tasks in a calibration cache of historical outcomes and selects
\begin{equation}
  \pi^{\star}(x; \lambda) \;=\; \argmax_{i \in \mathcal{C}} \big(\bar{q}_{i}(x) - \lambda \cdot \bar{c}_{i}(x) / \bar{c}_{\mathrm{baseline}}(x)\big),
  \label{eq:semknn-router}
\end{equation}
where $\mathcal{C}$ is the candidate pool (19 techniques plus the uncoded baseline in our experiments) and $\lambda \geq 0$ is a single Lagrangian knob: $\lambda \to 0$ recovers a pure quality maximizer and $\lambda \to \infty$ the cheapest candidate. The rule is fully determined by the cache and the embedding model, so one instance serves every cost regime by adjusting $\lambda$, with no retraining. The mechanism is per-instance algorithm selection~\citep{rice1976algorithm} over reliability operators: the framework supplies the candidate set and the cost-normalized objective, not the selection mechanism itself. Calibration costs one evaluation of the candidate pool per historical task ($0.110$ USD per task for our pool, \Cref{app:router-extended}). In deployment, the cache can be kept current by occasionally dispatching a task to a candidate other than the current argmax and recording the outcome; we note this standard exploration practice but do not evaluate it as a contribution. \looseness -1

% ===========================================================================
% 5. EXPERIMENTAL EVALUATION
% ===========================================================================
% \vspace{-3mm}
\section{Experimental evaluation}
\label{sec:experiments}
% \vspace{-2mm}
\subsection{Setup}
% \vspace{-2mm}
\paragraph{Benchmarks and curated tasks.}
\emph{Standard benchmarks}: a 300-item hard split with 100 items each from GSM8K, MMLU, and HumanEval, selected by published structural proxies of difficulty and never by technique behavior (construction in \Cref{app:data}). Every quantitative claim (the router, the technique comparisons, and the judge validation) rests on this split. \emph{Curated tasks}: 69 tasks in four categories (question answering, reasoning, creative, code) at three difficulty tiers, written to expose technique differences at the 3 to 30B operating points (\Cref{app:data}). They carry one qualitative claim, that no fixed model-technique-budget choice dominates across operating points, which the 300-item split reproduces independently.
\vspace{-3mm}
\paragraph{Models and scoring.}
Six channel configurations (\Cref{tab:model-configs}) span local open-weight models (3B, 8B, and 14B from the Qwen, Llama, Phi, Gemma, and DeepSeek families), an Anthropic + OpenAI cloud pair, and an Ollama-cloud trio (Nemotron-Nano-3 30B and Devstral-Small-2 24B as channels, GLM-5.1 as judge), which runs the 300-item split. The judge ranges from smaller than the channels it grades (Gemma-3 12B on the 14B pair) to larger; on the cloud pair, Haiku~4.5 is both channel and judge. Internal reasoning (``thinking'') is disabled in these configurations; \Cref{sec:discussion} and \Cref{app:reasoning} report the operators with it enabled. Local configurations also run the soft-output variants, which need token log-probabilities.
\vspace{-3mm}
\paragraph{Statistics and resolution.}
We report quality $q \in [0,1]$, coding gain $G$ (mean improvement over the uncoded single call, in percent of its mean unless marked pp), and normalized cost $\rho(\pi, X) := \mathrm{Cost}(\pi, X) / \mathrm{Cost}(\mathrm{baseline}, X)$, the analog of $E_b/N_0$ (\Cref{app:exp-methodology}). Means carry $95\%$ bootstrap intervals; tests are two-sided paired Wilcoxon signed-rank on per-task differences, with repeats averaged per task. The 300-item split and both 3B configurations were generated three independent times; the other four columns of \Cref{tab:results-summary} are single runs. On the three-repeat 300-item split, single runs rank techniques with Spearman $0.91$ to $0.94$, and the single-run standard deviation of a technique mean is $0.007$ against a between-technique spread of $0.040$; a single run therefore resolves differences of about $2$ pp on 300 tasks and about $4$ pp on 69 tasks, and we report smaller gaps as ties (\Cref{app:noise}). \looseness -1

\vspace{-3mm}
\subsection{Prior methods as special cases}
\label{sec:exp-baselines}
\vspace{-2mm}
We re-implement Self-Consistency, Best-of-$N$, Weighted Best-of-$N$, Confidence-Informed Self-Consistency (CISC), Mixture-of-Agents, Self-Refine, and Chain-of-Verification from their canonical recipes, scored with the same tasks and judge as the \agentcodec{} runs. Each is one of three combining operators (selection, equal-gain, maximal-ratio) over a continuous or a discrete answer space, with three orthogonal knobs: the CSI source, the pool size $N$, and passive versus iterative combining (\Cref{app:operator-view-derivations}). At matched budget on the 300-item split (\Cref{tab:baseline-direct}), the \agentcodec{} variant has the higher mean judge score in 6 of 7 pairs, significantly in four: Diversity-MRC-Discrete-$N$ over CISC ($+6.6$ pp at $5\%$ higher cost) and over Self-Consistency ($+0.6$ pp at $16\%$ lower cost), Diversity-EGC over Mixture-of-Agents ($+1.3$ pp at $47\%$ lower cost), and FEC over Chain-of-Verification ($+1.9$ pp at $39\%$ higher cost). Self-Refine scores $2.3$ pp above HARQ-IR under the judge ($p{=}0.09$), but on the regeneration graded without a judge it is $7.3$ pp below ($p{=}0.008$, \Cref{app:judge-validation}). Cost advantages generalize while small quality margins do not: Diversity-EGC is $47\%$ to $54\%$ cheaper than Mixture-of-Agents in all four configurations where both ran, but its quality difference ranges from $-4.7$ to $+1.3$ pp (\Cref{app:judge-validation}). FrugalGPT and RouteLLM are degenerate fountain and ACM decoders and are subsumed by the router's candidate pool.

\vspace{-3mm}
\subsection{Quality, cost, and the operating point}
\label{sec:exp-quality-cost}
\label{sec:exp-diversity}
\label{sec:exp-harq}
\label{sec:exp-fec}
\vspace{-2mm}
\paragraph{Diversity combining.}
With LLM synthesis, equal-gain combining beats maximal-ratio combining on the 14B pair, whose judge is smaller than both channels ($+4.0$ pp, $p{=}0.002$), and on the 300-item Ollama-cloud split ($+1.0$ pp, $p{<}10^{-3}$, three repeats): the reversal that \Cref{prop:csi-crossover} predicts once channel-state estimates are noisy enough. On the cloud pair the textbook ordering $G_{\mathrm{MRC}} = 7.9\% \geq G_{\mathrm{EGC}} = 6.2\% \geq G_{\mathrm{SC}} = 4.6\%$ appears, but its gaps lie below single-run resolution. Grading the 14B pair with the larger Gemma-4 31B instead of Gemma-3 12B (a separate generation run) also yields $\mathrm{MRC} \geq \mathrm{EGC} \geq \mathrm{SC}$ ($0.631$, $0.627$, $0.598$), as a less noisy judge should, though these gaps are small. Over extracted answers (the discrete combiner) equal-gain never wins significantly in any configuration (\Cref{app:prop-evidence}). With wider pools, Diversity-SC-$N$ and Diversity-MRC-Discrete-$N$ reach $G = +33\%$ and $+31\%$ at 8B.\looseness=-1
\vspace{-3mm}
\paragraph{Iterative refinement and forward error correction.}
At 14B, HARQ-CC reaches $G = +11.7\%$ ($p < 10^{-3}$), HARQ-IR $+14.0\%$ ($p < 10^{-3}$), and turbo $+14.2\%$ ($p < 10^{-5}$); the gap between the two HARQ variants is below single-run resolution. The best-of-sequence guard makes the delivered running maximum monotone ($0.726 \to 0.853$, \Cref{fig:turbo-waterfall}); how the unguarded iterate behaves across scales and initial quality is the graded evidence of \Cref{sec:analytical-predictions}. FEC at $r = 3/4$ (one extra reasoning section) gives $+7.6\%$ ($p < 0.05$, $\rho \approx 3\times$) at 14B; lower rates add cost without further gain (\Cref{fig:fec-rate}).
\vspace{-3mm}
\paragraph{The best fixed technique depends on the operating point and on the objective.}
Pairing each gain $G$ with its cost $\rho$ in the coding-gain efficiency $\eta := G/\rho$ shows that \emph{the $G$-leader is rarely the $\eta$-leader} (\Cref{tab:winner-vs-eff}). Among the single \agentcodec{} operators, at 3B, 8B, and on the Ollama-cloud curated tasks the $G$-leader is fountain at $\rho \in [6.3, 9.1]\times$, while the $\eta$-leader is Diversity-SC at $\rho \approx 2\times$; at 14B turbo leads on $G$ and Diversity-EGC on $\eta$; only on the cloud pair, whose baseline is near ceiling, does one technique (HARQ-IR) lead on both. Except at 14B, where the two leaders differ by less than single-run resolution, these are differences between technique families above that resolution, and they motivate per-task allocation.

\vspace{-3mm}
\subsection{The cost-aware router}
\label{sec:exp-acm}
\vspace{-2mm}
\Cref{tab:semknn-pareto} reports 5-fold stratified cross-validation of the router of \Cref{sec:router} on the 300-item split. The reference is the \emph{cross-validated fixed-best} policy (FB-CV): in each fold, the technique with the highest training-fold mean is applied to every test-fold task, so both policies are scored out of fold. At $\lambda{=}3$ the router improves mean quality over FB-CV by $\qualityimp\%$ ($+5.6$ pp; paired Wilcoxon $p{<}10^{-17}$, Cohen's $d_z{=}0.47$, per-task win rate $59.7\%$) at matched cost; at $\lambda{=}38$ the same instance matches FB-CV quality at $\costimp\%$ lower normalized cost, and the whole $\lambda$-sweep lies above every fixed technique (\Cref{fig:acm-oracle-pareto}). FB-CV itself is unstable: it switches between fountain and Self-Refine across folds. In the $\lambda$-sweeps below, FB-CV is cost-penalized: it applies the technique with the highest training-fold mean of quality minus $\lambda$ times cost.
\emph{Stronger baselines and further configurations.} Swept over the same $\lambda$ grid, a cost-sensitive ridge regressor, a gradient-boosted regressor, a RouteLLM-style classifier retrained for each $\lambda$, and an offline contextual bandit (disjoint LinUCB) all have lower out-of-fold quality than the router at every $\lambda$ on this split, significantly at most settings. On two further 300-item configurations (\Cref{app:router-extended}) the router again beats FB-CV at every $\lambda$ ($p{<}10^{-3}$); across the three configurations it is tied with or ahead of every learned baseline at 16 of 18 cost settings, while the ridge regressor is significantly ahead at the other two.
\emph{Calibration and transfer.} With 100 calibration tasks the router is within $1$ pp of its accuracy with the full 240-task training fold at each cost setting tested (judge-free grading). A 60-task cache costs $6.61$ USD and is amortized below the cheapest single call after about $5{,}500$ routed tasks. Fit on a regeneration of the cache and evaluated on the original tasks, the router keeps the highest mean quality at every $\lambda$; transfer across model configurations, model updates, and distribution shift is untested (\Cref{app:router-extended}).

\begin{table}[t]
\centering
\caption{Router on the 300-item split (Ollama-cloud trio; $300$ tasks $\times$ $3$ generation repeats; 5-fold stratified cross-validation with folds at the task level; $4000$ seed-aware bootstrap samples). $\Delta q$ is paired per task against the cross-validated fixed-best policy (FB-CV) and the strongest fixed technique scored in sample (FB-IS); $^{***}p{<}10^{-3}$ for a lead, paired Wilcoxon. Full sweep in \Cref{tab:semknn-pareto-full}.}
\label{tab:semknn-pareto}
\scriptsize
\setlength{\tabcolsep}{2pt}
\begin{tabular}{@{}lccccccc@{}}
\toprule
\textbf{Policy} & \textbf{$q$ [95\% CI]} & \textbf{Cost / task} & $\bm{q/\$}$ & $\bm{\rho}$ & $\bm{G}$ \textbf{(\%)} & $\bm{\Delta q}$ \textbf{vs.\ FB-CV} & $\bm{\Delta q}$ \textbf{vs.\ FB-IS}\\
\midrule
Oracle (per-task max) & $0.912$ [$0.900, 0.923$] & $\$0.00668$ & $165.4$ & $5.75\times$ & $+31.4$ & n/a & $+0.088$ \\
Feasible (router, leave-one-out) & $0.877$ [$0.863, 0.892$] & $\$0.00735$ & $145.5$ & $6.48\times$ & $+26.5$ & n/a & $+0.053$ \\
\textbf{Router, $\lambda{=}3$} & $\bm{0.874}$ [$0.859, 0.888$] & $\bm{\$0.00695}$ & $\bm{160.2}$ & $\bm{6.20\times}$ & $\bm{+26.0}$ & $\bm{+0.056^{***}}$ & $\bm{+0.050}$ \\
\textbf{Router, $\lambda{=}38$} & $\bm{0.820}$ [$0.802, 0.837$] & $\bm{\$0.00290}$ & $\bm{388.7}$ & $\bm{2.71\times}$ & $\bm{+18.2}$ & $\bm{+0.001}$ & $\bm{-0.004}$ \\
\emph{Strongest in-sample fixed: fountain} & $0.824$ & $\$0.00691$ & $119.2$ & $6.23\times$ & $+18.8$ & $+0.006$ & (FB-IS) \\
\textbf{Cross-validated fixed-best (FB-CV)} & $\bm{0.818}$ [$0.800, 0.838$] & $\bm{\$0.00698}$ & $\bm{143.6}$ & $\bm{6.13\times}$ & $\bm{+18.0}$ & (FB-CV) & $\bm{-0.006}$ \\
Always-baseline (uncoded) & $0.694$ [$0.668, 0.719$] & $\$0.00120$ & $577.9$ & $1.00\times$ & $+0.0$ & $-0.125$ & $-0.130$ \\
\bottomrule
\end{tabular}
\end{table}

\vspace{-3mm}
\subsection{Judge validation}
\label{sec:exp-judge}
\label{sec:exp-standard}
\label{sec:exp-predictions}
\vspace{-2mm}
Because several techniques use the judge internally and every headline score comes from it, we graded $6{,}900$ outputs (23 techniques $\times$ 300 tasks from a regeneration of the Ollama-cloud configuration that persisted every output; \Cref{app:extra-configs}) without any LLM: MMLU by option letter, GSM8K by final number, HumanEval by executing the official unit tests in a sandbox. Grading is post hoc, with no feedback into any technique (\Cref{app:judge-validation}).
\emph{Item level.} The judge scores a correct answer above an incorrect one with probability $0.77$ (mean score $0.80$ versus $0.56$).
\emph{Technique level.} Rank correlation with objective accuracy is weak (Spearman $0.37$, $p{=}0.08$), for two reasons: 21 of the 23 techniques lie within an objective band of $0.78$ to $0.91$, and the judge measures holistic quality, with only $12\%$ of its weight on final-answer correctness (\Cref{sec:system}). The weakness is a property of LLM judging on near-tied techniques rather than of our judge: DeepSeek-V4-Flash, a different-vendor judge never used inside any technique, agrees with GLM-5.1 at Spearman $0.96$ and correlates with objective accuracy just as weakly ($0.33$), and Gemma-3 12B, smaller than both channels it grades, reproduces the ranking at Spearman $0.84$ on 90 tasks. One technique is clearly over-ranked: Self-Refine is first of 23 under the judge ($0.827$) and 21st objectively ($0.783$), which we verified is not an answer-extraction artifact; iterative polishing raises exactly the criteria that carry most of the judge's weight. We therefore read judge-based comparisons between techniques within a few points of each other as ties.
\emph{The router under judge-free grading.} In a deployment-faithful test, the router and FB-CV both select using judge scores only, and we compare the objective grade of what each picks, out of fold. Over the paper's 20-candidate pool, the router leads at five of six $\lambda$ values, by $9.7$ pp at $\lambda{=}3$ ($0.880$ versus $0.783$, $p{=}2{\times}10^{-5}$), and ties at the sixth: a fixed policy must commit to the judge's single favorite, here Self-Refine, whereas per-task dispatch does not inherit that one error. The margin depends on the pool and configuration: restricted to the 11 \agentcodec{} operators the two policies are statistically indistinguishable, and on a second configuration whose judge favorite is objectively sound, fixed-best leads by $1.8$ to $2.6$ pp at low and moderate cost pressure while the router costs up to half as much, and the router leads once cost is penalized (\Cref{app:judge-validation}). Under judge-free grading the per-task oracle reaches $0.990$ against $0.907$ for the best fixed technique, so $8.3$ pp of headroom lies beyond any fixed choice.
\emph{Other checks.} The leading operator differs by category and dataset (Diversity-SC-$N$ on GSM8K, fountain on MMLU, HARQ-CC on HumanEval; \Cref{app:standard-benchmarks}), and five controlled synthesis-integrity tests find no case of the synthesizer adding knowledge absent from its inputs; the one exception, an equal-gain synthesizer ten times larger than its channels, is confined to Diversity-EGC, and hand-coded ACM through it, on one configuration (\Cref{app:synthesis-integrity}). Human evaluation was not performed.

% ===========================================================================
% 5b. SOFT-OUTPUT EXTENSIONS
% ===========================================================================
% ===========================================================================
% 6. DISCUSSION
% ===========================================================================
\vspace{-3mm}
\section{Discussion}
\label{sec:discussion}
\vspace{-2mm}
\paragraph{Design rules.}
The framework suggests a starting point for each operator: \emph{diversity combining} when errors are independent across channels; \emph{HARQ with incremental redundancy} when a critic can localize errors; \emph{fountain} sampling when difficulty is unknown in advance; and \emph{FEC} when an answer admits consistency checks. Two rules follow from the analysis and can be checked before committing compute: estimate scorer noise by scoring the same outputs with two judges, and prefer discrete or equal-gain combining when that noise is large (\Cref{prop:csi-crossover}); and gate refinement on estimated initial quality (\Cref{prop:refinement-threshold}). Since no fixed choice wins across operating points, heterogeneous task streams are best served by the router, whose single knob $\lambda$ trades quality for cost; on the 300-item split its frontier lies above every fixed technique.\looseness=-1
\vspace{-3mm}
\paragraph{Reasoning models.}
The main trio's Devstral-Small-2 has no reasoning mode, so we run a pair whose channels both have one (Nemotron-Nano-3 30B, Gemma-4 31B) with reasoning on and off, same judge without reasoning, 300 tasks, judge-free grading (\Cref{app:reasoning}); all comparisons below are within this pair. Reasoning raises every technique: the single call by $19.7$ pp ($0.697 \to 0.893$) and HARQ-IR, turbo, and Self-Refine by $5.7$, $5.0$, and $3.3$ pp. As a coding-theoretic account predicts, coding gain shrinks as the channel improves: HARQ-IR's lead over the single call falls from $18.0$ to $4.0$ pp and remains significant ($p{=}0.019$), turbo's falls to $2.7$ pp (not significant), and Self-Refine, which has no quality gate, ends $8.7$ pp \emph{below} the single call ($p{<}10^{-3}$), as the gating rule predicts once most initial answers are already good. Reasoning is thus a further way to improve the channel, and gated incremental-redundancy refinement still adds significant accuracy on top of it, at $5.6\times$ the cost of a single call, a trade-off the cost-aware router is designed to make.
\vspace{-3mm}
\paragraph{Limitations.}
Full details are in \Cref{app:limitations}. (i)~Reported qualities are LLM-judge scores; judge-free grading and two further judges bound their reliability (\Cref{sec:exp-judge}), but no human evaluation was performed, and judge-based comparisons between near-tied techniques should be read as ties. (ii)~The analytical results are first-order and partly post hoc: the synthesis crossover was located by comparing configurations rather than by sweeping noise within one, and no configuration operates strictly below the refinement threshold. (iii)~The router is close to non-parametric algorithm selection; its margin depends on cache coverage, task distribution, judge, and candidate pool, as the two objective-grading configurations show, and transfer across model configurations, model updates, and distribution shift is untested. (iv)~Four of the seven configuration columns of \Cref{tab:results-summary} are single runs (the 8B and 14B local configurations run on a single RTX~3090); single-run noise measured on the configurations generated three times sets their resolution, and smaller gaps are treated as ties. Rankings from cloud generation campaigns run weeks apart agree less (Spearman $0.77$), whereas 3B runs six months apart on unchanged local weights agree as closely as runs within one campaign, which points to model updates as the cause (\Cref{app:noise}). (v)~At flat per-token rates the judge accounts for about $83\%$ of per-task cost. \looseness -1

% ===========================================================================
% 7. CONCLUSION
% ===========================================================================
\vspace{-2mm}
\section{Conclusion}
\label{sec:conclusion}
\vspace{-2mm}
Sampling an autoregressive language model at temperature $T$ realizes a discrete stochastic channel, and this identity lets the communication-theoretic reliability toolkit transfer with little adaptation. Six classical operators subsume the common inference-time reliability methods; a closed-form CSI-noise crossover and a fixed-point reading of refinement turn into rules that can be checked before committing compute; and a cost-aware semantic-nearest-neighbor router traverses the quality-cost frontier with one knob, improving quality by $\qualityimp\%$ at matched cost or cutting normalized cost by $\costimp\%$ at matched quality relative to the best fixed technique, with its advantage over a fixed policy persisting under judge-free grading at most cost settings. A substantially larger design space (interleaving, structured codes over reasoning chains, joint source-channel coding; \Cref{app:design-space}), operational capacity bounds for LLM channels, and calibrating $\lambda$ from a user's cost constraint remain open.\looseness=-1

% ===========================================================================
% ACKNOWLEDGMENTS (camera-ready only)
% ===========================================================================
\begin{ack}
This work was funded entirely by INTELLERCE LLC; no third-party funding was received. \textbf{Competing interests.} Both authors are affiliated with INTELLERCE LLC, which offers AgentCodec, a commercial implementation of the techniques studied in this paper, as part of its AptAgents platform (\url{https://intellerce.com/products/agentcodec}). For reproduction, the curated tasks are released as supplementary material, \Cref{app:data} rebuilds the 300-item split from public sources, and \Cref{app:hyperparams,app:impl-details} specify every technique.
\end{ack}

% ===========================================================================
% REFERENCES
% ===========================================================================

\bibliographystyle{plainnat}
\bibliography{references}

% ===========================================================================
% APPENDIX
% ===========================================================================
\clearpage
\appendix
\label{sec:appendix}

\section*{Appendix}
\addcontentsline{toc}{section}{Appendix}

\setcounter{page}{1}
\renewcommand{\thepage}{A\arabic{page}}

\section{Primer on communication systems}
\label{app:comms-primer}

This appendix provides an accessible introduction to the communication-theoretic concepts used throughout the paper.
Readers familiar with digital communications may skip this section; it is intended for the ML/AI audience.

\subsection{The communication problem}

The fundamental problem in communication is: \emph{how do you send a message reliably over a noisy channel?}
A transmitter sends a signal, the channel corrupts it (adds noise, causes fading, drops bits), and the receiver must recover the original message.
Shannon's channel coding theorem~\citep{shannon1948} establishes that reliable communication is \emph{possible} up to a fundamental limit called the channel capacity~$C$, provided the transmitter adds sufficient structured redundancy.

A prompt (message) is sent through a stochastic model (channel) whose transition probability is $p_\theta(y \mid x)$, which may produce hallucinations, omissions, or reasoning errors (noise), and the user (receiver) must extract a correct answer (decoded message).
As formalized in~\Cref{sec:model}, the LLM's conditional distribution satisfies the definition of a channel transition probability.

\subsection{Signal-to-Noise Ratio (SNR)}

SNR measures how much ``useful signal'' there is relative to ``noise'' in a received signal.
Higher SNR means cleaner reception and lower error rates.
In wireless systems, SNR depends on transmit power, distance, and interference.

\textbf{Agent analog:} The LLM channel's intrinsic noise is its conditional entropy $\mathcal{H} = H(Y|X)$, the per-token uncertainty in the model's output distribution (\Cref{eq:intrinsic}).
Lower $\mathcal{H}$ means higher channel quality (analogous to higher SNR).
In practice, we estimate channel quality via a quality score $q \in [0, 1]$ produced by an LLM judge, a noisy external channel estimator.
Just as wireless engineers measure SNR to decide how to process a received signal, our framework measures output quality to decide how to process agent outputs.

\subsection{Diversity and combining}

In wireless systems, a signal may travel via multiple independent paths (e.g., multiple antennas in MIMO systems).
Each path experiences different noise, so if one path is in a deep fade, another may still be strong.
The receiver \emph{combines} these paths to improve reliability.
Three standard strategies exist:

\begin{itemize}[nosep]
  \item \textbf{Selection combining (SC):} Pick the single best path. Simple but discards useful information from other paths.
  \item \textbf{Equal-gain combining (EGC):} Average all paths equally. Better than SC but suboptimal if path qualities differ.
  \item \textbf{Maximal-ratio combining (MRC):} Weight each path proportionally to its SNR, then combine. This is provably optimal under independent noise.
\end{itemize}

\noindent With $d$ independent paths, the probability that \emph{all} are bad drops as $\mathrm{SNR}^{-d}$, an exponential improvement called \emph{diversity gain}.

\textbf{Agent analog:} Generate $d$ independent agent outputs (using different models, prompts, or temperatures) and combine them.
SC $\to$ pick the best output; EGC $\to$ synthesize all equally; MRC $\to$ synthesize with quality-weighted emphasis.
Independence can be achieved through model diversity (spatial), prompt rephrasing (frequency), or temperature variation (time). Our experiments focus on spatial diversity; frequency and time diversity are supported by the framework and left to future evaluation. The frequency axis as defined here is decode-time prompt rephrasing applied per request; an orthogonal line of work realizes prompt diversity at \emph{compile time}, either by treating prompts as compilable artifacts with searchable instructions, demonstrations, and pipelines~\citep{khattab2024dspy}, by retrieving or learning per-task in-context examples~\citep{liu2022makesgood,rubin2022learning}, or by LLM-in-the-loop instruction optimization~\citep{zhou2023ape,yang2024opro}. These methods amortize their search before inference and produce a single optimized prompt at decode time; composing them with the decode-time channel combiners studied here is a natural, and untested, direction we do not pursue in this paper.

\subsection{Automatic Repeat Request (ARQ) and Hybrid ARQ}

ARQ is the simplest reliability mechanism: if the receiver detects errors, it requests retransmission.
\emph{Hybrid ARQ} (HARQ) combines this with forward error correction for better efficiency:

\begin{itemize}[nosep]
  \item \textbf{HARQ-CC (Chase Combining):} Retransmit the exact same data. The receiver soft-combines all copies, improving SNR through repetition.
  \item \textbf{HARQ-IR (Incremental Redundancy):} Each retransmission sends \emph{new} parity information. The decoder accumulates information across rounds, converging more efficiently than CC.
\end{itemize}

\noindent HARQ-IR is strictly more efficient: each round adds new information rather than repeating what was already sent.

\textbf{Agent analog:} HARQ-CC $\to$ retry the same prompt and combine all attempts.
HARQ-IR $\to$ use a critic to identify specific errors, then retry with that feedback (new information) appended to the prompt.
The quality threshold $\tau$ serves as the target error rate; iteration stops when quality exceeds~$\tau$.

\subsection{Turbo Codes and Iterative Decoding}

Turbo codes~\citep{berrou1993turbo} revolutionized communications by approaching Shannon's limit.
They use two simple encoders/decoders that iteratively exchange ``extrinsic information,'' soft reliability estimates that each decoder produces for the other.
Above a decoder-specific SNR threshold, this iterative exchange can produce a \emph{waterfall} effect in classical channels: error rates drop sharply once a critical information threshold is crossed. We do not claim this waterfall behavior carries over quantitatively to LLM channels; our turbo implementation aims for the qualitative benefit of extrinsic-only feedback, not a bit-error-rate curve.

\textbf{Agent analog:} Two LLM agents (generator and critic) iteratively exchange outputs.
The generator produces an answer; the critic provides structured feedback (extrinsic information); the generator refines.
Each round, both agents condition on the other's latest output, mirroring the extrinsic information exchange in turbo decoding.

\subsection{Fountain (Rateless) Codes}

Traditional codes have a fixed rate $r = k/n$, chosen before transmission.
Fountain codes~\citep{luby2002lt} instead produce a potentially \emph{unlimited} stream of encoded symbols.
The receiver collects symbols until it has enough to decode, automatically adapting to unknown or time-varying channel conditions without rate negotiation.

\textbf{Agent analog:} Generate agent outputs one at a time, checking quality after each.
Stop as soon as accumulated outputs are sufficient to synthesize a high-quality answer.
No need to pre-specify how many samples are needed; the system adapts to task difficulty at runtime.

\subsection{Forward Error Correction (FEC)}

FEC adds structured redundancy to a message \emph{before} transmission so the receiver can detect and correct errors without retransmission.
The \emph{code rate} $r = k/n$ controls the redundancy: lower rates add more redundancy (more protection, higher cost).
The receiver performs \emph{syndrome decoding}, checking parity constraints to locate and fix errors.

\textbf{Agent analog:} Instead of asking for just an answer, prompt the model to also produce structured ``parity'' sections: step-by-step reasoning (parity~1), self-verification (parity~2), confidence assessment (parity~3), and an alternative approach (parity~4).
A decoder agent then cross-checks these sections for internal consistency (analogous to syndrome decoding) and corrects any detected errors in the main answer.
The code rate controls how many parity sections are requested.

\subsection{Adaptive Coding and Modulation (ACM)}

In wireless systems (LTE, 5G), the transmitter continuously estimates channel quality and selects the best modulation and coding scheme (MCS).
Good channels get high-rate, low-redundancy schemes (high throughput); bad channels get low-rate, high-redundancy schemes (high reliability).
This is formalized as a table mapping SNR ranges to MCS levels.

\textbf{Agent analog:} Estimate task difficulty (the ``channel quality''), then route to the most cost-efficient technique.
Easy tasks use cheap, low-redundancy methods (single fast model); hard tasks use expensive, high-redundancy methods (turbo decoding with a frontier model).
This is the key to optimizing the cost-quality tradeoff: not every task needs the same level of effort.

\subsection{Summary of mappings}

\Cref{tab:comms-mapping} consolidates the communication-to-agent concept correspondences used throughout the paper, from the channel transition probability identity through the soft-decision recipe to the rate-reliability trade-off.

\begin{table}[h]
\centering
\small
\caption{Communication $\leftrightarrow$ Agent concept mapping.}
\label{tab:comms-mapping}
\begin{tabular}{@{}ll@{}}
\toprule
\textbf{Communication Concept} & \textbf{Agent Analog} \\
\midrule
Transition prob.\ $W(y|x)$ & $p_\theta(y \mid x)$ (exact identity) \\
Channel & LLM model \\
Noise level (by analogy only) & Sampling temperature $T$ \\
Noise (AWGN, fading) & Hallucination, omission, errors \\
Conditional entropy $H(Y|X)$ & Per-token uncertainty of LLM \\
Received signal & Agent output \\
Soft-decision output (LLR) & Token log-probabilities \\
SNR / channel quality & $-\mathcal{H}(\calA_\theta, x)$ (neg.\ cond.\ entropy) \\
Channel estimator (CSI) & Sigma-delta LLM-judge scorer \\
Diversity (MIMO) & Multi-model/prompt/temp sampling \\
Combining (SC/EGC/MRC) & Output selection/synthesis \\
ARQ retransmission & Prompt retry \\
HARQ-IR parity bits & Critic feedback \\
Turbo extrinsic info & Generator-critic exchange \\
Decoding threshold & Min.\ model capacity for convergence \\
Fountain symbols & Incremental samples \\
FEC parity bits & Reasoning/verification sections \\
Code rate $r = k/n$ & Redundancy level \\
ACM / MCS table & Task-difficulty routing \\
$I(X;Y)$, capacity $C$ & Mutual info., max quality at cost \\
\bottomrule
\end{tabular}
\end{table}

% \begin{table}[h]
% \centering
% \caption{Mapping of existing LLM reliability approaches to communication-theoretic techniques.}
% \label{tab:related-mapping}
% \small
% \begin{tabular}{@{}lll@{}}
% \toprule
% \textbf{Prior Work} & \textbf{Comm.\ Analog} & \textbf{Limitation} \\
% \midrule
% Self-consistency & SC diversity & Fixed voting, no MRC \\
% MoA / Self-MoA & EGC / time div. & No quality weighting \\
% Self-Refine & HARQ-IR & No extrinsic constraint \\
% Multi-agent debate & Turbo decoding & No extrinsic exchange \\
% Best-of-$N$ & SC with verifier & No combining \\
% FrugalGPT / RouteLLM & Fountain / ACM (model only) & No redundancy control \\
% \midrule
% \textbf{Ours} & \textbf{All six unified} & \textbf{---} \\
% \bottomrule
% \end{tabular}
% \end{table}

\section{Evaluation data: how to reconstruct every set from public sources}
\label{app:data}
\label{app:tasks}

The code is not released with this paper. The 69 curated tasks are released as supplementary material, and this appendix specifies the 300-item split completely enough to rebuild it, item for item and in the same order, from public sources; \Cref{app:hyperparams,app:impl-details} do the same for the techniques. Every task is stored as a record with five fields: an identifier, a category (question answering, reasoning, creative, or code), the prompt sent to the channel, a reference (gold answer or worked solution, absent for creative tasks), and metadata used only for grading.

\subsection{The 300-item hard split of MMLU, GSM8K, and HumanEval}
\label{app:data-hard}
All three subsets are drawn from the public test splits on the Hugging Face Hub, read in their distributed row order, and selected by deterministic rules that use only properties of the item itself. No model output or technique result enters the selection.

\paragraph{MMLU-Hard (100 items, category: question answering).} Source: \texttt{cais/mmlu}, configuration \texttt{all}, split \texttt{test}~\citep{hendrycks2021mmlu}. The first 100 rows of this split are exactly the 100 test questions of the subject \texttt{abstract\_algebra}, and these form MMLU-Hard; abstract algebra is among the lowest-scoring MMLU subjects for small open-weight models and stresses multi-step symbolic reasoning rather than recall. Prompt: the question, then one line per option in the form ``(A)~\textit{option text}'' through ``(D)'', then the line ``Answer with the letter only.''. Reference: the gold letter. Identifiers \texttt{mmlu\_hard\_0001} to \texttt{mmlu\_hard\_0100} follow the dataset order.

\paragraph{GSM8K-Hard (100 items, category: reasoning).} Source: \texttt{openai/gsm8k}, configuration \texttt{main}, split \texttt{test}~\citep{cobbe2021verifiers}. Take the first 200 problems in dataset order. For each, count the calculator annotations \texttt{<<\dots>>} in the reference solution (the number of annotated computation steps) and parse the final answer that follows \texttt{\#\#\#\#}, with thousands separators removed. Sort by step count, descending, breaking ties by the absolute value of the final answer, descending, and keep the first 100. This retains all 71 problems with four or more steps and the 29 three-step problems with the largest answers; step counts in the subset range from 3 to 7. Prompt: the problem text verbatim. Reference: the full worked solution (its last number is the gold answer). Identifiers follow the sorted order.

\paragraph{HumanEval-Hard (100 items, category: code).} Source: \texttt{openai/openai\_humaneval}, split \texttt{test} (164 problems)~\citep{chen2021humaneval}. Each problem is wrapped in the prompt ``Write a Python function to complete the following:'', a blank line, the problem's function signature and docstring inside a fenced \texttt{python} code block, a blank line, and ``Return only the complete function implementation.''. Count the docstring examples (occurrences of \texttt{>>>}) in the signature and docstring; sort by that count, descending, breaking ties by the character length of the wrapped prompt, descending; keep the first 100. Only 76 problems contain any docstring example, so the remaining 24 selected problems are the longest prompts without one. Reference: signature plus canonical solution. Grading metadata: the official test code and entry point.

\paragraph{Objective grading.} Grades are computed post hoc on each technique's final answer and never feed back into any technique. \emph{MMLU}: the option letter is taken from the first match of an explicit answer phrase (``answer is X'', ``Answer: X''), else from an option letter at the start of the reply, else from the last standalone letter; the grade is 1 if it equals the gold letter. \emph{GSM8K}: the last number in the reply (currency symbols and thousands separators removed) must equal the gold answer to within $10^{-9}$. \emph{HumanEval}: the first fenced \texttt{python} block (else the whole reply) must parse, and the program formed by the candidate, the official test code, and a call to \texttt{check} on the entry point must exit with status 0 in a sandboxed subprocess within 10\,s and 256\,MB of memory. A trailing ``Confidence: \textit{N}'' line, which confidence-weighted techniques append, is removed before grading.

\paragraph{Splits for the router.} Cross-validation uses 5 folds stratified by category with a fixed seed, assigned at the task level so that all generation repeats of a task fall in the same fold; each fold holds 240 calibration and 60 test tasks.

\subsection{The 69 curated tasks}
\label{app:data-curated}
The curated tasks were written by the authors and are not derived from a public benchmark. They come from an in-house pool of 131 tasks in four categories, each labeled by the authors, before any technique was run, with one of five difficulty levels (easy, medium, hard, very hard, extreme); the 69 used here are exactly the tasks labeled hard or above. \Cref{tab:curated-composition} gives the composition. The 69 tasks, with their references and evaluation criteria, are released as supplementary material under the CC BY 4.0 license. By category:
\begin{itemize}[nosep,leftmargin=*]
  \item \emph{Question answering} (16, all with references): multi-part factual questions that demand several precise values in one answer (for example, counts, atomic numbers, and dates asked together with an instruction not to guess), questions that hinge on fine distinctions between related concepts, and questions built around common misconceptions.
  \item \emph{Reasoning} (30, all with worked references): multi-step arithmetic and financial calculations, combinatorics and counting, probability (including conditional probability and strategy), number theory, base conversion followed by further computation, logic and constraint-satisfaction puzzles, scheduling, recursive computation, and state tracking, each with a fully worked solution.
  \item \emph{Creative} (9, no reference): constrained writing tasks that combine a form constraint with a content constraint (for example, a dialogue between two named historical scientists on a technical question), each with an explicit list of evaluation criteria that replaces the reference.
  \item \emph{Code} (14, all in Python): function specifications with a signature and, where relevant, a complexity requirement (for example, the median of two sorted arrays in $O(\log\min(m,n))$ time), graded by the judge.
\end{itemize}
Prompts are short (median length 42, 68, 47, and 64 words by category). At evaluation, all curated tasks are scored by the judge of \Cref{app:exp-methodology}, with the reference-based criterion set when a reference exists and the reference-free set otherwise.

\begin{table}[h]
\centering
\small
\caption{Composition of the 69 curated tasks by category and author-assigned difficulty.}
\label{tab:curated-composition}
\begin{tabular}{@{}lcccc@{}}
\toprule
\textbf{Category} & \textbf{Hard} & \textbf{Very hard} & \textbf{Extreme} & \textbf{Total (with reference)} \\
\midrule
Question answering & 6 & 7 & 3 & 16 (16) \\
Reasoning & 7 & 7 & 16 & 30 (30) \\
Creative & 3 & 4 & 2 & 9 (0) \\
Code & 6 & 5 & 3 & 14 (2) \\
\midrule
Total & 22 & 23 & 24 & 69 (48) \\
\bottomrule
\end{tabular}
\end{table}

\section{Linear-Combiner justification for the MRC weighting}
\label{app:mrc-proof}

This appendix reproduces the linear-combiner argument that motivates the weighting in \Cref{prop:mrc}; it is not a proof that this weighting is end-to-end optimal for a nonlinear LLM synthesizer.

Consider $d$ agent outputs $Y_1, \ldots, Y_d$ with quality scores $q_1, \ldots, q_d$, treated as independent signal components with additive noise of variance $\sigma_i^2$ and signal amplitude proportional to $q_i$. For a linear combiner with weights $w_i$ summing to a fixed constant, the combined SNR is maximized when $w_i \propto q_i / \sigma_i^2$. Under the further simplifying assumption that noise variance is approximately equal across branches (so $\sigma_i^2 \approx \sigma^2$), this reduces to $w_i^* \propto q_i$, giving the normalized form
\begin{equation}
  w_i^{\text{MRC}} = \frac{q_i}{\sum_{j=1}^d q_j}.
\end{equation}
This directly parallels the classical MRC result in AWGN channels, where optimal weights are proportional to per-branch SNR~\citep{proakis2008digital}.

\textbf{Caveats specific to the agent setting.} (i)~The LLM synthesizer is not a linear combiner; increasing $w_i$ biases the synthesizer's attention toward input $i$, but does not literally scale $Y_i$ additively. (ii)~$q_i$ is a noisy LLM-judge estimate, not an unbiased SNR. (iii)~Branch outputs $Y_i$ are correlated through shared training data, which reduces the effective diversity order $d_{\mathrm{eff}}$. The empirical ordering in \Cref{sec:exp-diversity} (where EGC outperforms MRC) is consistent with these caveats being active; see \Cref{rem:intrinsic-csi} and~\citet{gans1971effect,tomiuk1999general}.

\subsection{CSI-noise crossover between MRC and EGC}
\label{app:mrc-egc-crossover}

The empirical reversal $\mathrm{EGC} \succ \mathrm{MRC}$ observed in \Cref{sec:exp-diversity} under a coarse numeric-rating judge admits a quantitative explanation in a linear-combiner idealization of the synthesizer. Take a linear combiner with weights $w_i$ acting on noisy observations $r_i = a_i s + n_i$, where $a_i > 0$ are per-branch amplitudes proportional to true quality, $n_i \sim \mathcal{N}(0, \sigma^2)$ are independent, and $s = \pm 1$ is a binary signal. Let $S_k := \sum_{i=1}^d a_i^k$. The output SNRs of the two ideal combiners are
\begin{equation*}
  \gamma_{\mathrm{MRC}}\big|_{w_i = a_i} \;=\; \frac{S_2}{\sigma^2}, \qquad
  \gamma_{\mathrm{EGC}}\big|_{w_i = 1} \;=\; \frac{S_1^2}{d\, \sigma^2},
\end{equation*}
and the Cauchy-Schwarz inequality gives the classical $\gamma_{\mathrm{MRC}} \geq \gamma_{\mathrm{EGC}}$ ordering~\citep{proakis2008digital}, with equality iff $a_1 = \cdots = a_d$.

Let the judge supply noisy estimates $\hat{a}_i = a_i + \epsilon_i$ with $\epsilon_i \sim \mathcal{N}(0, \sigma_w^2)$ independent of $n_i$, and let the combiner use $w_i = \hat{a}_i$. Under this Gaussian weighting-error model of \citet{gans1971effect} and \citet{tomiuk1999general}, taking the expectations of numerator and denominator separately (a first-order, ratio-of-expectations approximation that we derive here) gives
\begin{equation}
  \E\!\left[\gamma_{\hat{\mathrm{MRC}}}\right] \;\approx\; \frac{S_2 \,(S_2 + \sigma_w^2)}{(S_2 + d\, \sigma_w^2)\, \sigma^2},
  \label{eq:approx-mrc-snr}
\end{equation}
which interpolates between $\gamma_{\mathrm{MRC}}$ at $\sigma_w \to 0$ and $\gamma_{\mathrm{MRC}}/d$ as $\sigma_w \to \infty$.

\paragraph{What is new in \Cref{prop:csi-crossover}.}
The degradation of maximal-ratio combining under Gaussian weighting errors is classical~\citep{gans1971effect,tomiuk1999general}; what is \emph{not} in those references is the closed-form value of $\sigma_w^2$ at which~\eqref{eq:approx-mrc-snr} crosses $\gamma_{\mathrm{EGC}}$. We record it below in elementary symmetric polynomials $(S_1, S_2)$ of the per-branch amplitudes, so that the formula is parameter-free in $\sigma$ and serves as a one-line diagnostic for the LLM-judge channel-state-information regime of \Cref{rem:intrinsic-csi}: every empirical $\mathrm{EGC} \succ \widehat{\mathrm{MRC}}$ reversal in our experiments is read as crossing $\sigma_w^{*2}$, and every reversion to the textbook ordering as falling back below it (\Cref{sec:exp-diversity}, \Cref{tab:results-summary}).

\begin{proposition}[CSI-noise crossover]
\label{prop:csi-crossover}
Let $d \geq 2$ and $a_1, \ldots, a_d > 0$ with at least two distinct amplitudes. Under the first-order expectation~\eqref{eq:approx-mrc-snr}, the map $\sigma_w^2 \mapsto \E[\gamma_{\hat{\mathrm{MRC}}}]$ is strictly decreasing on $[0, \infty)$ and crosses $\gamma_{\mathrm{EGC}}$ at exactly one positive value, the critical CSI-noise variance
\begin{equation}
  \sigma_w^{*2} \;=\; \frac{S_2 \,(d S_2 - S_1^2)}{d \,(S_1^2 - S_2)};
  \label{eq:csi-crossover}
\end{equation}
hence $\E[\gamma_{\hat{\mathrm{MRC}}}] > \gamma_{\mathrm{EGC}}$ iff $\sigma_w^2 < \sigma_w^{*2}$. In the degenerate equal-amplitude case $a_1 = \cdots = a_d$ (excluded by the hypothesis) the two combiners coincide at $\sigma_w = 0$ and $\sigma_w^{*2}$ collapses to $0$ as a limit of~\eqref{eq:csi-crossover}.
\end{proposition}

\begin{proof}
Write $u := \sigma_w^2$. Differentiating~\eqref{eq:approx-mrc-snr},
\begin{equation*}
\frac{d}{du}\,\E[\gamma_{\hat{\mathrm{MRC}}}]
\;=\; \frac{S_2}{\sigma^2} \cdot \frac{(S_2 + d u) - d (S_2 + u)}{(S_2 + d u)^2}
\;=\; -\,\frac{(d-1)\, S_2^{\,2}}{(S_2 + d u)^2 \,\sigma^2}
\;<\; 0
\end{equation*}
for $d \geq 2$, so $\E[\gamma_{\hat{\mathrm{MRC}}}]$ is strictly decreasing in $u$ with endpoints $\gamma_{\mathrm{MRC}} = S_2/\sigma^2$ at $u = 0$ and $S_2/(d\sigma^2)$ as $u \to \infty$. With $a_i > 0$ and at least two distinct amplitudes, $S_1^2 = S_2 + 2 \sum_{i<j} a_i a_j > S_2$ and $d S_2 > S_1^2$ (Cauchy-Schwarz, strict here), giving $\gamma_{\mathrm{MRC}} > \gamma_{\mathrm{EGC}} = S_1^2/(d\sigma^2) > S_2/(d\sigma^2)$, so the intermediate value theorem yields a unique crossing on $(0, \infty)$. Setting $\E[\gamma_{\hat{\mathrm{MRC}}}] = \gamma_{\mathrm{EGC}}$ and clearing denominators gives $d S_2 (S_2 + \sigma_w^2) = S_1^2 (S_2 + d \sigma_w^2)$, which rearranges to~\eqref{eq:csi-crossover}; the numerator $d S_2 - S_1^2$ and denominator $S_1^2 - S_2 = 2\sum_{i<j} a_i a_j$ are both strictly positive under the hypothesis and both vanish in the equal-amplitude limit.
\end{proof}

\paragraph{Testable signature.} The crossover variance~\eqref{eq:csi-crossover} is written in the per-branch amplitudes $a_{i}$, which we do not measure directly: the judge supplies $\hat{a}_{i}$ rather than $a_{i}$, and the noise floor $\sigma$ is itself a model-fitted quantity. The proposition's empirical content is therefore not the numerical value of $\sigma_{w}^{*2}$ for any one configuration, but the \emph{signature} it predicts: a $\widehat{\mathrm{MRC}}$-versus-$\mathrm{EGC}$ ordering reversal whenever the judge's effective resolution coarsens, and a return to the Cauchy-Schwarz ordering once the resolution is sharpened, with the crossing happening at a single value of $\sigma_{w}^{2}$ rather than gradually. Both directions of this signature, and the single-crossing structure, are observed in the experiments below (and summarized in \Cref{tab:results-summary}); the proposition is falsifiable by, e.g., a configuration where coarsening the judge \emph{improves} $\widehat{\mathrm{MRC}}$ relative to $\mathrm{EGC}$, or where the two combiners cross more than once as the judge is varied. We have not seen such a configuration in this paper; finding one would refute the linear-Gaussian idealization underlying~\eqref{eq:approx-mrc-snr}.

\paragraph{Attribution.} The first-order expectation \eqref{eq:approx-mrc-snr} for the MRC SNR under independent Gaussian channel-state-information errors is the textbook imperfect-CSI MRC analysis. The Cauchy--Schwarz comparison $\gamma_{\mathrm{MRC}} \geq \gamma_{\mathrm{EGC}}$ at perfect CSI is in any classical diversity-combining reference~\citep{proakis2008digital}, and the degradation of maximal-ratio combining under Gaussian weighting errors goes back to \citet{gans1971effect} and \citet{tomiuk1999general}; multichannel-receiver performance more broadly is treated in \citet{simon2005digital}. Our specific contribution is the closed-form crossover variance \eqref{eq:csi-crossover} as a one-line diagnostic in $(S_1, S_2)$ form, and its application to the LLM-judge CSI-quantization regime: the binary-checklist judge of \Cref{rem:intrinsic-csi} pushes $\sigma_w$ across the threshold and the empirical $\mathrm{EGC} \succ \widehat{\mathrm{MRC}}$ reversal of \Cref{sec:exp-diversity} reverts to the Cauchy--Schwarz ordering.

\textbf{Implication for the empirical reversal.} The pre-sigma-delta judge produced $\sim 33$ unique values across $516$ measurements (\Cref{rem:intrinsic-csi}), so the per-branch CSI quantization step was on the order of the spread of true amplitudes themselves, with $\sigma_w$ comparable to or exceeding $\sigma_w^*$. The 15-criterion binary checklist enlarges the effective level count by roughly a factor $2^{15}/33 \approx 10^3$, sharply reducing $\sigma_w$ and pushing operation back below the crossover, where Cauchy--Schwarz restores the $\mathrm{MRC} \succeq \mathrm{EGC}$ ordering. Intrinsic-logprob CSI (\Cref{sec:soft}) reduces $\sigma_w$ further still, since the weight is generated by the same process that produced $y$ and is not subject to judge quantization.

\textbf{Caveats.} (i) The first-order expectation neglects the covariance between the (random) numerator and denominator; we do not have a closed form for the exact ratio expectation, so the crossover value is a first-order statement; the controlled sweep of \Cref{app:prop-evidence} checks its qualitative content (monotone decay, single crossing) directly. (ii) The synthesizer is not a literal linear combiner; the result motivates the heuristic $w_i = q_i / \sum_j q_j$ but does not establish end-to-end optimality for an LLM combiner. (iii) Branch correlation introduces an additional EGC-favoring term not captured here.

\subsection{Coding Gain and Channel Capacity}

For a technique $\pi$, the coding gain $G(\pi) = \E[q_\pi(X)] - \E[q_{\text{baseline}}(X)]$ measures improvement over uncoded transmission; coding gain efficiency $\eta = G(\pi) / (\text{Cost}(\pi)/\text{Cost}(\text{baseline}))$ normalizes by cost overhead.
Since $p_\theta(y \mid x)$ is a well-defined conditional distribution, $I(X; Y) = H(Y) - H(Y|X)$ exists, and Shannon's coding theorem applies~\citep{shannon1948, cover2006elements}. We define the \emph{operational quality capacity} $Q^{*}(C) = \sup_{\pi \in \Pi_C} \E[q_\pi(X)]$ over technique configurations $\Pi_C$ with cost $\leq C$. Computing the true $C$ remains open (\Cref{app:formal-foundations}), but $I(X;Y)$ formalizes diminishing diversity returns under correlation (subadditivity of mutual information).

\section{Experimental methodology details}
\label{app:exp-methodology}

\paragraph{Hardware and infrastructure.}
All local experiments (3B/8B/14B configurations) were run on a single workstation: Intel Core i9-13900K (24 cores), NVIDIA GeForce RTX~3090 (24\,GB VRAM), 32\,GB DDR5 RAM, Ubuntu 24.04~LTS. Open-weight models are served via Ollama with default quantization (typically 4-bit GGUF for 14B models to fit within 24\,GB VRAM). Cloud experiments use the respective provider APIs (Anthropic and OpenAI for Claude Haiku~4.5 and GPT-5-mini; Ollama cloud for the cloud-served models of \Cref{tab:model-configs}). No multi-GPU or cluster resources are required; all local results are reproducible on consumer hardware with a single GPU.

\paragraph{Cost normalization.}
Raw cost conflates task difficulty with technique overhead: a hard task requiring five HARQ rounds costs more than an easy task needing one, but appears in the ``high cost, low quality'' quadrant, misleadingly suggesting that spending more degrades quality. We therefore normalize cost per task against its baseline, defining cost overhead:
\begin{equation}
  \rho(\pi, X) = \frac{\text{Cost}(\pi, X)}{\text{Cost}(\text{baseline}, X)}
  \label{eq:cost-overhead}
\end{equation}
This is analogous to the $E_b/N_0$ normalization in communications, where performance is plotted against energy \emph{per information bit} rather than total transmitted power, enabling fair comparison across channels with different noise levels.

\paragraph{Computational efficiency and scaling with dataset size.}
The dominant cost of every \agentcodec{} technique is LLM inference, and the cost overhead $\rho$ of \eqref{eq:cost-overhead} is the quantity reported throughout the experiments. Each technique scales linearly in its tunable budget knob and is independent across tasks: diversity combining at branch count $d$ uses $d$ generator calls, $d$ judge calls, and one synthesizer call; HARQ-CC/IR and turbo use up to $K{=}5$ rounds with one generator and one critic call per round (with adaptive early-exit when $q_k \ge \tau$); fountain stops adaptively in $[N_{\min}, N_{\max}] = [2, 8]$ samples; FEC at code rate $r$ uses $1/r$ generator calls plus one syndrome-decoding call. The deployed semantic-nearest-neighbor router adds one BAAI/bge-large-en-v1.5 embedding per task and a $k{=}20$ nearest-neighbor lookup over the training cache, which is $O(N_{\text{cache}})$ with the naive scan used here and $O(\log N_{\text{cache}})$ with any standard approximate-nearest-neighbor index; for $N_{\text{cache}} \le 300$ as in \Cref{sec:experiments}, both the embedding and the lookup are negligible against the seconds-scale latency of a single generator call. Training the parametric router ablations (\textsc{AcmRidge}, \textsc{AcmLearned}) completes in seconds on a single CPU core via closed-form $L_2$-ridge or a small multinomial logit per fold (\Cref{sec:exp-acm-decomp}). For a corpus of $N$ tasks, total inference compute therefore scales as $\Theta(N)$ (each task is processed independently) and the one-time cache-embedding cost as $\Theta(N)$; in our cloud configurations, wall-clock is throughput-bound by provider rate limits rather than by local compute, while in the local Ollama configurations it is bound by single-GPU generator throughput on the workstation of the previous paragraph. Because each task is processed independently, we do not anticipate qualitative changes in any technique's behavior at substantially larger $N$; whether the router's quality lead over the cross-validated best fixed technique (\Cref{tab:semknn-pareto}) sharpens or flattens as the cache grows, and whether an indexed neighbor structure is needed in deployment, are left to future evaluation.

\paragraph{Weighted binary-checklist judge (``sigma-delta-style'').}
A naive numeric-rating judge (``score this response from 0 to 1'') is severely quantized for small judge models: an initial Gemma-3 12B judge produced only $\sim$33 unique values across 516 measurements ($\sim$5-bit resolution), with $39\%$ of scores collapsing onto a single value. Under such coarse CSI, MRC's quality weighting amplifies estimation noise and the theoretical MRC $\geq$ EGC ordering collapses (\Cref{rem:intrinsic-csi}). We instead elicit 15 weighted \emph{binary} criteria covering correctness, completeness, reasoning, and presentation (different criterion set with and without a reference answer) and compute the score as the weighted sum of yes-answers. Binary yes/no decisions are far more reliable for small LLMs than numeric ratings, and 15 weighted 1-bit measurements yield up to $2^{15}$ distinct score values.\footnote{We label this design \emph{sigma-delta-style} by loose analogy with oversampled $\Sigma\Delta$ ADCs~\citep{norsworthy1997delta}, which trade bit depth for sample count. Our scorer is strictly a weighted-binary classifier; it lacks the noise-shaping feedback loop and temporal oversampling of a true $\Sigma\Delta$ modulator, so the analogy is qualitative and design-motivating rather than structural.} All experiments in this paper use this scorer; \Cref{tab:judge-rubric} lists both criterion sets with their weights.

\begin{table}[h]
\centering
\scriptsize
\setlength{\tabcolsep}{3pt}
\caption{The judge's 15 binary criteria and weights. Criteria marked $\star$ are strict (any violation fails them). With a reference, only the first criterion ($12\%$) checks the final answer against the reference.}
\label{tab:judge-rubric}
\begin{tabular}{@{}p{0.47\linewidth}r@{\hspace{10pt}}p{0.40\linewidth}r@{}}
\toprule
\multicolumn{2}{@{}l}{\textbf{With reference}} & \multicolumn{2}{l@{}}{\textbf{Without reference}} \\
\midrule
Main conclusion matches the reference & $.12$ & Directly answers the question & $.08$ \\
First key supporting detail from the reference present & $.08$ & No off-topic tangents & $.06$ \\
Second key supporting detail present & $.08$ & Main answer factually correct & $.10$ \\
Third key detail present & $.07$ & $\star$ All supporting facts verifiable & $.08$ \\
$\star$ No factual errors & $.10$ & $\star$ No hallucinated details & $.08$ \\
$\star$ Every specific claim verifiable & $.05$ & Covers several aspects of the topic & $.08$ \\
$\star$ No claims beyond or contradicting the reference & $.05$ & $\star$ Goes beyond surface level & $.06$ \\
Every part of the prompt answered & $.07$ & $\star$ Includes a specific example & $.05$ \\
$\star$ Explains why, not only what & $.07$ & $\star$ Explains mechanisms, not only facts & $.05$ \\
$\star$ Concrete example or evidence given & $.06$ & Reasoning chain valid & $.06$ \\
Each reasoning step follows & $.06$ & No logical contradictions & $.06$ \\
No internal contradictions & $.06$ & Well structured, clear flow & $.05$ \\
$\star$ Precise technical terminology & $.05$ & $\star$ Precise language, no vagueness & $.05$ \\
Clear structure & $.04$ & $\star$ Not repetitive or verbose & $.04$ \\
$\star$ No unnecessary repetition & $.04$ & Reader learns something specific & $.02$ \\
\bottomrule
\end{tabular}
\end{table}

\paragraph{Exact and blended scoring (implemented, disabled in all reported runs).}
For tasks with objectively checkable answers, the scorer also offers an exact mode, which replaces the judge by a deterministic check of the final answer (letter or numeric match, or execution against the task's tests), and a blend $q_{\text{blend}} = w \cdot q_{\text{obj}} + (1-w) \cdot q_{\text{judge}}$. Both are switched off in every reported run, so every reported quality is a judge score except the judge-free grades of \Cref{sec:exp-judge,app:judge-validation}, which are computed after the fact. Neither mode is available during decoding, where techniques call the scorer to rank, weight, and select their own candidates without access to the reference (\Cref{sec:system}); a reference-aware scorer there would let a technique pick the provably correct candidate and report oracle accuracy. Both modes are legitimate in two places only: for evaluation after the fact, and for computing the router's per-technique estimates on historical calibration tasks, whose references are not those of the task being served. \Cref{app:judge-validation} reports a sweep of $w$ for evaluation and the effect of blended calibration targets on the router.

\paragraph{Synthesis integrity.}
For MRC/EGC combining, synthesis contamination is a methodological concern: the synthesizer model may inject its own knowledge rather than faithfully merging inputs. Quality gains attributed to combining would then reflect the synthesizer solving the task independently, invalidating the diversity framework. Our constrained synthesis prompt (``output must contain ONLY information from the responses above'') combined with low temperature ($T\!=\!0.1$) is designed to prevent this; five controlled tests found no contamination, and the one configuration where a much larger synthesizer does act as a solver is analyzed in \Cref{app:synthesis-integrity}.

\section{Extended experimental results}
\label{app:extended-results}

This appendix presents the full experimental tables and figures omitted from the main text for space. \Cref{tab:baseline-direct} reports the direct head-to-head at matched inference budget on the 300-task hard-benchmark split. \Cref{tab:operator-view} consolidates the unified operator view of inference-time reliability methods and \Cref{tab:model-configs} the per-configuration channel and judge models. \Cref{tab:results-summary} reports the per-configuration relative quality gain over baseline (percent of baseline mean) for every technique. \Cref{fig:turbo-waterfall} shows the per-iteration turbo trajectory on the 14B local configuration and \Cref{fig:threshold-3panel} the per-trajectory refinement evidence across model scales and initial quality. \Cref{fig:fec-rate} reports forward-error-correction rate-distortion behavior, \Cref{fig:pareto-acm} the per-category quality-cost Pareto scatter, \Cref{fig:matched-budget} the matched-budget evaluation, \Cref{fig:quality-distribution} the per-technique quality distributions, \Cref{fig:heatmap,fig:heatmap_datasets} the technique-by-category heatmaps (14B local and the Ollama-cloud trio), and \Cref{fig:hard-benchmark-pareto} the per-dataset Pareto scatters on the 300-task hard split. The per-configuration figure gallery (\Cref{app:figure-gallery}) reproduces the same plot families for every channel configuration of \Cref{tab:model-configs}.

\begin{table}[h]
\centering
\small
\caption{\textbf{Direct head-to-head at matched inference budget on the 300-task hard-benchmark split} (Ollama-cloud trio, judge GLM-5.1; three generation repeats averaged per task; uncoded baseline $q_0=0.694$). $\Delta q$ is the paired per-task difference in judge score (positive favors \agentcodec{}); $\Delta$ cost is the change in mean cost per task; $p$ is the paired Wilcoxon signed-rank test over 300 tasks ($^{*}p{<}0.05$, $^{**}p{<}0.01$, $^{***}p{<}10^{-3}$). Judge-free grades for these pairs are in \Cref{app:judge-validation}.}
\label{tab:baseline-direct}
\begin{tabular}{@{}>{\raggedright\arraybackslash}p{3.6cm}>{\raggedright\arraybackslash}p{4.1cm}>{\centering\arraybackslash}p{1.8cm}>{\centering\arraybackslash}p{1.7cm}>{\centering\arraybackslash}p{1.1cm}@{}}
\toprule
\textbf{Prior method} & \textbf{\agentcodec{} variant} & \textbf{$\Delta q$ (paired)} & \textbf{$\Delta$ cost} & \textbf{$p$} \\
\midrule
Self-Consistency ($N\!=\!5$)            & Diversity-MRC-Discrete-$N$       & $+0.0063^{*}$   & $-16\%$       & $0.011$ \\
Best-of-$N$ ($N\!=\!5$, single-policy)  & Diversity-SC-$N$ (multi-model)   & $+0.0360$       & $+3\%$        & $0.052$ \\
Confidence-Inf.\ Self-Consist.\ (CISC)  & Diversity-MRC-Discrete-$N$       & $+0.0658^{***}$ & $+5\%$        & ${<}10^{-20}$ \\
Weighted Best-of-$N$                    & Diversity-MRC-Discrete-$N$       & $+0.0269$       & $+3\%$        & $0.61$ \\
Mixture-of-Agents                       & Diversity-EGC                    & $+0.0125^{***}$ & $-47\%$       & ${<}10^{-5}$ \\
Self-Refine ($K\!=\!3$)                 & HARQ-IR                          & $-0.0227$       & $-17\%$       & $0.085$ \\
Chain-of-Verification                   & FEC ($r\!=\!1/2$)                & $+0.0194^{**}$  & $+39\%$       & $0.005$ \\
\bottomrule
\end{tabular}
\end{table}

\begin{table}[h]
\centering
\small
\caption{\textbf{Unified operator view of inference-time reliability methods.} Each cell names the prior method occupying that combining-operator $\times$ answer-space slot. ``Discrete'' space means outputs are partitioned into semantic equivalence classes before combining; ``continuous'' means outputs are combined as free-form text via an LLM synthesizer. Iterative variants (HARQ, turbo) are orthogonal to the combining operator and reuse the same SC/EGC/MRC selector as their stopping gate.}
\label{tab:operator-view}
\begin{tabular}{@{}>{\raggedright\arraybackslash}p{2.8cm}>{\raggedright\arraybackslash}p{4.7cm}>{\raggedright\arraybackslash}p{5.5cm}@{}}
\toprule
\textbf{Operator} & \textbf{Continuous answer space} & \textbf{Discrete answer space (clusters)} \\
\midrule
SC (select best) & Diversity-SC, Diversity-SC-$N$ (multi-model, N=5); BoN~\citep{cobbe2021verifiers,lightman2024verify} (single-model) & (degenerate: clusters $=$ samples) \\
EGC (equal weight) & Diversity-EGC; Mixture-of-Agents~\citep{wang2025mixture} & Plurality vote (uniform $w$ across clusters) \\
MRC (quality weighted) & Diversity-MRC (synthesis) & Diversity-MRC-Discrete-$N$ (multi-model, judge-CSI); Diversity-MRC-Discrete-$N$-Soft (multi-model, logprob-CSI); Weighted Best-of-$N$ (single-model, judge-CSI); CISC~\citep{taubenfeld2025confidence} (single-model, self-reported-confidence CSI); Self-Consistency~\citep{wang2023selfconsistency} at $q_j\!\equiv\!1$ \\
\midrule
\multicolumn{3}{@{}>{\raggedright\arraybackslash}p{13cm}@{}}{\textbf{Iterative extensions} (any operator above, applied inside a feedback loop):} \\
\multicolumn{3}{@{}>{\raggedright\arraybackslash}p{13cm}@{}}{HARQ-CC $=$ EGC + retry; HARQ-IR, Self-Refine~\citep{madaan2023selfrefine}, Reflexion~\citep{shinn2023reflexion} $=$ MRC + critic; Multi-Agent Debate~\citep{du2024debate} $=$ two-generator turbo decoding.} \\
\midrule
\multicolumn{3}{@{}>{\raggedright\arraybackslash}p{13cm}@{}}{\textbf{Orthogonal axes} (apply to any cell):} \\
\multicolumn{3}{@{}>{\raggedright\arraybackslash}p{13cm}@{}}{CSI source $\in$ \{judge score, token log-prob (soft), trained PRM, none (EGC)\}; pool size $N$; iterative rounds $K$.} \\
\bottomrule
\end{tabular}
\end{table}

\begin{table}[h]
\centering
\caption{Channel configurations. Local models served via Ollama on a single workstation; cloud configurations use the respective vendor APIs. ``Spatial diversity'' = different architectures as independent channel realizations.}
\label{tab:model-configs}
\footnotesize
\setlength{\tabcolsep}{2pt}
\begin{tabular}{@{}llll@{}}
\toprule
\textbf{Config} & \textbf{Channels} & \textbf{Judge} & \textbf{Provider} \\
\midrule
3B local                & Qwen-2.5 3B, Gemma-3 4B                            & Gemma-3 12B       & Ollama (local)         \\
3B local + cloud judge  & Qwen-2.5 3B, Llama-3.2 3B                          & Gemma-4 31B       & Ollama (local + cloud) \\
8B local                & Llama-3.1 8B, Qwen-2.5 7B                          & Gemma-3 12B       & Ollama (local)         \\
14B local               & DeepSeek-R1 14B, Phi-3 14B                         & Gemma-3 12B       & Ollama (local + cloud judge) \\
Anthropic + OpenAI cloud & Claude Haiku 4.5, GPT-5-mini                      & Claude Haiku 4.5  & Anthropic, OpenAI APIs \\
Ollama-cloud trio       & Nemotron-Nano-3 30B, Devstral-Small-2 24B          & GLM-5.1           & Ollama cloud           \\
\bottomrule
\end{tabular}
\end{table}

% NUMBERS FROM: cache_3B, cache_3B_gemma4, cache_8B_llama_qwen_gemma, cache_deepseek14_phi314,
% cache3_haiku, ollama_nemotron_devstral_glm51, ollama_nemotron_devstral_glm51_datasets. All
% gains are relative to that configuration's own baseline.json on the matched task set.
\begin{table}[!ht]
\centering
\caption{Quality improvement over the uncoded single-call baseline (percent of baseline mean quality, paired per task) across the seven channel-configuration $\times$ task-set combinations of \Cref{tab:model-configs}; the two right-most columns use the Ollama-cloud trio (Nemotron-Nano-3 + Devstral-Small-2, judge GLM-5.1) on the 69 curated tasks and on the 300-task hard split. Columns marked $^\dagger$ were generated three times and averaged per task (in the 3B local column, entries marked $^\ddagger$ use the two later runs and their own baseline, since those techniques were added after the first run). Entries marked $^\S$ (3B + cloud-judge column only) are inflated: the Diversity-EGC synthesizer there, the 31B judge, answers tasks itself instead of merging the two 3B outputs, and hand-coded ACM inherits part of this through the tasks it routes to Diversity-EGC; no other entry is affected (\Cref{app:synthesis-integrity}). Columns marked $^{\circ}$ are single runs, on which a gap between two techniques below about $4$ pp ($2$ pp on 300 tasks) is not resolved and should be read as a tie (\Cref{app:noise}). Significance via paired Wilcoxon signed-rank against the baseline, uncorrected for multiple comparisons: $^{***}p{<}0.001$, $^{**}p{<}0.01$, $^{*}p{<}0.05$; ``(ns)'' not significant; ``--'' not run. Cost overhead $\rho$ is the per-task cost ratio to the baseline (repeats averaged first), averaged over tasks, and $\eta = G/\rho$.}
\label{tab:results-summary}
\scriptsize
\setlength{\tabcolsep}{2pt}
\begin{tabular}{@{}l ccccccc@{}}
\toprule
& \textbf{3B}$^\dagger$ & \textbf{3B+gemma4}$^\dagger$ & \textbf{8B}$^{\circ}$ & \textbf{14B}$^{\circ}$ & \textbf{A+O}$^{\circ}$ & \textbf{Ollama}$^{\circ}$ & \textbf{Ollama}$^\dagger$ \\
\textbf{Technique} & local & local+cloud & local & local & cloud & cloud n{=}69 & cloud n{=}300 \\
$q_0$ baseline & $0.622$ & $0.554$ & $0.604$ & $0.746$ & $0.818$ & $0.764$ & $0.694$ \\
\midrule
Diversity-SC & $+19.2^{***}$ & $+7.3^{***}$ & $+23.5^{***}$ & $+6.2^{*}$ & $+4.6^{*}$ & $+9.4^{***}$ & $+14.2^{***}$ \\
\hspace{1em}$\rho$ & $2.2\times$ & $1.9\times$ & $2.2\times$ & $1.8\times$ & $1.7\times$ & $1.9\times$ & $2.1\times$ \\
\hspace{1em}$\eta$ & $+8.7$ & $+3.8$ & $+10.7$ & $+3.5$ & $+2.6$ & $+5.0$ & $+6.8$ \\[2pt]
Diversity-MRC & $+20.6^{***}$ & $+8.0^{***}$ & $+26.9^{***}$ & $+6.2^{*}$ & $+7.9^{***}$ & $+11.7^{***}$ & $+14.0^{***}$ \\
\hspace{1em}$\rho$ & $5.0\times$ & $3.8\times$ & $3.7\times$ & $2.8\times$ & $2.8\times$ & $4.1\times$ & $3.2\times$ \\
\hspace{1em}$\eta$ & $+4.2$ & $+2.1$ & $+7.3$ & $+2.2$ & $+2.9$ & $+2.9$ & $+4.3$ \\[2pt]
Diversity-EGC & $+21.3^{***}$ & $+49.2^{***\S}$ & $+30.6^{***}$ & $+11.6^{***}$ & $+6.2^{***}$ & $+15.1^{***}$ & $+15.5^{***}$ \\
\hspace{1em}$\rho$ & $5.1\times$ & $4.8\times$ & $4.1\times$ & $3.0\times$ & $2.9\times$ & $6.0\times$ & $4.0\times$ \\
\hspace{1em}$\eta$ & $+4.2$ & $+10.3$ & $+7.5$ & $+3.8$ & $+2.2$ & $+2.5$ & $+3.9$ \\[2pt]
HARQ-CC & $+20.1^{***}$ & $+17.0^{***}$ & $+22.7^{***}$ & $+11.7^{***}$ & $+7.3^{**}$ & $+14.1^{***}$ & $+11.2^{***}$ \\
\hspace{1em}$\rho$ & $5.4\times$ & $6.0\times$ & $6.1\times$ & $4.0\times$ & $3.3\times$ & $4.9\times$ & $4.0\times$ \\
\hspace{1em}$\eta$ & $+3.7$ & $+2.8$ & $+3.7$ & $+2.9$ & $+2.2$ & $+2.8$ & $+2.8$ \\[2pt]
HARQ-IR & $+11.0^{***}$ & $+5.1$ (ns) & $+22.8^{***}$ & $+14.0^{***}$ & $+10.7^{***}$ & $+8.9^{**}$ & $+14.7^{***}$ \\
\hspace{1em}$\rho$ & $7.5\times$ & $8.3\times$ & $7.5\times$ & $5.4\times$ & $3.6\times$ & $7.1\times$ & $6.2\times$ \\
\hspace{1em}$\eta$ & $+1.5$ & $+0.6$ & $+3.0$ & $+2.6$ & $+3.0$ & $+1.3$ & $+2.4$ \\[2pt]
Turbo & $+10.0^{***}$ & $+5.0^{*}$ & $+13.9^{**}$ & $+14.2^{***}$ & $+10.2^{***}$ & $+8.5^{**}$ & $+13.0^{***}$ \\
\hspace{1em}$\rho$ & $6.3\times$ & $6.7\times$ & $6.7\times$ & $4.1\times$ & $4.3\times$ & $4.7\times$ & $5.6\times$ \\
\hspace{1em}$\eta$ & $+1.6$ & $+0.8$ & $+2.1$ & $+3.4$ & $+2.4$ & $+1.8$ & $+2.3$ \\[2pt]
Fountain & $+30.0^{***}$ & $+24.6^{***}$ & $+32.5^{***}$ & $+11.1^{***}$ & $+5.0$ (ns) & $+20.0^{***}$ & $+18.8^{***}$ \\
\hspace{1em}$\rho$ & $7.4\times$ & $7.8\times$ & $9.1\times$ & $7.4\times$ & $5.5\times$ & $6.3\times$ & $6.2\times$ \\
\hspace{1em}$\eta$ & $+4.1$ & $+3.1$ & $+3.6$ & $+1.5$ & $+0.9$ & $+3.2$ & $+3.0$ \\[2pt]
FEC ($r{=}3/4$) & $+2.8$ (ns) & $-7.2^{***}$ & $+4.2$ (ns) & $+7.6^{*}$ & -- & $+5.4^{*}$ & $+4.4^{***}$ \\
\hspace{1em}$\rho$ & $3.2\times$ & $3.0\times$ & $2.9\times$ & $3.1\times$ & -- & $4.3\times$ & $3.4\times$ \\
\hspace{1em}$\eta$ & $+0.9$ & $-2.4$ & $+1.5$ & $+2.4$ & -- & $+1.2$ & $+1.3$ \\[2pt]
FEC ($r{=}1/2$) & $+1.6$ (ns) & $-2.8$ (ns) & $+7.5$ (ns) & $+4.3$ (ns) & -- & $+2.4$ (ns) & $+4.6^{***}$ \\
\hspace{1em}$\rho$ & $3.5\times$ & $3.2\times$ & $3.5\times$ & $3.9\times$ & -- & $3.8\times$ & $3.7\times$ \\
\hspace{1em}$\eta$ & $+0.4$ & $-0.9$ & $+2.1$ & $+1.1$ & -- & $+0.6$ & $+1.2$ \\[2pt]
ACM (hand-coded) & $+29.1^{***}$ & $+27.2^{***\S}$ & $+34.3^{***}$ & $+10.2^{**}$ & $+2.7$ (ns) & $+12.9^{***}$ & $+8.6^{***}$ \\
\hspace{1em}$\rho$ & $8.6\times$ & $8.8\times$ & $8.4\times$ & $3.3\times$ & $2.5\times$ & $5.6\times$ & $4.0\times$ \\
\hspace{1em}$\eta$ & $+3.4$ & $+3.1$ & $+4.1$ & $+3.1$ & $+1.1$ & $+2.3$ & $+2.1$ \\[2pt]
Diversity-SC-$N$ & $+25.7^{***}$ & $+21.0^{***}$ & $+33.0^{***}$ & $+11.9^{**}$ & $+9.6^{***}$ & $+16.9^{***}$ & $+18.6^{***}$ \\
\hspace{1em}$\rho$ & $5.3\times$ & $4.9\times$ & $5.5\times$ & $4.3\times$ & $4.5\times$ & $5.0\times$ & $5.2\times$ \\
\hspace{1em}$\eta$ & $+4.9$ & $+4.3$ & $+6.0$ & $+2.8$ & $+2.1$ & $+3.4$ & $+3.6$ \\[2pt]
Diversity-MRC-Disc-$N$ & $+23.8^{***}$ & $+19.7^{***}$ & $+31.3^{***}$ & $+10.0^{**}$ & $+10.1^{***}$ & $+14.5^{***}$ & $+17.7^{***}$ \\
\hspace{1em}$\rho$ & $5.6\times$ & $5.1\times$ & $6.3\times$ & $5.7\times$ & $5.3\times$ & $6.0\times$ & $5.5\times$ \\
\hspace{1em}$\eta$ & $+4.3$ & $+3.9$ & $+4.9$ & $+1.7$ & $+1.9$ & $+2.4$ & $+3.2$ \\[2pt]
Self-Consistency ($N{=}5$) & $+23.4^{***}$ & $+18.6^{***}$ & $+31.6^{***}$ & $+8.7^{*}$ & $+8.2^{**}$ & $+14.9^{***}$ & $+16.8^{***}$ \\
\hspace{1em}$\rho$ & $6.5\times$ & $5.9\times$ & $7.1\times$ & $6.2\times$ & $5.9\times$ & $6.9\times$ & $6.6\times$ \\
\hspace{1em}$\eta$ & $+3.6$ & $+3.2$ & $+4.5$ & $+1.4$ & $+1.4$ & $+2.2$ & $+2.6$ \\[2pt]
Self-Refine ($K{=}3$) & $+7.8^{***}$ & $+2.3$ (ns) & $+9.6^{*}$ & $+5.1$ (ns) & $-1.0$ (ns) & $+9.6^{***}$ & $+18.0^{***}$ \\
\hspace{1em}$\rho$ & $5.3\times$ & $4.7\times$ & $6.8\times$ & $6.1\times$ & $10.9\times$ & $7.4\times$ & $6.9\times$ \\
\hspace{1em}$\eta$ & $+1.5$ & $+0.5$ & $+1.4$ & $+0.8$ & $-0.1$ & $+1.3$ & $+2.6$ \\[2pt]
Chain-of-Verification & $+1.5$ (ns) & $+0.2$ (ns) & $-0.9$ (ns) & $+2.9$ (ns) & -- & $+2.8$ (ns) & $+1.8^{**}$ \\
\hspace{1em}$\rho$ & $2.2\times$ & $2.1\times$ & $2.9\times$ & $2.5\times$ & -- & $3.2\times$ & $2.7\times$ \\
\hspace{1em}$\eta$ & $+0.7$ & $+0.1$ & $-0.3$ & $+1.2$ & -- & $+0.9$ & $+0.7$ \\[2pt]
Best-of-$N$ & $+22.9^{***\ddagger}$ & $+21.6^{***}$ & $+18.3^{***}$ & $+12.4^{***}$ & -- & $+10.6^{***}$ & $+13.4^{***}$ \\
\hspace{1em}$\rho$ & $5.1\times$ & $5.0\times$ & $5.2\times$ & $4.4\times$ & -- & $5.2\times$ & $5.0\times$ \\
\hspace{1em}$\eta$ & $+4.5$ & $+4.4$ & $+3.5$ & $+2.8$ & -- & $+2.0$ & $+2.7$ \\[2pt]
Weighted Best-of-$N$ & $+19.2^{***\ddagger}$ & $+19.9^{***}$ & $+14.9^{***}$ & $+12.1^{***}$ & $+7.4^{***}$ & $+10.2^{***}$ & $+13.8^{***}$ \\
\hspace{1em}$\rho$ & $5.6\times$ & $5.2\times$ & $6.1\times$ & $5.8\times$ & $6.4\times$ & $5.9\times$ & $5.3\times$ \\
\hspace{1em}$\eta$ & $+3.5$ & $+3.8$ & $+2.4$ & $+2.1$ & $+1.1$ & $+1.7$ & $+2.6$ \\[2pt]
Mixture-of-Agents & $+23.1^{***}$ & $+15.5^{***}$ & $+30.5^{***}$ & $+2.1$ (ns) & $+7.8^{**}$ & $+14.6^{***}$ & $+13.7^{***}$ \\
\hspace{1em}$\rho$ & $7.7\times$ & $6.6\times$ & $7.3\times$ & $11.1\times$ & $7.1\times$ & $6.2\times$ & $7.6\times$ \\
\hspace{1em}$\eta$ & $+3.0$ & $+2.3$ & $+4.2$ & $+0.2$ & $+1.1$ & $+2.4$ & $+1.8$ \\[2pt]
CISC & $+16.4^{***\ddagger}$ & $+14.8^{***}$ & -- & $+3.7$ (ns) & $+0.3$ (ns) & $+9.5^{***}$ & $+8.2^{***}$ \\
\hspace{1em}$\rho$ & $4.8\times$ & $4.8\times$ & -- & $8.1\times$ & $5.9\times$ & $5.1\times$ & $5.2\times$ \\
\hspace{1em}$\eta$ & $+3.4$ & $+3.1$ & -- & $+0.5$ & $+0.0$ & $+1.9$ & $+1.6$ \\[2pt]
\bottomrule
\end{tabular}
\end{table}

\begin{figure}[h]
  \centering
  \includegraphics[width=0.9\linewidth]{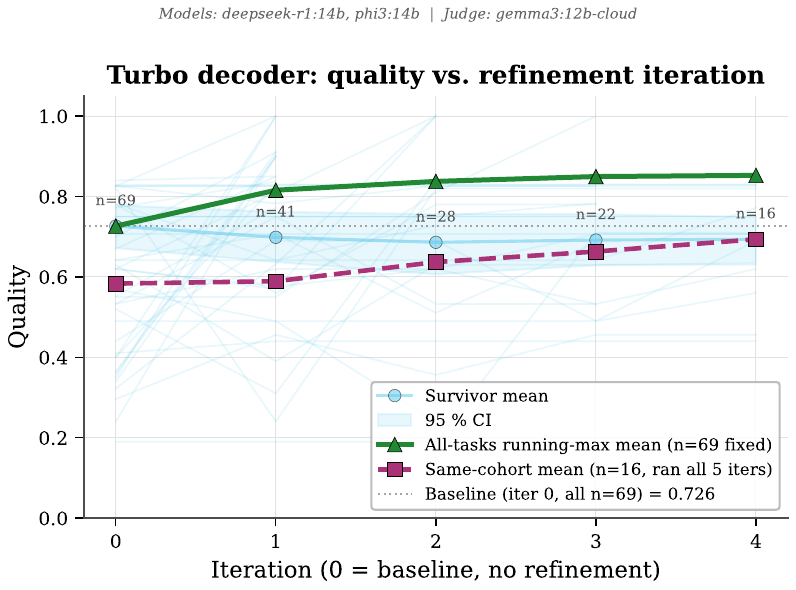}
  \caption{Turbo-decoder quality versus refinement iteration on the 14B local configuration (DeepSeek-R1 14B + Phi-3 14B; judge Gemma-3 12B; $n{=}69$ tasks). Iteration~0 is the initial generation, before any critic feedback; later iterations apply the structured-critique refinement of \Cref{prop:refinement-threshold}. The thin background lines are per-task trajectories; the bold curve is the all-tasks running-max mean (early-exited tasks contributing their best-so-far score), which is monotone by \Cref{rem:bos-guard}.}
  \label{fig:turbo-waterfall}
\end{figure}

\begin{figure}[h]
  \centering
  \includegraphics[width=\linewidth]{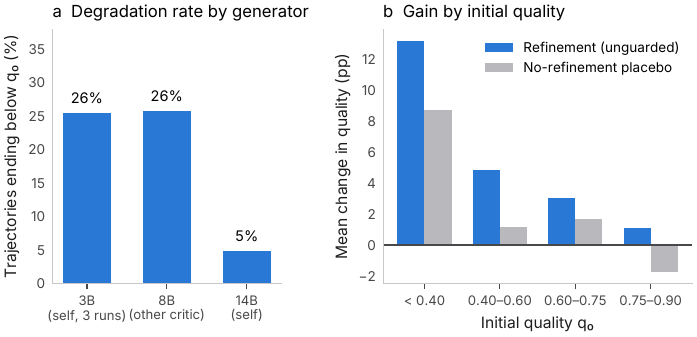}
  \caption{Composition-free evidence for the graded refinement claim of \Cref{prop:refinement-threshold}. Each trajectory is the sequence of scores of one refined answer; its change is the final \emph{unguarded} iterate minus the initial answer $q_0$. \textbf{(a)}~Share of turbo trajectories (with at least one refinement) that end below where they started, on the three local configurations of \Cref{tab:model-configs}: Qwen-2.5 3B critiquing itself (three generation runs, $n{=}145$ on 57 tasks), Llama-3.1 8B with a Qwen-2.5 7B critic ($n{=}58$), and DeepSeek-R1 14B critiquing itself ($n{=}41$); all run $5$ iterations at threshold $\tau{=}0.85$ with Gemma-3 12B as judge. Mean changes are $+6.2$ pp ($p{=}0.007$, tested on per-task means), $+5.1$ pp ($p{=}0.039$), and $+19.7$ pp ($p{=}2{\times}10^{-5}$). \textbf{(b)}~Mean change by initial quality, pooled over turbo and HARQ-IR on ten configurations (blue), against a no-refinement placebo: the change in the uncoded answer's score between two independent generations of the same task, binned by the first one's score (gray). Trajectories that start at $q_0 \ge 0.9$ are not refined, since the operators stop at their threshold.}
  \label{fig:threshold-3panel}
\end{figure}

\begin{figure}[h]
  \centering
  \includegraphics[width=0.95\linewidth]{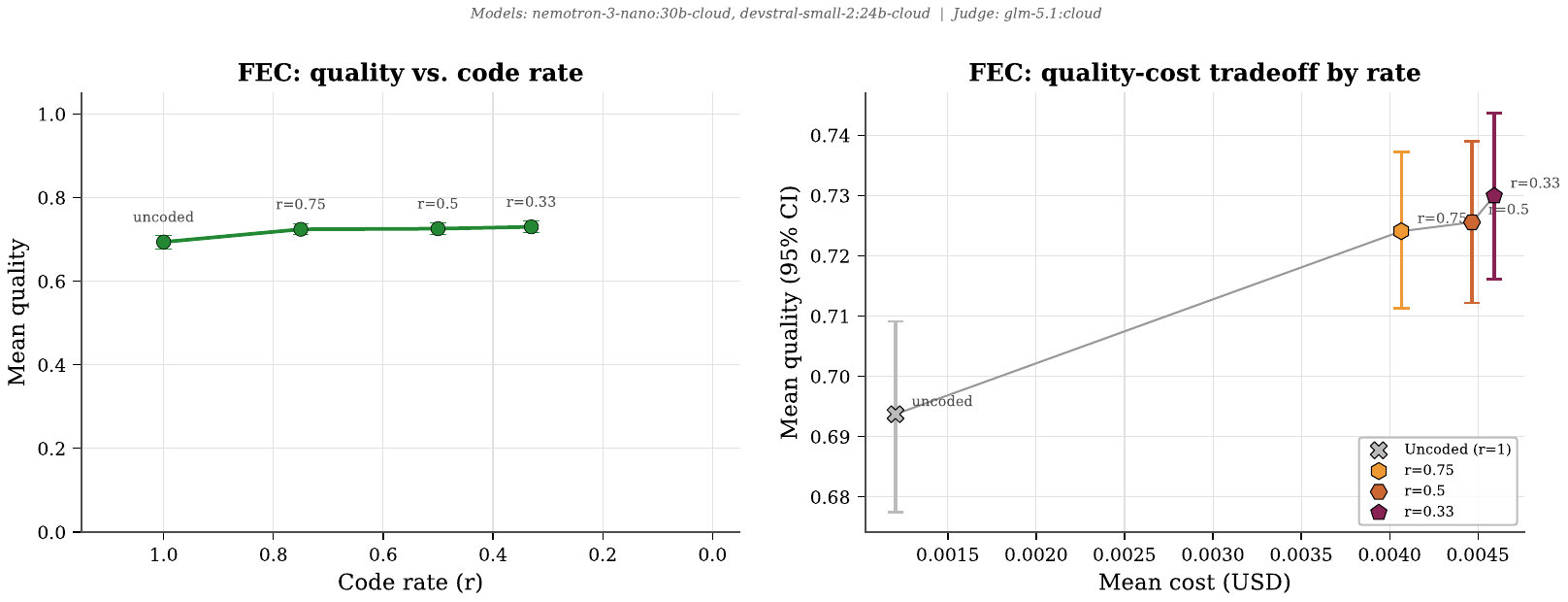}
  \caption{Forward error correction (FEC) on the Ollama-cloud configuration. \emph{Left}: mean quality versus code rate $r$ (with $r{=}1$ the uncoded baseline) with 95\% bootstrap confidence intervals; the curve is monotone in $r$ on this set, with $r{=}3/4$ delivering the best mean quality (one extra step-by-step reasoning section) and lower rates saturating because the syndrome-decoding LLM does not exploit the additional, less-decorrelated parity sections. \emph{Right}: per-task quality versus cost overhead $\rho$ scatter, colored by code rate.}
  \label{fig:fec-rate}
\end{figure}

\begin{figure}[h]
  \centering
  \includegraphics[width=0.85\linewidth]{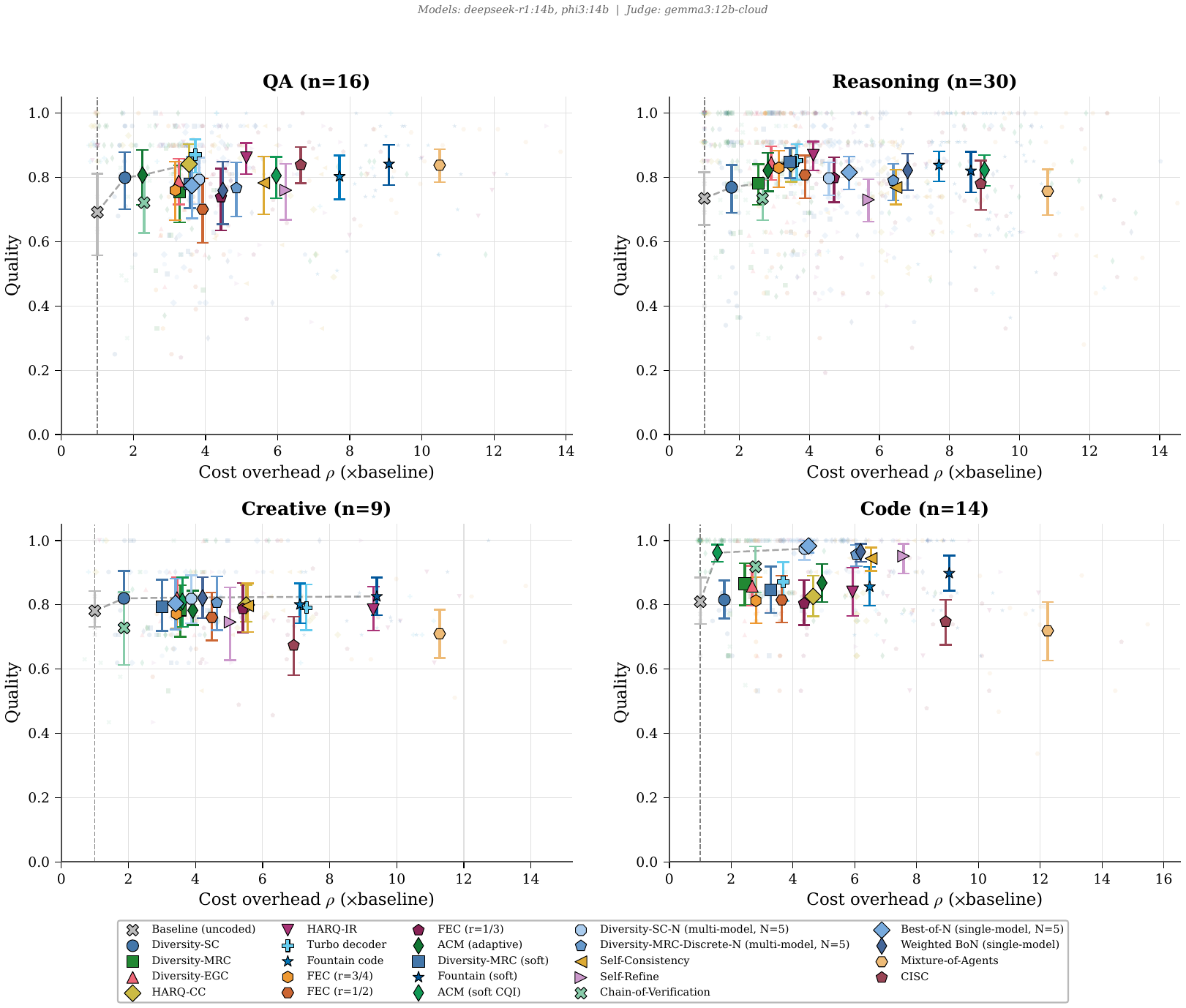}
  \caption{Per-category quality-cost Pareto scatter on the 14B local configuration ($n{=}69$ curated tasks; one panel per task category). The curated tasks cover four categories (question answering, reasoning, creative, code); each panel plots per-task quality versus normalized cost overhead $\rho$ for every technique, with the per-technique mean shown as a large marker with a 95\% bootstrap confidence interval. The \textsc{ACM} marker is the hand-coded router of \Cref{app:impl-acm}; its position differs by category, illustrating why the best choice depends on the task.}
  \label{fig:pareto-acm}
\end{figure}

\begin{figure}[h]
  \centering
  \includegraphics[width=0.95\linewidth]{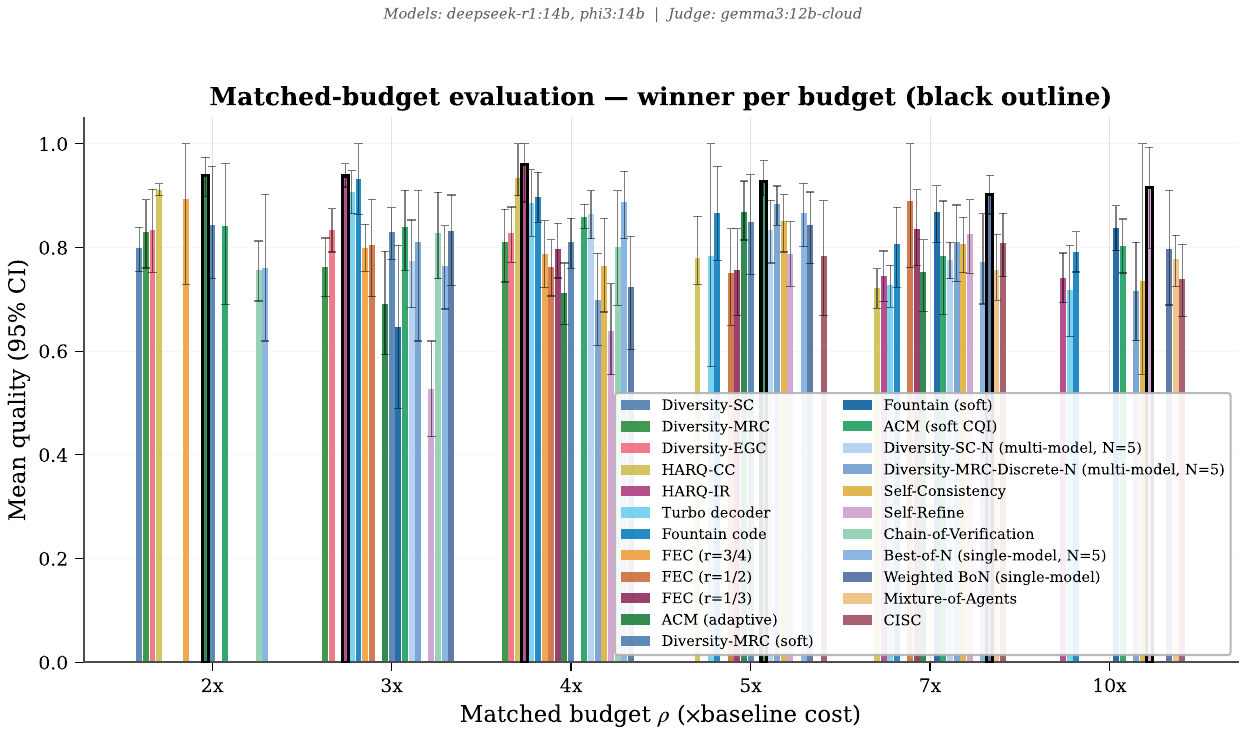}
  \caption{Matched-budget evaluation on the 14B local configuration (cf.\ \citealp{snell2025scaling}): mean quality with 95\% bootstrap confidence interval for each technique, binned by observed cost overhead $\rho \in \{2\times, 3\times, 4\times, 5\times, 7\times, 10\times\}$. The black outline at each budget marks the technique with the highest mean quality. Where \Cref{tab:results-summary} reports each technique at its own operating point, this figure forces a common budget and asks who wins there. Per-configuration matched-budget bars for the other channels are in the corresponding gallery subsections (\Cref{app:figure-gallery}).}
  \label{fig:matched-budget}
\end{figure}

% \begin{figure}[h]
%   \centering
%   \includegraphics[width=\linewidth]{figures/acm_oracle_gap/acm_oracle_gap_cv.pdf}
%   \caption{ACM oracle-gap and Pareto-frontier analysis on the Ollama-cloud configuration (\texttt{ollama\_nemotron\_devstral\_glm51\_datasets}, $n{=}300$ tasks; 5-fold stratified CV; 4000 bootstrap samples). \emph{Left}: per-policy mean quality with 95\% bootstrap CIs over 19 candidate techniques and four router classes (hand-coded ACM, scalar-difficulty bins, $(d,\text{cat})$ logit, ridge on extended features, semantic-KNN at $\lambda \in \{1,5,10,20\}$). The deployed semknn router at $\lambda{=}1$ beats fixed-best by $+0.054$ ($p<10^{-9}$, paired Wilcoxon) at essentially the same cost; sweeping $\lambda$ to $20$ preserves $+0.042$ quality at $51\%$ lower cost. \emph{Right}: additive decomposition $q_{\text{oracle}}-q_{\text{realized}}^{\text{ACM}} = \text{info} + \text{gen.} + \text{policy} + \text{realiz.}$ against the deployed semknn router. The policy term collapses once the router moves from the hand-coded $(d, \text{category})$ bins to semantic-KNN; the binding constraint shifts to the feature-set information limit ($+0.059$) plus the finite-sample generalization gap ($+0.019$). See \Cref{sec:exp-acm-decomp,tab:semknn-pareto} for the full table.}
%   \label{fig:acm-oracle-gap}
% \end{figure}

\begin{figure}[h]
  \centering
  \includegraphics[width=0.95\linewidth]{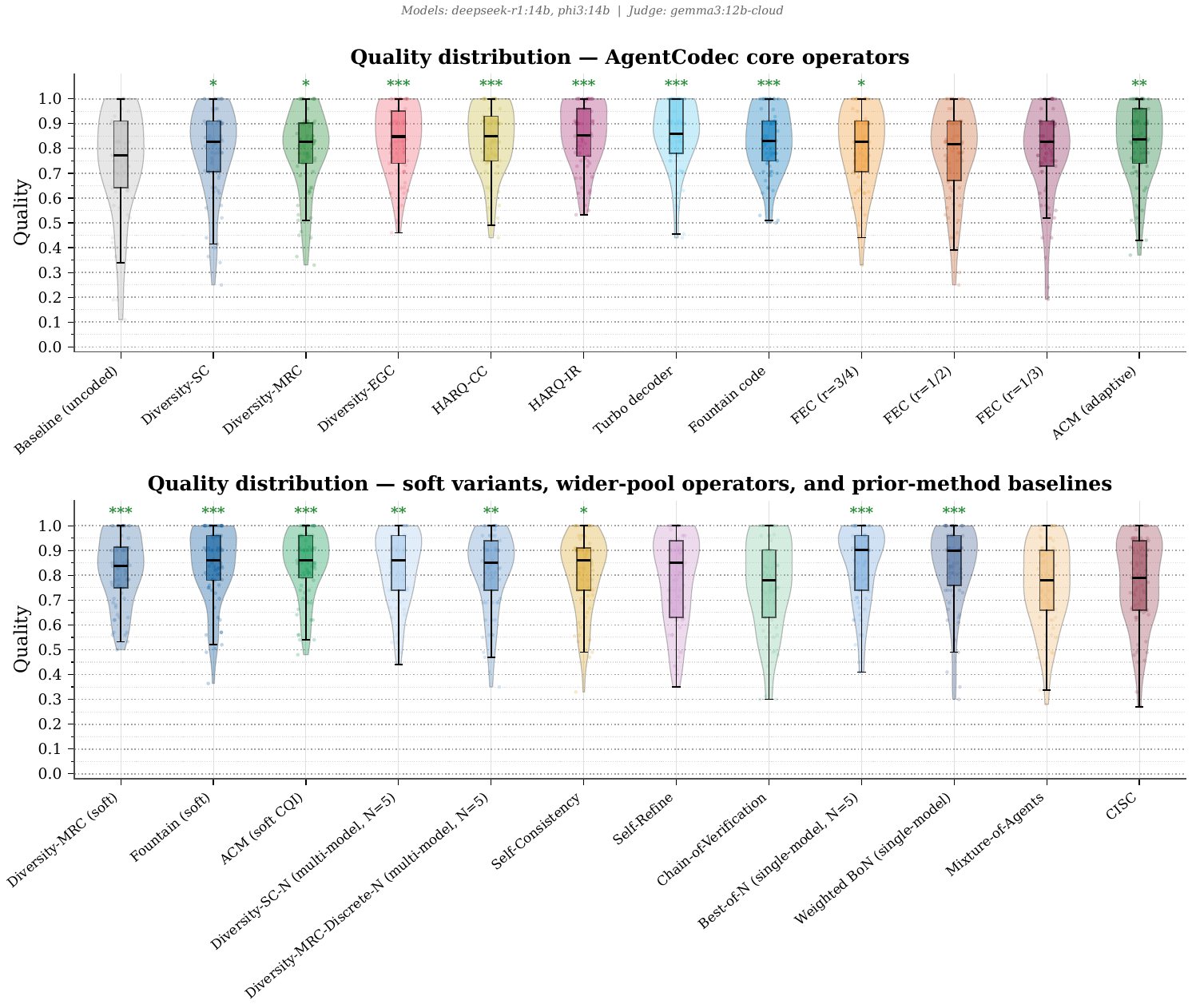}
  \caption{Per-technique quality distributions on the 14B local configuration (violin plots with individual data points and significance stars vs.\ baseline). \emph{Top}: \agentcodec{} techniques (uncoded baseline plus the comm-theoretic operators and soft variants). \emph{Bottom}: prior-method baselines and wider-pool multi-model diversity operators. The figure complements the per-technique-mean view of \Cref{fig:matched-budget} by exposing the spread of per-task quality, which the means alone do not.}
  \label{fig:quality-distribution}
\end{figure}

\begin{figure}[h]
  \centering
  \includegraphics[width=\linewidth]{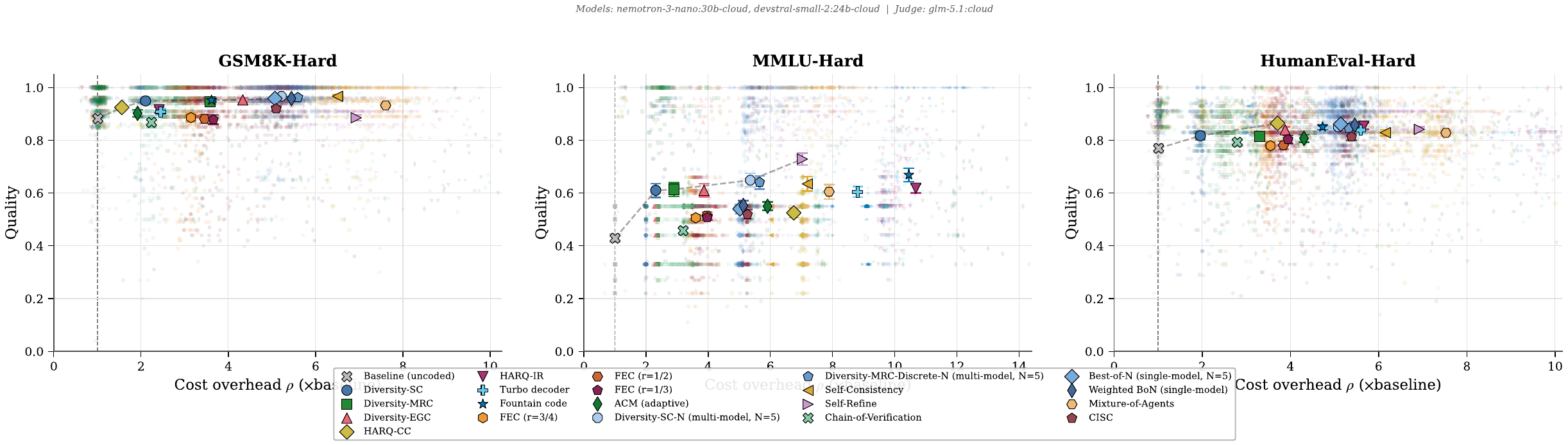}
  \caption{Per-dataset quality versus normalized cost overhead $\rho$ on the 300-task hard-benchmark split (Ollama-cloud trio; one panel per dataset: GSM8K-Hard, MMLU-Hard, HumanEval-Hard). Small markers are individual tasks; large markers with error bars show per-technique means with 95\% bootstrap confidence intervals; the dashed line is the empirical Pareto frontier per panel. The ACM marker is the hand-coded router of \Cref{app:impl-acm}; the semantic-nearest-neighbor router is evaluated in \Cref{tab:semknn-pareto}.}
  \label{fig:hard-benchmark-pareto}
\end{figure}

% ============================================================================
% Detailed ACM oracle-gap analysis (relocated from main text for compactness).
% ============================================================================
\subsection{ACM oracle-gap decomposition: methodology and full table}
\label{sec:exp-acm-decomp}

\paragraph{Setup.} To address potential calibration-leakage concerns for the semantic-nearest-neighbor router (\textsc{AcmSemKNN}, \Cref{sec:exp-acm}), all routing policies are calibrated and scored under $5$-fold cross-validation (CV) stratified by category. The evaluation cache is the Ollama-cloud configuration with $n{=}300$ standard-benchmark tasks generated $n_{\text{rep}}{=}3$ independent times each ($900$ trial-task observations), $240$/$60$ train/test split per fold, and $4000$ bootstrap samples for confidence intervals; the same fold partition is applied across repeats so within-task generation variance is absorbed by the bootstrap. The KNN router is fit on the BAAI/bge-large-en-v1.5 embeddings of the train prompts, with kernel size $k{=}20$ and a single Lagrangian knob $\lambda$ controlling the cost penalty. Five additional routers serve as comparators: \textsc{FixedBest} (CV-deployable, abbreviated FB-CV: argmax mean quality on train), \textsc{CategoryBest} (per-category argmax), \textsc{ScalarDifficultyBins} (per-bin argmax over quantile edges of the train pilot-difficulty distribution), and two parametric ablations of \textsc{AcmSemKNN}: \textsc{AcmRidge} ($L_2$-ridge regression per technique on the extended feature set $[1, d, \mathds{1}_{\text{cat}}, \log\!\#\text{words}, \mathds{1}_{\text{has\_code}}, \mathds{1}_{\text{has\_math}}, \mathds{1}_{\text{has\_numbers}}, \mathds{1}_{\text{has\_question}}, \log\!\#\text{sentences}, \overline{|\text{word}|}]$, dispatch by argmax predicted quality) and \textsc{AcmLearned} ($L_2$-regularized multinomial logit on $[1, d, \mathds{1}_{\text{cat}}]$ predicting the per-task argmax-technique label; see \Cref{sec:acm-learned-impl}). The hand-coded ACM bins are evaluated twice: \emph{simulated}, by applying the routing rules to the cached per-technique quality, and \emph{realized}, by reading the actual \texttt{acm.json} quality under ACM's profile-specific round and branch counts. The realized-vs-oracle gap admits the additive decomposition
\begin{align*}
q_{\text{oracle}} - q_{\text{ACM}}^{\text{real}}
  = \underbrace{(q_{\text{oracle}} - q_{\text{feas}})}_{\text{info gap}}
  + \underbrace{(q_{\text{feas}} - q_{\text{learn}}^{\text{CV}})}_{\text{gen.\ gap}}
  + \underbrace{(q_{\text{learn}}^{\text{CV}} - q_{\text{ACM}}^{\text{sim}})}_{\text{policy gap}}
  + \underbrace{(q_{\text{ACM}}^{\text{sim}} - q_{\text{ACM}}^{\text{real}})}_{\text{realiz.\ gap}},
\end{align*}
where $q_{\text{feas}}$ is the leave-one-out (LOO) in-sample ceiling of \textsc{AcmSemKNN}: the best deterministic dispatch achievable on the kernel without changing the embedding model. Headline numbers on the $n{=}300$ Ollama-cloud cache (full per-policy comparison in \Cref{tab:semknn-pareto-full}):

% \begin{center}\small
% \begin{tabular}{@{}lcc@{}}
% \toprule
% policy & mean quality (95\% CI) & mean cost \\
% \midrule
% Oracle (per-task max)                                   & $0.912$ [$0.900, 0.923$]   & $\$0.00668$ \\
% Feasible (semknn LOO, in-sample)                        & $0.877$ [$0.863, 0.892$]   & $\$0.00735$ \\
% \textsc{AcmSemKNN}, $\lambda{=}3$ (CV, \emph{deployed}) & $0.874$ [$0.859, 0.888$]   & $\$0.00695$ \\
% \textsc{AcmSemKNN}, $\lambda{=}20$ (CV)                 & $0.851$ [$0.836, 0.866$]   & $\$0.00381$ \\
% \textsc{AcmSemKNN}, $\lambda{=}38$ (CV, cost extreme)   & $0.820$ [$0.802, 0.837$]   & $\$0.00290$ \\
% Per-category fixed-best (CV)                            & $0.851$ [$0.836, 0.867$]   & $\$0.00645$ \\
% \textsc{AcmRidge} (CV, ablation)                        & $0.849$ [$0.833, 0.865$]   & $\$0.00660$ \\
% \textsc{AcmLearned} (CV, ablation)                      & $0.846$ [$0.830, 0.861$]   & $\$0.00707$ \\
% Scalar-difficulty bins (CV)                             & $0.837$ [$0.820, 0.855$]   & $\$0.00685$ \\
% Fixed-best CV (FB-CV)                                   & $0.818$ [$0.800, 0.838$]   & $\$0.00698$ \\
% ACM realized (\texttt{acm.json}, in-sample)             & $0.753$                    & $\$0.00473$ \\
% ACM simulated (config $\to$ cache, in-sample)           & $0.694$                    & $\$0.00120$ \\
% \bottomrule
% \end{tabular}
% \end{center}

\begin{table}[!ht]
\centering
\caption{Per-configuration $G$-leader (raw coding gain) versus $\eta$-leader (gain per unit cost overhead) on the 69-task curated split, with $G, \rho, \eta$ as defined in \Cref{app:formal-foundations}. Leaders are taken among the single \agentcodec{} operators of \Cref{tab:results-summary} (Diversity-SC, -MRC, -EGC, HARQ-CC, HARQ-IR, turbo, fountain, and FEC), excluding hand-coded ACM, which routes among them, the wider-pool $N$ variants, and the re-implemented prior methods. $G$ is reported as percent of baseline mean (the convention of \Cref{tab:results-summary}); $\rho$ is the mean per-task cost overhead. The two leaders coincide only on the A+O cloud, where the baseline is already near ceiling.}
\label{tab:winner-vs-eff}
\footnotesize
\setlength{\tabcolsep}{4pt}
\begin{tabular}{@{}l llrr llrr@{}}
\toprule
& \multicolumn{4}{c}{\textbf{$G$-leader (raw coding gain)}} & \multicolumn{4}{c}{\textbf{$\eta$-leader (gain per unit cost)}} \\
\cmidrule(lr){2-5}\cmidrule(lr){6-9}
Configuration & Technique & $G$ & $\rho$ & $\eta$ & Technique & $G$ & $\rho$ & $\eta$ \\
\midrule
3B local      & Fountain  & $+30.0\%$ & $7.4\times$ & $4.1$ & Diversity-SC  & $+19.2\%$ & $2.2\times$ & $8.7$  \\
8B local      & Fountain  & $+32.5\%$ & $9.1\times$ & $3.6$ & Diversity-SC  & $+23.5\%$ & $2.2\times$ & $10.7$ \\
14B local     & Turbo     & $+14.2\%$ & $4.1\times$ & $3.4$ & Diversity-EGC & $+11.6\%$ & $3.0\times$ & $3.8$  \\
Ollama cloud  & Fountain  & $+20.0\%$ & $6.3\times$ & $3.2$ & Diversity-SC  & $+9.4\%$  & $1.9\times$ & $5.0$  \\
A+O cloud     & HARQ-IR   & $+10.7\%$ & $3.6\times$ & $3.0$ & HARQ-IR       & $+10.7\%$ & $3.6\times$ & $3.0$  \\
\bottomrule
\end{tabular}
\end{table}
\vspace{-2mm}

The Oracle, Feasible, simulated-ACM, and realized-ACM rows are scored in-sample on all $n{=}300$ tasks (averaged over $n_{\text{rep}}{=}3$ generation repeats per task); every other row is out-of-fold and pays a generalization tax that the in-sample rows do not. Plugging the deployed router at $\lambda{=}3$ into the decomposition gives
\begin{align*}
q_{\text{oracle}} - q_{\text{ACM}}^{\text{real}} \;=\; +0.159
\;=\; \underbrace{+0.034}_{\text{info gap}}
\;+\; \underbrace{+0.003}_{\text{gen.\ gap}}
\;+\; \underbrace{+0.181}_{\text{policy gap}}
\;+\; \underbrace{(-0.060)}_{\text{realiz.\ gap}}.
\end{align*}

\paragraph{What the decomposition says about the KNN router.} Three observations follow.
\emph{(i)~The policy term dominates the gap, and \textup{\textsc{AcmSemKNN}} closes it.} The hand-coded $(d, \text{category})$ bins were paying a policy gap of $+0.181$ to oracle; replacing the dispatch rule by the embedding kernel reduces the residual to $+0.034$ (information limit) plus $+0.003$ (finite-sample generalization), so essentially the entire policy term is recovered. The deployed router lands within $\approx 0.3$ percentage points of the in-sample feasibility ceiling at slightly lower cost ($\$0.00695$ vs.\ $\$0.00735$).
\emph{(ii)~The cost axis is genuinely orthogonal.} The same router instance traverses the empirical Pareto frontier with one Lagrangian knob and no retraining. At $\lambda{=}38$ it still matches FB-CV in mean quality ($0.820$ vs.\ $0.818$) while running at $\$0.00290$ per task, $\approx 58\%$ below FB-CV cost and $\approx 57\%$ below the per-task oracle cost; the intermediate $\lambda{=}20$ ($0.851$, $\$0.00381$) loses only $\approx 2.3$ percentage points relative to $\lambda{=}3$ at $\approx 45\%$ lower cost.
\emph{(iii)~The hand-coded table is miscalibrated on this split.} In simulation, its difficulty bins send all 300 tasks to the uncoded call ($q{=}0.694$), while the realized run with profile-specific rounds and branches reaches $q{=}0.753$; the large policy gap therefore measures the miscalibration of a hand-written table on a new task distribution, which is the motivation for a learned router, rather than a property of routing in general. The residual quality gap to oracle is therefore information-bound: closing the $+0.034$ information gap requires either richer features (response length, pilot-branch agreement, category-conditioned log-probability) or a stronger channel-state-information signal than the pilot's mean log-probability.

\paragraph{Cross-validation and calibration leakage.} Every test-fold task is dispatched using nearest-neighbor means computed only over its $240$ train-fold neighbors; test-fold per-technique scores are never seen at fit time, and the same fold partition is applied across the $n_{\text{rep}}{=}3$ repeats so within-task generation variance is absorbed by the bootstrap rather than by the train/test split. The deployed router at $\lambda{=}3$ scores $0.874$ [$0.859, 0.888$], $5.6$ percentage points above FB-CV (paired Wilcoxon $p{<}10^{-17}$, Cohen's $d_{z}{=}0.47$, win rate $59.7\%$); against the simulated-ACM hand-coded bins the lead is $+0.179$ ($p{<}10^{-46}$, $d_{z}{=}0.97$, win rate $93\%$). The two parametric ablations on identical splits, \textsc{AcmRidge} ($0.849$) and \textsc{AcmLearned} ($0.846$), lie significantly below the kernel router (paired $p \le 5{\times}10^{-4}$ against the router at $\lambda{=}1$, whose quality equals that at $\lambda{=}3$ to within $0.1$ pp) and significantly above FB-CV ($p{<}10^{-5}$), so the lead over fixed-best is not an artifact of one router class. FB-CV itself collapses to two distinct techniques across the five folds (\texttt{fountain} on $80\%$ of tasks, \texttt{self\_refine} on $20\%$), the structural reason a single fixed technique cannot match a router on this split.

\paragraph{Per-category breakdown.} The deployed kernel sits essentially at the LOO ceiling on every populated category: question answering, $0.800$ vs.\ feasible $0.795$ ($n{=}100$); reasoning, $0.960$ vs.\ $0.969$ ($n{=}100$); code, $0.862$ vs.\ $0.868$ ($n{=}100$). Most of the residual headroom to oracle is therefore information-bound rather than dispatcher-bound: oracle reaches $0.841$ on QA and $0.913$ on code, so closing the per-category gap requires richer features or a better CSI signal, not a richer router class. The creative category is empty ($n{=}0$) on the standard-benchmark split.

\paragraph{Pick-frequency structure of the $\lambda$-sweep.} The dispatch profile of each router (\Cref{fig:acm-technique-frequency}) is informative in its own right. At $\lambda{=}3$ the deployed kernel uses $12$ distinct techniques, with no single share above $25.3\%$ (\texttt{harq\_cc}, $76/300$); the next four shares are $\texttt{self\_refine}$ ($20.7\%$), $\texttt{diversity\_sc\_N}$ ($19.7\%$), $\texttt{fountain}$ ($12.7\%$), and $\texttt{best\_of\_n}$ ($7.7\%$). As $\lambda$ rises the dispatch contracts: at $\lambda{=}20$ it picks four techniques dominated by $\texttt{diversity\_sc}$ ($186/300$, $62.0\%$); at the cost extreme $\lambda{=}38$ it collapses to three, with $\texttt{diversity\_sc}$ ($50.3\%$) and the uncoded $\texttt{baseline}$ ($32.0\%$) covering $\approx 82\%$ of tasks while $\texttt{self\_refine}$ ($17.7\%$) handles the residual hard ones, exactly the cost-penalty behavior the Lagrangian knob predicts: prune expensive variants on easy tasks where a cheap multi-model diversity operator (or the uncoded baseline) is already near-oracle. The two parametric ablations dispatch differently still: \textsc{AcmRidge} concentrates on $\texttt{self\_refine}$ ($32.7\%$) and $\texttt{diversity\_sc\_N}$ ($26.3\%$) across seven techniques, while \textsc{AcmLearned} collapses harder, picking five techniques led by $\texttt{self\_refine}$ ($53.7\%$) and $\texttt{best\_of\_n}$ ($35.0\%$). The KNN router's quality lead therefore co-occurs with a markedly more diverse dispatch profile, consistent with the embedding kernel exploiting per-task structure that the linear ablations cannot.

We report the decomposition for this split only, the largest evaluated split and the one with the most candidate techniques; \Cref{app:router-extended} reports the router itself on two further 300-task configurations.

\subsection{Full per-policy semantic-nearest-neighbor table}
\label{sec:acm-tables-recalibrated}

\begin{table}[!ht]
\centering
\caption{Full per-policy comparison underlying the headline of \Cref{tab:semknn-pareto}: cost-aware semantic-nearest-neighbor adaptive-coding-and-modulation router vs.\ standard router baselines on the Ollama-cloud trio (Nemotron-Nano-3 + Devstral-Small-2 channels; GLM-5.1 judge). Each task is generated $n_{\text{rep}}{=}3$ independent times: $n{=}300$ unique tasks $\times\,3$ repeats $=900$ trial-task observations, with the same fold partition applied to every repeat (no information leakage across repeats); 5-fold stratified cross-validation for every row marked ``CV'' and for the 14 semknn rows; $4000$ bootstrap samples drawn with seed-aware resampling that propagates within-task generation variance into every confidence interval, so the $\sqrt{3}$-tighter intervals reported below already absorb run-to-run noise rather than hiding it. The two ``$\Delta q$'' columns report the deployed router's lead over (i)~the cross-validated fixed-best policy (the apples-to-apples baseline, since both are out-of-fold) and (ii)~the strongest single technique scored in-sample on all $300\times 3$ trial-task observations (\texttt{fountain}, $q{=}0.824$); both are paired by task (repeats averaged); only the FB-CV delta is tested (Wilcoxon signed-rank; $^{***}$: lead, $p{<}10^{-3}$).}
\label{tab:semknn-pareto-full}
\scriptsize
\setlength{\tabcolsep}{2pt}
\begin{tabular}{@{}lccccccc@{}}
\toprule
\textbf{Policy} & \textbf{$q$ [95\% CI]} & \textbf{Cost / task} & $\bm{q/\$}$ & $\bm{\rho}$ & $\bm{G}$ \textbf{(\%)} & $\bm{\Delta q}$ \textbf{vs.\ FB-CV} & $\bm{\Delta q}$ \textbf{vs.\ FB-IS}\\
\midrule
Oracle (per-task max) & $0.912$ [$0.900, 0.923$] & $\$0.00668$ & $165.4$ & $5.75\times$ & $+31.4$ & n/a & $+0.088$ \\
Feasible (semknn LOO, in-samp.) & $0.877$ [$0.863, 0.892$] & $\$0.00735$ & $145.5$ & $6.48\times$ & $+26.5$ & n/a & $+0.053$ \\
Semknn router, $\lambda{=}1$ & $0.873$ [$0.858, 0.888$] & $\$0.00738$ & $147.1$ & $6.51\times$ & $+25.8$ & $+0.055^{***}$ & $+0.049$ \\
\textbf{Semknn router, $\lambda{=}3$} & $\bm{0.874}$ [$0.859, 0.888$] & $\bm{\$0.00695}$ & $\bm{160.2}$ & $\bm{6.20\times}$ & $\bm{+26.0}$ & $\bm{+0.056^{***}}$ & $\bm{+0.050}$ \\
Semknn router, $\lambda{=}4$ & $0.873$ [$0.858, 0.887$] & $\$0.00656$ & $176.5$ & $5.90\times$ & $+25.8$ & $+0.054^{***}$ & $+0.049$ \\
Semknn router, $\lambda{=}5$ & $0.871$ [$0.857, 0.885$] & $\$0.00600$ & $204.8$ & $5.48\times$ & $+25.6$ & $+0.053^{***}$ & $+0.047$ \\
Semknn router, $\lambda{=}7$ & $0.870$ [$0.855, 0.884$] & $\$0.00552$ & $230.3$ & $5.09\times$ & $+25.4$ & $+0.052^{***}$ & $+0.046$ \\
Semknn router, $\lambda{=}10$ & $0.867$ [$0.852, 0.881$] & $\$0.00521$ & $242.4$ & $4.79\times$ & $+24.9$ & $+0.048^{***}$ & $+0.043$ \\
Semknn router, $\lambda{=}15$ & $0.857$ [$0.842, 0.872$] & $\$0.00422$ & $272.4$ & $3.76\times$ & $+23.6$ & $+0.039^{***}$ & $+0.033$ \\
Semknn router, $\lambda{=}20$ & $0.851$ [$0.836, 0.866$] & $\$0.00381$ & $287.2$ & $3.42\times$ & $+22.7$ & $+0.033^{***}$ & $+0.027$ \\
Semknn router, $\lambda{=}35$ & $0.829$ [$0.812, 0.846$] & $\$0.00309$ & $362.5$ & $2.87\times$ & $+19.5$ & $+0.011$ & $+0.005$ \\
\textbf{Semknn router, $\lambda{=}38$} & $\bm{0.820}$ [$0.802, 0.837$] & $\bm{\$0.00290}$ & $\bm{388.7}$ & $\bm{2.71\times}$ & $\bm{+18.2}$ & $\bm{+0.001}$ & $\bm{-0.004}$ \\
Semknn router, $\lambda{=}40$ & $0.808$ [$0.788, 0.827$] & $\$0.00271$ & $404.5$ & $2.51\times$ & $+16.4$ & $-0.011$ & $-0.016$ \\
Semknn router, $\lambda{=}50$ & $0.758$ [$0.734, 0.780$] & $\$0.00167$ & $505.6$ & $1.44\times$ & $+9.2$ & $-0.061$ & $-0.067$ \\
Semknn router, $\lambda{=}100$ & $0.748$ [$0.725, 0.770$] & $\$0.00142$ & $565.9$ & $1.24\times$ & $+7.8$ & $-0.071$ & $-0.076$ \\
Semknn router, $\lambda{=}1000$ & $0.694$ [$0.668, 0.719$] & $\$0.00120$ & $577.9$ & $1.00\times$ & $+0.0$ & $-0.125$ & $-0.130$ \\
Ridge regressor, extended feats & $0.849$ [$0.833, 0.865$] & $\$0.00660$ & $146.9$ & $5.69\times$ & $+22.4$ & $+0.031^{***}$ & $+0.025$ \\
Logistic regression, $(d,\text{cat})$ & $0.846$ [$0.830, 0.861$] & $\$0.00707$ & $130.5$ & $5.99\times$ & $+21.9$ & $+0.027^{***}$ & $+0.022$ \\
Per-category fixed-best (CV) & $0.851$ [$0.836, 0.867$] & $\$0.00645$ & $157.7$ & $5.57\times$ & $+22.7$ & $+0.033^{***}$ & $+0.027$ \\
Scalar-difficulty bins (CV) & $0.837$ [$0.820, 0.855$] & $\$0.00685$ & $136.1$ & $5.81\times$ & $+20.7$ & $+0.019^{***}$ & $+0.013$ \\
\emph{Strongest in-sample fixed: \texttt{fountain}} & $0.824$ & $\$0.00691$ & $119.2$ & $6.23\times$ & $+18.8$ & $+0.006$ & (FB-IS ref.) \\
\textbf{Cross-validated fixed-best (FB-CV)} & $\bm{0.818}$ [$0.800, 0.838$] & $\bm{\$0.00698}$ & $\bm{143.6}$ & $\bm{6.13\times}$ & $\bm{+18.0}$ & (FB-CV ref.) & $\bm{-0.006}$ \\
Always-baseline (uncoded) & $0.694$ [$0.668, 0.719$] & $\$0.00120$ & $577.9$ & $1.00\times$ & $+0.0$ & $-0.125$ & $-0.130$ \\
\bottomrule
\end{tabular}
\end{table}

The ``feasible'' row in \Cref{tab:semknn-pareto-full} is the leave-one-out in-sample upper bound on the same router class (semantic-nearest-neighbor applied to the full cache); it bounds what additional features or richer pools could in principle deliver without changing the dispatch rule. The router at $\lambda{=}3$ is within $\approx 0.3$ percentage points of feasible at slightly lower cost, so the residual gap to oracle is dominated by the feature-set information limit, not by the dispatcher. The $\lambda$-sweep is genuine: at $\lambda{=}38$ the same router instance still matches FB-CV in mean quality while running at $\approx \costimp\%$ lower normalized cost per task than fixed-best, sliding the operating point all the way along the empirical Pareto frontier with one Lagrangian knob. With $n_{\text{rep}}{=}3$ independent generation repeats per task, the bootstrap CIs in the table absorb within-task run-to-run variance rather than hiding it, and every $\lambda \le 20$ row clears Wilcoxon $p<10^{-3}$ vs.\ FB-CV.

\begin{figure}[h]
  \centering
  \includegraphics[width=0.95\linewidth]{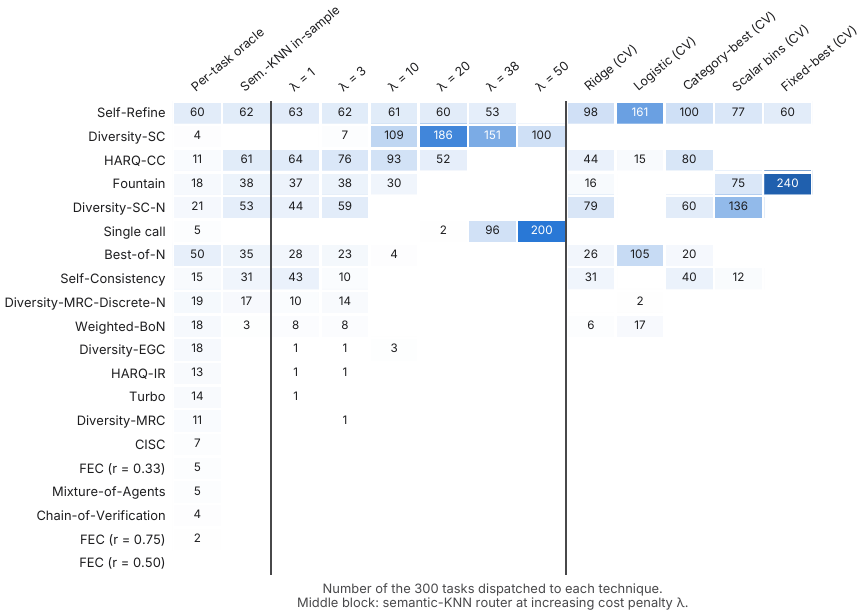}
  \caption{Technique pick frequency per routing policy on the 300-task Ollama-cloud split: each cell counts the tasks (of 300) that a policy dispatches to a technique. As $\lambda$ rises, the router's support contracts from 11 techniques at $\lambda{=}1$ and 12 at $\lambda{=}3$ (led by HARQ-CC, Self-Refine, and Diversity-SC-$N$) to three at $\lambda{=}38$ (Diversity-SC on $151$ tasks, the uncoded call on $96$, Self-Refine on $53$): expensive operators are pruned on tasks where a cheap one is already close to the oracle. The parametric routers concentrate on fewer techniques (the logistic router sends $161$ tasks to Self-Refine), and the cross-validated fixed-best policy uses fountain in four folds and Self-Refine in one.}
  \label{fig:acm-technique-frequency}
\end{figure}

\subsection{The \texttt{acm\_learned} technique}
\label{sec:acm-learned-impl}

\texttt{acm\_learned} serves two roles in this paper, neither of which is the headline result. First, it is a \emph{parametric ablation} of the semantic-nearest-neighbor router (\Cref{sec:router}): identical cross-validation protocol and candidate set, but a multinomial logit on $[1, d, \mathds{1}_{\text{cat}}]$ in place of an embedding-kernel mean. Its out-of-fold mean of $0.846$ on the 300-task split (\Cref{tab:semknn-pareto-full}) lies $2.8$ percentage points below the router at $\lambda{=}3$ and $2.7$ points above the cross-validated fixed-best policy, so the router's lead over fixed-best is not an artifact of one router class. Second, it shows that bin edges which an off-the-shelf logit recovers from any calibration cache need not be hand-tuned per deployment.

The logit is fit per fold with an $L_2$ penalty on the features $[1, d, \mathds{1}_{\text{cat}}]$ (bias, pilot difficulty $d$, one-hot category) to predict the per-task argmax technique of the training fold; each technique is then executed with the same rounds, branches, and code rate as its standalone version (\Cref{app:hyperparams}). Its comparison with the router is methodologically symmetric: identical folds, candidate set, and dispatch, differing only in how task features map to a technique, a logit versus a kernel mean over the $20$ nearest BAAI/bge-large-en-v1.5 prompt embeddings.

\section{Additional configurations used for validation}
\label{app:extra-configs}
Four runs beyond \Cref{tab:model-configs} support the validation studies; all use the 300-task split of \Cref{app:data-hard}, persist every output's text, and grade every output with the judge-free grader of \Cref{app:data-hard}.
\begin{itemize}[nosep,leftmargin=*]
  \item \emph{Ollama-cloud regeneration}: Nemotron-Nano-3 30B and Devstral-Small (served as \texttt{mistralai/Devstral-Small-2507} from a private endpoint after the Ollama-cloud Devstral model was retired), judge GLM-5.1; one generation per task; 20 techniques plus three single-call controls (each channel alone, and Nemotron with reasoning on), 23 in total.
  \item \emph{Nemotron + Gemma-4}: Nemotron-Nano-3 30B and Gemma-4 31B, judge GLM-5.1; one generation per task; the reasoning-disabled arm of \Cref{app:reasoning}.
  \item \emph{Nemotron + Gemma-4, reasoning on}: the same pair with internal reasoning enabled on both channels for every technique; judge unchanged, reasoning off.
  \item \emph{Gemma-4 + GPT-OSS}: Gemma-4 31B and GPT-OSS 20B, judge GLM-5.2; three generations per task.
\end{itemize}

\section{Judge validation}
\label{app:judge-validation}

\paragraph{Judge versus judge-free grading.} On the Ollama-cloud regeneration (\Cref{app:extra-configs}), $6{,}900$ outputs (23 techniques $\times$ 300 tasks; $84\%$ objectively correct) were graded by the judge-free grader. The judge ranks a correct output above an incorrect one with probability $0.771$ (AUC), and its mean score is $0.802$ on correct and $0.556$ on incorrect outputs. \Cref{tab:judge-vs-objective} lists both metrics per technique. Across techniques the Spearman correlation is $0.374$ ($p{=}0.079$); 21 of 23 techniques fall in an objective band from $0.78$ to $0.91$. Self-Refine is ranked first by the judge and 21st objectively. We checked that this is not an extraction artifact: only 2 of its 100 MMLU outputs lack an extractable letter, and none of its MMLU failures states the correct letter in a form the grader missed.

\begin{table}[h]
\centering
\scriptsize
\setlength{\tabcolsep}{4pt}
\caption{Mean judge score and judge-free accuracy per technique on the Ollama-cloud regeneration (300 tasks), sorted by judge score. ``Single call (2)'' is Devstral alone; ``single call, reasoning'' is Nemotron with reasoning on.}
\label{tab:judge-vs-objective}
\begin{tabular}{@{}lcc@{\hspace{14pt}}lcc@{}}
\toprule
\textbf{Technique} & \textbf{Judge} & \textbf{Objective} & \textbf{Technique} & \textbf{Judge} & \textbf{Objective} \\
\midrule
Self-Refine & 0.827 & 0.783 & Diversity-MRC & 0.763 & 0.837 \\
Diversity-SC-$N$ & 0.809 & 0.907 & Diversity-EGC & 0.761 & 0.830 \\
Mixture-of-Agents & 0.807 & 0.890 & ACM (hand-coded) & 0.759 & 0.823 \\
HARQ-IR & 0.792 & 0.857 & Diversity-SC & 0.747 & 0.803 \\
Best-of-$N$ & 0.790 & 0.883 & Single call, reasoning & 0.745 & 0.893 \\
Diversity-MRC-Discrete-$N$ & 0.790 & 0.847 & CISC & 0.740 & 0.820 \\
Weighted Best-of-$N$ & 0.790 & 0.857 & FEC ($r{=}1/3$) & 0.730 & 0.870 \\
Fountain & 0.788 & 0.857 & FEC ($r{=}3/4$) & 0.727 & 0.883 \\
Turbo & 0.777 & 0.853 & FEC ($r{=}1/2$) & 0.723 & 0.877 \\
Self-Consistency & 0.777 & 0.853 & Chain-of-Verification & 0.720 & 0.793 \\
HARQ-CC & 0.776 & 0.843 & Single call (2) & 0.709 & 0.750 \\
 & & & Single call & 0.701 & 0.727 \\
\bottomrule
\end{tabular}
\end{table}

\paragraph{Second and weaker judges.} DeepSeek-V4-Flash, a judge from a different vendor that never enters any technique, rescored all $6{,}900$ outputs. On identical outputs its technique ranking agrees with GLM-5.1 at Spearman $0.96$ ($p{<}10^{-12}$), and it correlates with objective accuracy at $0.33$ ($p{=}0.13$), as weakly as GLM-5.1 does. Gemma-3 12B, smaller than both channels it grades, rescored the outputs of 90 tasks (all 23 techniques, $2{,}070$ outputs): it reproduces GLM-5.1's ranking at Spearman $0.84$ ($p{=}6{\times}10^{-7}$), agrees on the top two techniques, and spans the same range of technique means ($0.133$ versus $0.136$), so it does not compress scores toward the mean.

\paragraph{The 14B channels under two judges.} The 14B pair (DeepSeek-R1 14B, Phi-3 14B) was run once with Gemma-3 12B as judge and once with Gemma-4 31B. The two runs were generated about nine days apart and did not persist output text, so generation was not held fixed and both judge and sample change between them. Over the 22 techniques common to both, the technique rankings agree at Spearman $0.50$ ($p{=}0.017$). The larger judge is much stricter (single-call mean $0.444$ versus $0.746$) and spreads technique means $2.7$ times more widely ($0.287$ versus $0.106$); HARQ-IR is among the top three under both. Under the larger judge the synthesis combiners follow the textbook order $\mathrm{MRC} \ge \mathrm{EGC} \ge \mathrm{SC}$ ($0.631$, $0.627$, $0.598$), the direction \Cref{prop:csi-crossover} predicts for a less noisy scorer, whereas under the smaller judge equal-gain leads (EGC $0.832$, MRC $0.793$, SC $0.792$).

\paragraph{Blend-weight sweep.} With the objective grade available, a blended evaluation metric $w \cdot q_{\text{obj}} + (1-w) \cdot q_{\text{judge}}$ can be swept after the fact (\Cref{tab:blend-sweep}). The technique ordering is insensitive to $w$ over $[0.4, 0.8]$ (Spearman at least $0.91$ against $w{=}0.6$), the top and bottom techniques are the same for every $w \ge 0.2$, and 10 of the 23 techniques move three or more places across that range, so mid-table orderings are weight-sensitive. Using such a blend to compute the router's per-technique estimates on calibration tasks only, which leaks nothing about the task being served, changes the router's objective accuracy by $+1.3$ to $+2.7$ pp at five of six cost settings and by $-0.3$ pp at the sixth, none significant ($p \ge 0.22$; 11-operator pool).

\begin{table}[h]
\centering
\small
\caption{Technique ordering under a blended evaluation metric (Ollama-cloud regeneration, 23 techniques).}
\label{tab:blend-sweep}
\begin{tabular}{@{}lcccccccc@{}}
\toprule
$w$ & 0.0 & 0.2 & 0.4 & 0.5 & 0.6 & 0.7 & 0.8 & 1.0 \\
\midrule
Spearman vs.\ $w{=}0.6$ & 0.701 & 0.793 & 0.912 & 0.979 & 1.000 & 0.978 & 0.920 & 0.884 \\
Top technique & Self-Refine & \multicolumn{7}{c}{Diversity-SC-$N$} \\
\bottomrule
\end{tabular}
\end{table}

\paragraph{The router under judge-free grading.} \Cref{tab:router-objective} compares the router with the cross-validated fixed-best policy, cost-penalized at each $\lambda$ as in \Cref{sec:exp-acm}, when both select using judge scores only, as at serving time, and are graded on what they picked (5-fold, out of fold, paired over 300 tasks). With the 20-candidate pool the router leads at five of six settings; the fixed policy commits to the judge's favorite, Self-Refine, which is second-weakest objectively, ahead only of the uncoded call. Restricted to the 11 \agentcodec{} operators, where the judge's favorite is objectively sound, the two are statistically indistinguishable. On Gemma-4 + GPT-OSS, whose judge favorite (CISC) is also objectively strong, fixed-best leads at $\lambda \le 10$ and the router at $\lambda \ge 20$. The per-task oracle reaches $0.990$ on the regeneration, against $0.907$ for the best fixed technique (Diversity-SC-$N$).

\begin{table}[h]
\centering
\scriptsize
\setlength{\tabcolsep}{3pt}
\caption{Judge-free accuracy of the router (R) and the cross-validated fixed-best policy (F, cost-penalized at each $\lambda$), both selecting on judge scores. Paired Wilcoxon over 300 tasks: $^{*}p{<}0.05$, $^{**}p{<}0.01$, $^{***}p{<}10^{-3}$ for the leading policy.}
\label{tab:router-objective}
\begin{tabular}{@{}lcccccc@{}}
\toprule
& \multicolumn{2}{c}{\textbf{Regeneration, 20 candidates}} & \multicolumn{2}{c}{\textbf{Regeneration, 11 operators}} & \multicolumn{2}{c}{\textbf{Gemma-4 + GPT-OSS, 20 cand.}} \\
\cmidrule(lr){2-3}\cmidrule(lr){4-5}\cmidrule(l){6-7}
$\lambda$ & R & F & R & F & R & F \\
\midrule
0 & $0.877^{***}$ & $0.783$ & $0.813$ & $0.840$ & $0.941$ & $0.959^{*}$ \\
1 & $0.877^{***}$ & $0.783$ & $0.817$ & $0.840$ & $0.934$ & $0.959^{*}$ \\
3 & $0.880^{***}$ & $0.783$ & $0.823$ & $0.847$ & $0.934$ & $0.959^{*}$ \\
10 & $0.873$ & $0.900$ & $0.833$ & $0.843$ & $0.933$ & $0.959^{**}$ \\
20 & $0.867^{**}$ & $0.803$ & $0.827$ & $0.803$ & $0.939^{***}$ & $0.887$ \\
38 & $0.830^{*}$ & $0.783$ & $0.790$ & $0.783$ & $0.934^{***}$ & $0.887$ \\
\bottomrule
\end{tabular}
\end{table}

\paragraph{Where the gains come from.} Graded objectively on the regeneration: \emph{candidate-pool diversity} is the largest robust contributor, with Diversity-SC, -EGC, and -MRC above the single call by $7.7$, $10.3$, and $11.0$ pp ($10.6\%$ to $15.1\%$ relative; $p \le 0.003$). Most of it comes from having a candidate pool at all: at matched budget, the multi-channel Diversity-SC-$N$ exceeds the single-channel Best-of-$N$ by $2.3$ pp ($p{=}0.22$; $3.6$ pp under the judge on the three-repeat split, $p{=}0.052$). \emph{Operator design} is mixed: HARQ-IR beats Self-Refine by $7.3$ pp ($p{=}0.008$), Diversity-MRC-Discrete-$N$ is ahead of CISC by $2.7$ pp ($p{=}0.29$; $5.1$ pp under the judge on this run, $p{<}10^{-12}$), and Diversity-EGC trails Mixture-of-Agents by $6.0$ pp on this run ($p{=}0.009$). \emph{Judge strength} does not drive the ordering (agreement $0.96$ across vendors, $0.84$ with a weaker judge). \emph{Routing} is reported above.

\paragraph{Diversity-EGC versus Mixture-of-Agents across configurations.} On the judge metric, Diversity-EGC minus Mixture-of-Agents is $+1.25$ pp on the three-repeat headline split ($p{<}10^{-5}$), $-4.68$ pp on the regeneration ($p{=}0.03$), $-0.14$ pp on Nemotron + Gemma-4 (not significant), and $+0.08$ pp on Gemma-4 + GPT-OSS (not significant), while it is $47\%$, $51\%$, $50\%$, and $54\%$ cheaper respectively. The cost advantage is general; the quality difference is configuration-specific.

\paragraph{The judge's share of cost.} Judge calls are included in every cost we report. At the flat per-token rates of our accounting (each model priced equally for input and output tokens), the judge accounts for $83\%$ of per-task cost on both the headline split and the regeneration. Judge calls are input-dominated (the judge reads the task and the candidate, plus the reference when scoring final outputs, and returns fifteen yes/no verdicts), whereas generation is output-heavy, so under price lists that charge more for output than for input tokens the judge's share is correspondingly lower.

\section{Run-to-run noise and the resolution of single runs}
\label{app:noise}
\Cref{tab:noise} measures single-run noise directly on the four configurations generated three times. Single runs rank techniques consistently, and the single-run standard deviation of a technique mean is a fifth to a quarter of the between-technique spread on the three configurations with the larger spread. The paired standard error of a difference between two techniques within one run sets the resolution: about $2$ pp on 300 tasks and about $4$ pp on 69 tasks. Within the three-repeat headline configuration, the standard error of a three-run mean is $0.004$, about one tenth of the spread. Across generation campaigns run weeks apart on cloud models (the headline cache and the regeneration, 21 common techniques) agreement drops to Spearman $0.77$, with a mean absolute difference in technique means of $0.018$ against a spread of $0.042$; the regeneration also served the second channel from a different endpoint and checkpoint name (\Cref{app:extra-configs}), so this figure mixes model change with sampling noise. The 3B local configuration isolates the effect of time: its first run and its two later runs were generated six months apart on unchanged local model weights, and they agree at Spearman $0.89$ and $0.91$, as closely as the two later runs agree with each other ($0.90$; techniques of \Cref{tab:results-summary}).

\begin{table}[h]
\centering
\footnotesize
\caption{Single-run noise measured from three independent generations per task.}
\label{tab:noise}
\begin{tabular}{@{}lccccc@{}}
\toprule
\textbf{Configuration} & \textbf{Tasks} & \textbf{Spearman between runs} & \textbf{SD of mean (1 run)} & \textbf{Spread} & \textbf{Resolution} \\
\midrule
3B local & 69 & 0.89, 0.90, 0.91 & 0.0163 & 0.0638 & ${\approx}4$ pp \\
3B local + cloud judge & 69 & 0.92, 0.95, 0.96 & 0.0126 & 0.0690 & ${\approx}4$ pp \\
Ollama-cloud trio & 300 & 0.91, 0.93, 0.94 & 0.0070 & 0.0396 & ${\approx}2$ pp \\
Gemma-4 + GPT-OSS & 300 & 0.90, 0.92, 0.96 & 0.0021 & 0.0236 & ${\approx}1$ pp \\
\bottomrule
\end{tabular}
\end{table}

\section{Evidence for the CSI-noise crossover}
\label{app:prop-evidence}

\paragraph{Controlled noise sweep.} The discrete combiner is a deterministic function of the candidate answers and their weights, both of which are persisted, so it can be re-decided offline under any weight noise; at uniform weights it reduces to Self-Consistency, which supplies equal-gain combining on the identical pool. On the 200 MMLU and GSM8K tasks of the regeneration (HumanEval answers are not discrete), with five candidates per task, we add Gaussian noise of standard deviation $\sigma$ to the judge scores, recombine, and grade objectively, with $1{,}500$ draws per noise level. $\sigma$ is expressed in units of the within-task score spread $\sigma_s$ ($0.078$ to $0.087$ across techniques). \Cref{tab:prop1-sweep} shows that the maximal-ratio advantage decays and changes sign once, while equal-gain combining is unaffected since it ignores the scores; the crossing replicates across the three techniques that persist five candidates, at $13.0$, $12.4$, and $12.2$ times $\sigma_s$.

\begin{table}[h]
\centering
\small
\caption{Maximal-ratio minus equal-gain accuracy (pp) for the discrete combiner as score noise grows (Diversity-MRC-Discrete-$N$ pool; the other two pools give the same shape).}
\label{tab:prop1-sweep}
\begin{tabular}{@{}lccccccc@{}}
\toprule
Noise $\sigma / \sigma_s$ & 0 & 3 & 5 & 7.5 & 10 & 15 & 30 \\
\midrule
MRC $-$ EGC (pp) & $+6.5$ & $+6.0$ & $+4.4$ & $+2.5$ & $+1.1$ & $-0.7$ & $-2.7$ \\
\bottomrule
\end{tabular}
\end{table}

\paragraph{Synthesis versus discrete combining.} \Cref{tab:prop1-spaces} compares equal-gain and maximal-ratio combining in both answer spaces on every configuration that ran both pairs. With LLM synthesis, equal-gain is significantly ahead on the 14B pair and on the three-repeat Ollama-cloud split; with discrete combining, every significant difference favors maximal-ratio. A vote redistributes weight among answer groups and is dominated by group size, whereas a weighted synthesis lets one mis-weighted branch corrupt the combined answer directly; \Cref{prop:csi-crossover} is stated for a linear combiner, which is the synthesis case. Where the synthesis crossover sits was identified by this comparison across configurations, not by a within-configuration sweep, which synthesis does not permit offline.

\begin{table}[h]
\centering
\small
\caption{Equal-gain minus maximal-ratio combining (pp, judge metric, paired per task, repeats averaged); positive means past the crossover.}
\label{tab:prop1-spaces}
\begin{tabular}{@{}lccc@{}}
\toprule
\textbf{Configuration} & \textbf{Tasks} & \textbf{Synthesis} & \textbf{Discrete} \\
\midrule
3B local (3 repeats) & 69 & $+0.44$ ($p{=}0.98$) & $-0.25$ ($p{=}0.45$) \\
8B local & 69 & $+2.21$ ($p{=}0.07$) & $+0.19$ ($p{=}0.94$) \\
14B local & 69 & $+3.99$ ($p{=}0.002$) & $-0.94$ ($p{=}0.29$) \\
14B local, Gemma-4 31B judge & 69 & $-0.44$ ($p{=}0.96$) & $+2.77$ ($p{=}0.76$) \\
Haiku 4.5 + GPT-5-mini & 69 & $-1.40$ ($p{=}0.12$) & $-1.58$ ($p{=}0.013$) \\
Ollama-cloud trio (3 repeats) & 300 & $+1.01$ ($p{<}10^{-3}$) & $-0.63$ ($p{=}0.011$) \\
Ollama-cloud regeneration & 300 & $-0.26$ ($p{=}0.63$) & $-1.29$ ($p{=}0.031$) \\
Gemma-4 + GPT-OSS (3 repeats) & 300 & $+0.28$ ($p{=}0.25$) & $-0.68$ ($p{<}10^{-5}$) \\
\bottomrule
\end{tabular}
\end{table}

\paragraph{Estimating scorer noise without ground truth.} Scoring the same outputs with two judges from different vendors (GLM-5.1 and Gemma-3 12B; $1{,}980$ paired outputs: the 90 rescored tasks for every technique except hand-coded ACM) gives a standard deviation of the score difference of $0.192$, so per-judge noise is at most $0.136$, about $1.7$ times $\sigma_s$. This is an upper bound, since disagreement between judges also contains systematic rubric differences. At this level the discrete combiner is far from its crossing, which is why the rule of \Cref{sec:analytical-predictions} favors discrete combining under a noisy scorer.

\section{Router: extended results}
\label{app:router-extended}

\paragraph{Stronger baselines.} \Cref{tab:router-baselines} compares the router with four learned routing policies swept over the same $\lambda$ grid, under the protocol of \Cref{sec:exp-acm} (judge scores, 5-fold cross-validation, paired per-task tests, 20-candidate pool). The cost-sensitive ridge and gradient-boosted regressors predict each technique's quality from the extended task features and select $\argmax(\hat q - \lambda \hat c)$; the RouteLLM-style classifier is a multinomial logit retrained for each $\lambda$ on the label $\argmax(q - \lambda c)$; the bandit is a disjoint LinUCB that observes only the outcome of the arm it pulls while replaying the training fold, then acts greedily on the test fold.

\begin{table}[h]
\centering
\scriptsize
\setlength{\tabcolsep}{4pt}
\caption{Out-of-fold mean judge quality on the 300-task split (three repeats). $^{\dagger}$: the router is significantly ahead of this policy ($p{<}0.05$, paired Wilcoxon over tasks). Fixed-best (CV) is cost-penalized at each $\lambda$ (\Cref{sec:exp-acm}). This analysis uses its own fold seed, so its router and fixed-best values differ from \Cref{tab:semknn-pareto} in the third decimal.}
\label{tab:router-baselines}
\begin{tabular}{@{}lcccccc@{}}
\toprule
\textbf{Policy} & $\lambda{=}0$ & $\lambda{=}1$ & $\lambda{=}3$ & $\lambda{=}10$ & $\lambda{=}20$ & $\lambda{=}38$ \\
\midrule
Semantic-KNN router & \textbf{0.873} & \textbf{0.872} & \textbf{0.873} & \textbf{0.867} & \textbf{0.853} & \textbf{0.821} \\
Ridge regressor & $0.850^{\dagger}$ & $0.851^{\dagger}$ & $0.853^{\dagger}$ & $0.852^{\dagger}$ & $0.838^{\dagger}$ & $0.777^{\dagger}$ \\
Gradient-boosted regressor & $0.861^{\dagger}$ & $0.862$ & $0.860^{\dagger}$ & $0.853^{\dagger}$ & $0.843^{\dagger}$ & $0.806^{\dagger}$ \\
RouteLLM-style classifier & $0.845^{\dagger}$ & $0.845^{\dagger}$ & $0.845^{\dagger}$ & $0.844^{\dagger}$ & $0.830^{\dagger}$ & $0.803^{\dagger}$ \\
LinUCB bandit & $0.834^{\dagger}$ & $0.826^{\dagger}$ & $0.831^{\dagger}$ & $0.816^{\dagger}$ & $0.810^{\dagger}$ & $0.786^{\dagger}$ \\
Fixed-best (CV) & $0.816^{\dagger}$ & $0.822^{\dagger}$ & $0.822^{\dagger}$ & $0.793^{\dagger}$ & $0.793^{\dagger}$ & $0.793^{\dagger}$ \\
\bottomrule
\end{tabular}
\end{table}

\paragraph{Two further configurations.} Under the same protocol, on Gemma-4 + GPT-OSS (three repeats) and on Nemotron + Gemma-4 (one generation), the router beats the cross-validated fixed-best policy at every $\lambda$ ($p \le 4{\times}10^{-4}$). Against the learned policies it is tied or ahead at every setting except one per configuration, where the ridge regressor is significantly ahead ($\lambda{=}3$ on Gemma-4 + GPT-OSS, $0.896$ versus $0.891$; $\lambda{=}10$ on Nemotron + Gemma-4, $0.853$ versus $0.841$).

\paragraph{Calibration size and cost.} \Cref{tab:calibration} subsamples each training fold to $N$ calibration tasks (ten subsamples per fold), fits the router on judge scores, and grades its picks on the held-out fold without the judge. Accuracy is within $1$ pp of the full 240-task fold from $N{=}100$ at every cost setting tested. Building the cache costs one evaluation of the whole candidate pool per calibration task, $0.110$ USD per task for our 20 candidates, which is amortized below the cost of the cheapest single call ($0.0012$ USD) after about $1{,}800$, $5{,}500$, and $9{,}200$ routed tasks for caches of 20, 60, and 100 tasks.

\begin{table}[h]
\centering
\small
\caption{Judge-free accuracy of the router's picks versus calibration-cache size (Ollama-cloud regeneration, 20 candidates; mean over 5 folds $\times$ 10 subsamples).}
\label{tab:calibration}
\begin{tabular}{@{}lcccccc@{}}
\toprule
Calibration tasks $N$ & 20 & 40 & 60 & 100 & 160 & 240 \\
\midrule
$\lambda{=}1$ & 0.840 & 0.840 & 0.865 & 0.874 & 0.873 & 0.877 \\
$\lambda{=}10$ & 0.860 & 0.861 & 0.871 & 0.876 & 0.870 & 0.873 \\
$\lambda{=}20$ & 0.842 & 0.839 & 0.850 & 0.858 & 0.857 & 0.867 \\
\bottomrule
\end{tabular}
\end{table}

\paragraph{Transfer to an independent regeneration.} Fit on the regeneration and applied to the original single-run cache of the headline configuration (judge scores, same 300 tasks), the router has the highest mean quality of all policies at every $\lambda$. Its lead over fixed-best is significant at four of six settings ($p \le 10^{-4}$; not at $\lambda{=}10$ or $38$), and over the gradient-boosted regressor and the RouteLLM-style classifier at most settings; it is statistically tied with the ridge regressor throughout. This test stays within one channel pair, so transfer across model configurations, model updates, and distribution shift remains untested.

\section{Reasoning-enabled ablation}
\label{app:reasoning}
The main trio cannot be used for this ablation: its second channel, Devstral-Small-2, has no reasoning mode, and its Ollama-cloud endpoint rejects requests to enable one. Both channels of the Nemotron + Gemma-4 configuration (\Cref{app:extra-configs}), each of which supports reasoning, were run with internal reasoning enabled for every technique, the judge held fixed with reasoning off, and compared with the reasoning-disabled run of the same pair on the same 300 tasks under judge-free grading (\Cref{tab:reasoning}). Reasoning took effect: channel tokens per task (judge excluded) rose by factors of $2.5$ and $3.3$ for the Nemotron and Gemma-4 single calls, and by $1.03$ for a control that reasons in both runs. With reasoning on, the judge ranks Self-Refine first of the four techniques in the table ($0.880$, against $0.748$ for the single call) while it is objectively the weakest, another instance of the over-ranking of \Cref{app:judge-validation}.

\begin{table}[h]
\centering
\small
\caption{Judge-free accuracy with internal reasoning off and on (paired over 300 tasks). Left: effect of reasoning on each technique. Right: each operator against the single call within the same mode.}
\label{tab:reasoning}
\begin{tabular}{@{}lcccc@{\hspace{12pt}}cc@{}}
\toprule
& \multicolumn{4}{c}{\textbf{Reasoning off $\to$ on}} & \multicolumn{2}{c}{\textbf{vs.\ single call}} \\
\cmidrule(lr){2-5}\cmidrule(l){6-7}
\textbf{Technique} & \textbf{Off} & \textbf{On} & \textbf{Gain (pp)} & $p$ & \textbf{Off (pp)} & \textbf{On (pp)} \\
\midrule
Single call & 0.697 & 0.893 & $+19.7$ & $4{\times}10^{-10}$ & n/a & n/a \\
HARQ-IR & 0.877 & 0.933 & $+5.7$ & $0.008$ & $+18.0^{***}$ & $+4.0^{*}$ \\
Turbo & 0.870 & 0.920 & $+5.0$ & $0.032$ & $+17.3^{***}$ & $+2.7$ \\
Self-Refine & 0.773 & 0.807 & $+3.3$ & $0.20$ & $+7.7^{*}$ & $-8.7^{***}$ \\
\bottomrule
\end{tabular}
\end{table}

\section{Operator-view derivations}
\label{app:operator-view-derivations}

\paragraph{Direct empirical comparison.}
We implement Self-Consistency, Self-Refine, and Chain-of-Verification as standalone baselines (not \agentcodec{} techniques) following the canonical recipes in the cited papers, using the same generator, judge, and task set as the \agentcodec{} runs. Budget matching: Self-Consistency and Best-of-$N$ use $N\!=\!5$ samples, Self-Refine uses $3$ refinement rounds, and Chain-of-Verification uses the $4$-step canonical pipeline. The \agentcodec{} counterparts are configured at the same call budget ($N\!=\!5$ for Diversity~MRC, $K\!=\!5$ for HARQ-IR/turbo, rate~$r\!=\!0.5$ for FEC). \Cref{tab:baseline-direct} reports the paired-difference results; statistical tests use paired Wilcoxon signed-rank at $\alpha\!=\!0.05$. The direction favors the \agentcodec{} variant in 6 of 7 pairs, but only four differences are significant, and the cost advantages are more robust across configurations than the quality margins (\Cref{app:judge-validation}).

\paragraph{Diversity-SC-$N$ is Selection Combining at wider pool size.}
Canonical Best-of-$N$~\citep{cobbe2021verifiers,lightman2024verify} and our Diversity~SC operator differ only in pool size and sample source: BoN draws $N$ independent samples from one policy model via temperature sampling; Diversity~SC ($N\!=\!2$) draws one sample from each of two diverse generators. Both apply the same selection operator $\argmax_i q_i$. The natural generalization is \emph{Diversity-SC-$N$}: cycle $N$ samples through the full set of configured channels (a multi-model diversity pool sampled to depth $N$) and select the argmax. Diversity-SC-$N$ collapses to canonical BoN when only one channel is configured, and to ordinary Diversity~SC when $N\!=\!|\text{channels}|$. At 8B with $N\!=\!5$ we observe $q_{\text{Div-SC-}N}\!=\!0.803$ vs.\ $q_{\text{SC},N=2}\!=\!0.746$, a $+0.057$ gain attributable to two compounding effects: the $N$-branch SC gain $G_{\text{SC}}(N)\!\propto\!\log N$ for Rayleigh-faded branches with perfect CSI, plus cross-model diversity that a single-policy BoN cannot exploit. Canonical BoN with a trained PRM~\citep{lightman2024verify} is the single-model, learned-CSI instance of the same selection operator; Diversity-SC-$N$ is the multi-model, judge-CSI instance. The operator is identical in both cases, and the free parameters in the SC-family reduce to the pool size $N$, the selector $s$ (LLM judge, PRM, or logprob), and the sample source (single policy vs.\ multi-model pool).

\paragraph{Diversity-MRC-Discrete-$N$ is MRC on the discrete answer space.}
Weighted Best-of-$N$~\citep{li2023diverse} extends BoN by partitioning the $N$ samples into semantic equivalence classes and weighting each class by the sum of its members' scores before selecting; confidence-informed self-consistency (CISC)~\citep{taubenfeld2025confidence} does the same with the model's own confidence as the weight. This is exactly Diversity~MRC transposed to the discrete answer space: \emph{clusters replace branches, cluster-summed judge scores replace per-branch SNR, and the cluster with the highest weighted sum is selected with its top-$q$ member returned as the combined output}. Two degenerate cases recover familiar prior methods: if every sample forms its own cluster, the operator reduces to BoN (no clustering); if all per-sample scores are set to one, cluster weights collapse to cluster sizes and the operator reduces to Self-Consistency~\citep{wang2023selfconsistency} (majority vote). Self-Consistency and Weighted Best-of-$N$ are therefore the same discrete-MRC operator at $q_j\!\equiv\!1$ vs.\ $q_j\!=\!$judge-score. Extending the same generalization as above yields \emph{Diversity-MRC-Discrete-$N$}: cycle $N$ samples through the multi-model channel pool, cluster by semantic equivalence, sum judge scores within clusters, return the top-$q$ member of the winning cluster. On the matched 8B subset (n=69) Diversity-MRC-Discrete-$N$ achieves $q\!=\!0.793$ vs.\ continuous-MRC $0.766$ ($\Delta\!=\!+0.027$, 26W/29T/14L), with wins in all four categories; canonical single-model Weighted-BoN / CISC reduces to the same operator with a single policy and $N$ temperature draws.

\paragraph{Continuous vs.\ discrete answer spaces.}
The split between continuous and discrete answer spaces drives the framework's two MRC variants: clustering (the discrete path) requires well-defined semantic equivalence classes, clean on numeric/multiple-choice answers (math, QA, code) but fuzzy on free-form text; synthesis (the continuous path) trades the clustering LLM call for a combining LLM call and avoids the equivalence-class problem at the cost of one extra generation. Empirically, on our matched 8B split (n=69, hard subset), discrete-MRC dominates continuous-MRC at the aggregate ($q_{\text{MRC-disc}}\!=\!0.793$ vs.\ $q_{\text{dMRC}}\!=\!0.766$, $\Delta\!=\!+0.027$, 26W/29T/14L paired), and wins in every category: $+0.043$ on QA, $+0.029$ on creative, $+0.022$ on reasoning, $+0.015$ on code. The ordering flips at 3B on QA (continuous-MRC $+0.032$) where the voter LLM is too weak to cluster consistently, confirming that the gap is CSI-limited rather than operator-intrinsic: discrete-MRC requires a reliable voter to define the equivalence classes, and falls back to the synthesis path when the voter is noisy.

\paragraph{Why we expected domination at matched budget, and what we measured.}
The reductions in \Cref{tab:baseline-reduction-app} are not loose analogies; they are hyperparameter specializations. In Diversity~MRC the combining weights reduce to majority vote if we threshold $q_i$ at $0.5$; in HARQ-IR the critic can be configured to issue unstructured free-form rewrites (recovering Self-Refine) by removing the JSON critique template and the $\alpha$-scaled correction cap. The added mechanism is always a \emph{restrictor}: it removes degrees of freedom that were empirically harmful (e.g.\ over-correction at $\alpha\!=\!1$, common-mode judge noise without interleaving, parity calls that duplicate rather than complement the main reasoning). Because each restrictor is an operational change rather than an architectural one, we expected each variant to match or exceed its prior method at matched budget. The measured outcome is mixed: under the judge the variant scores higher in 6 of 7 pairs, significantly in four (\Cref{tab:baseline-direct}), and under judge-free grading the operator-level advantage is significant only for HARQ-IR over Self-Refine (\Cref{app:judge-validation}).

\section{Standard Benchmark Validation: Hard-Split Construction and Results}
\label{app:standard-benchmarks}

\begin{figure}[h]
  \centering
  \includegraphics[width=0.8\linewidth]{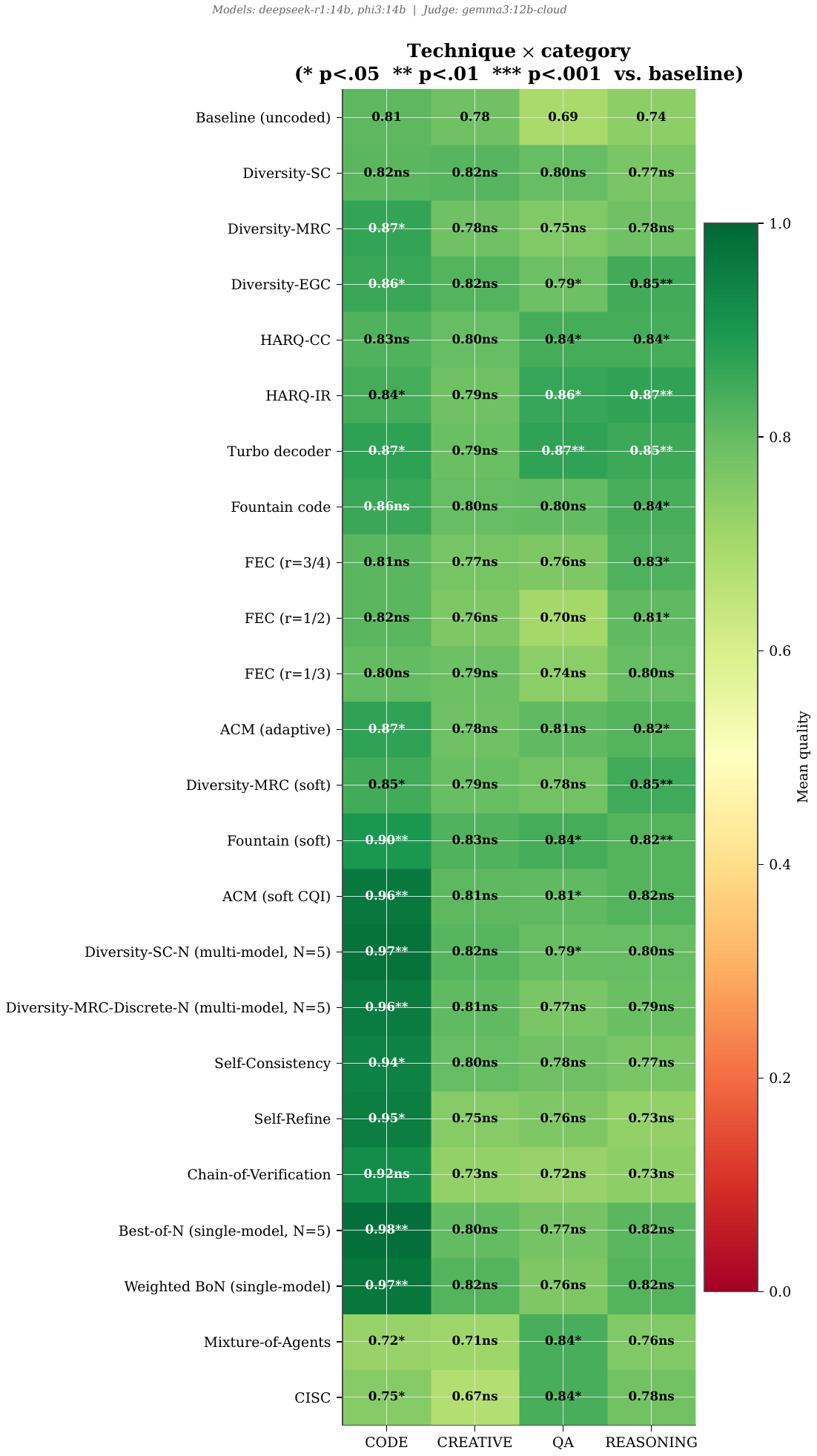}
  \caption{Technique $\times$ task-category heatmap on the 14B local configuration (DeepSeek-R1 14B + Phi-3 14B; judge Gemma-3 12B; $n{=}69$ curated tasks). Each cell reports the per-(technique, category) mean quality with paired-Wilcoxon-vs.-baseline significance stars. Different techniques lead on different task types: the iterative decoders (HARQ-IR, turbo) on question answering and reasoning, selection and weighted voting over five samples (Best-of-$N$, Diversity-SC-$N$, Weighted Best-of-$N$) on code; on creative, the top five lie within one percentage point of each other.}
  \label{fig:heatmap}
\end{figure}
% plots/goods/8B_llama_qwen_gemma3/technique_heatmap.pdf

\begin{figure}[h]
  \centering
  \includegraphics[width=0.8\linewidth]{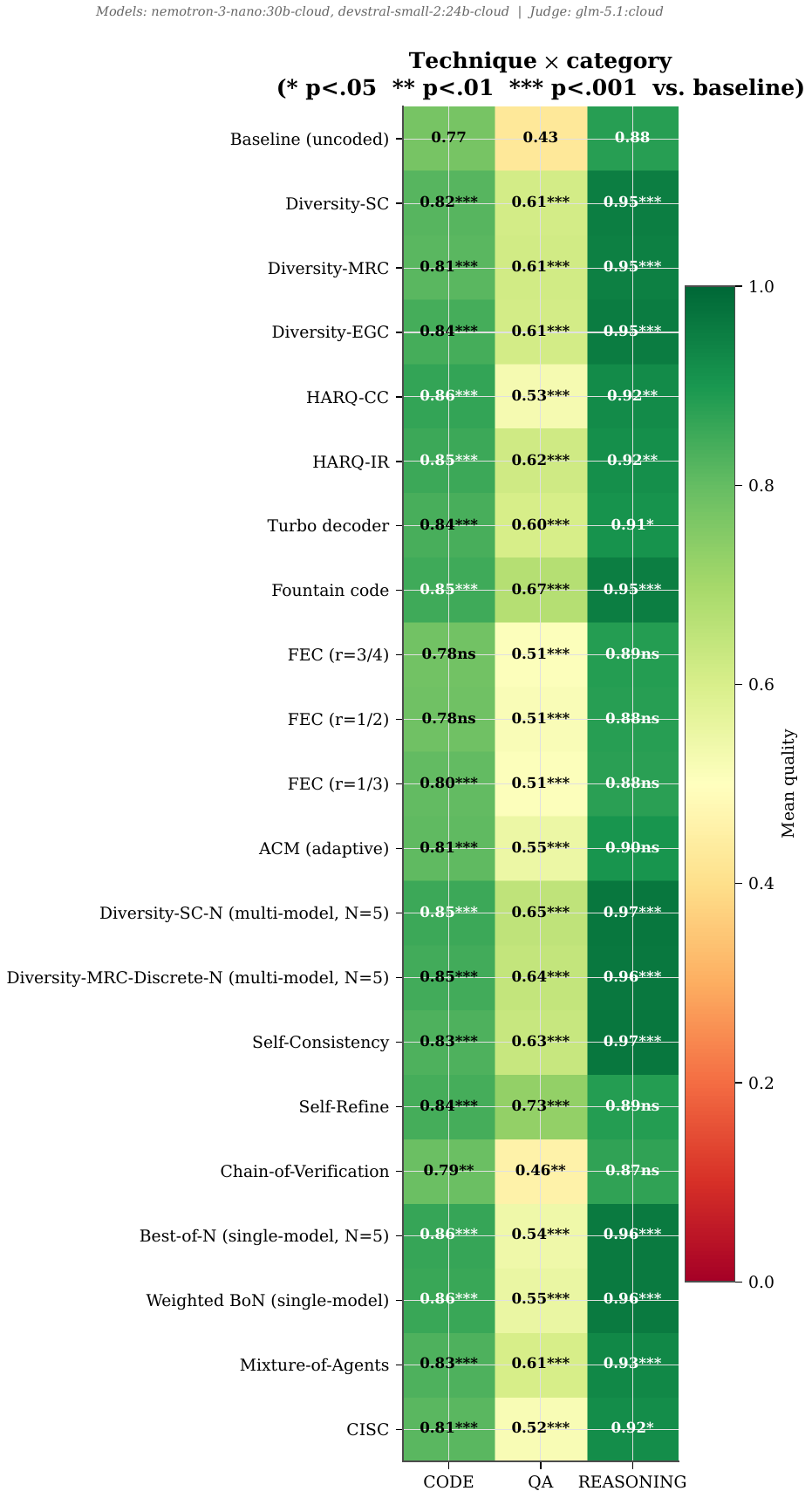}
  \caption{Technique $\times$ task-category heatmap on the Ollama-cloud trio. Each cell reports the per-(technique, category) mean quality with paired-Wilcoxon-vs.-baseline significance stars. Different techniques lead on different task types: Self-Refine (which the judge over-ranks; \Cref{sec:exp-judge}) and fountain on question answering (MMLU-Hard), multi-sample voting and selection (Diversity-SC-$N$, Self-Consistency, Diversity-MRC-Discrete-$N$) on reasoning (GSM8K-Hard), and on code (HumanEval-Hard) HARQ-CC, Best-of-$N$, Weighted Best-of-$N$, and HARQ-IR within $1.1$ pp of each other.}
  \label{fig:heatmap_datasets}
\end{figure}

To check that the operating-point dependence seen on the curated tasks is not a curation artifact, we evaluate on a hard split of 100 items from each of three standard benchmarks. Channel models saturate on easy items, so averaging over the full sets compresses technique differences; we therefore select by published structural proxies of difficulty (calculator-annotation steps for GSM8K, the \texttt{abstract\_algebra} subject for MMLU, docstring examples for HumanEval) rather than by any model's behavior, which keeps the split model-agnostic. \Cref{app:data-hard} gives the exact construction.

\begin{table}[h]
\centering
\caption{Standard-benchmark validation on the Ollama-cloud trio (Nemotron-Nano-3 + Devstral-Small-2 channels; GLM-5.1 judge; $n{=}100$ tasks per split; three generation repeats averaged per task). Baseline is the uncoded single call; ``Best fixed technique'' is the fixed \agentcodec{} operator with the highest mean quality on that split; $\Delta$ is the absolute mean-quality improvement over the per-split baseline. Among the re-implemented prior methods, only Self-Refine scores higher on any split (MMLU-Hard, $q{=}0.729$), a technique the judge over-ranks (\Cref{sec:exp-judge}).}
\label{tab:standard-benchmarks}
\small
\begin{tabular}{@{}llccc@{}}
\toprule
\textbf{Benchmark split} & \textbf{Best fixed technique} & \textbf{Baseline $q$} & \textbf{Technique $q$} & $\bm{\Delta}$ \\
\midrule
GSM8K-Hard    & Diversity-SC-$N$  & $0.882$ & $0.967$ & $+0.085$ \\
MMLU-Hard     & Fountain          & $0.429$ & $0.669$ & $+0.240$ \\
HumanEval-Hard & HARQ-CC          & $0.770$ & $0.865$ & $+0.095$ \\
\bottomrule
\end{tabular}
\end{table}

The qualitative direction of the technique ranking is preserved across all three hard-benchmark splits (\Cref{fig:hard-benchmark-bars,fig:gallery-ollama300-bars}, \Cref{tab:standard-benchmarks}): diversity-style techniques and the HARQ variants remain among the top performers, the uncoded baseline remains the cheapest but lowest-quality choice, and on each dataset the semantic-nearest-neighbor router at $\lambda{=}3$ is within $1$ pp of its in-sample ceiling (\Cref{sec:exp-acm-decomp}). The per-split best fixed technique varies (Diversity-SC-$N$ on GSM8K-Hard, fountain on MMLU-Hard, HARQ-CC on HumanEval-Hard; \Cref{tab:standard-benchmarks}), tracking the same operating-point dependence observed on the curated tasks. Per-dataset quality-versus-cost Pareto scatters are in \Cref{fig:hard-benchmark-pareto,fig:gallery-ollama300-pareto}.

\begin{figure}[h]
  \centering
  \includegraphics[width=\linewidth]{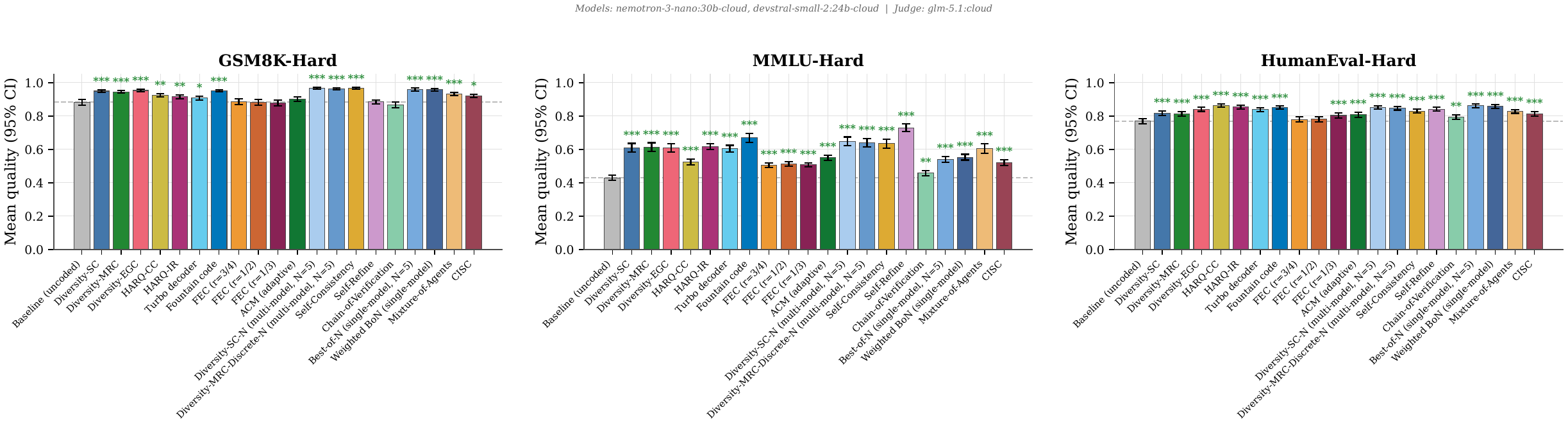}
  \caption{Mean quality by technique on the three hard-benchmark splits (Ollama-cloud trio; one panel per dataset, 100 tasks each). Bars show per-technique mean quality with 95\% bootstrap confidence intervals; significance stars denote paired-Wilcoxon-vs.-baseline tests; the dashed line is the per-panel baseline mean.}
  \label{fig:hard-benchmark-bars}
\end{figure}

\section{Validation experiments}
\label{app:validation-experiments}

\paragraph{Judge-in-loop confound.}
Because the judge model serves as final scorer and also scores candidates inside several techniques (selection, MRC weighting, stopping rules), one might worry that judge-in-loop techniques win by borrowing the 12B model's capability rather than from the decoding scheme itself. The oracle-winner distribution on the 8B curated split rebuts the strong version of this concern: among per-task oracle winners, \emph{judge-free-in-decode} techniques (baseline, Diversity-EGC, fountain with agreement-based stopping) account for $39.1\%$ of wins ($13.0\%$, $14.5\%$, $11.6\%$ respectively), while HARQ-IR and HARQ-CC together account for only $15.9\%$. Selection combining (Diversity-SC) is the most frequent winner ($18.8\%$), consistent with a passive-diversity interpretation rather than a judge-borrowing effect. If the 12B critic were the dominant driver, we would expect HARQ/turbo/BoN to sweep; they do not. The judge-strength ablations of \Cref{app:judge-validation} test this directly.

\paragraph{Synthesis integrity.} Five controlled tests found no synthesizer contamination; the one configuration where a much larger synthesizer does act as a solver is analyzed in \Cref{app:synthesis-integrity}.

\paragraph{Prediction validation.} Four communication-theoretic predictions tested at the cloud-Haiku operating point (Claude Haiku 4.5 + GPT-5-mini, judge Haiku 4.5; baseline $q_{0}=0.818$): (P1)~MRC~$\geq$~EGC~$\geq$~SC: \textbf{direction consistent, not resolved} (observed $0.884 \geq 0.870 \geq 0.856$, with per-task MRC$\succ$EGC at $19$/$17$ and EGC$\succ$SC at $28$/$16$; both gaps are below the resolution of a single 69-task run); (P2)~HARQ-IR converges faster than HARQ-CC: \textbf{PASS} (IR running-max jumps from $0.807$ to $0.891$ between rounds~1 and 2 and reaches $99\%$ of its peak by round~3, while CC plateaus more gradually and only reaches $99\%$ of its lower peak $0.878$ at round~4); (P3)~diminishing diversity returns: \textbf{PASS} (marginal quality gain per extra channel falls from $+3.74$pp ($d{=}1{\to}2$, baseline$\to$SC) to $+1.38$pp per channel ($d{=}2{\to}5$, SC$\to$SC-$N$), matching the classical sublinear scaling); (P4)~turbo iterations improve quality monotonically: \textbf{PASS in direction, FAIL in shape} (all-tasks running-max mean climbs $0.808 \to 0.869 \to 0.882 \to 0.898 \to 0.902$ with the steepest step at iteration~1 and diminishing increments thereafter, but we do not observe a sharp ``waterfall'' threshold; per-task gains are diffuse rather than concentrated at a critical iteration). All four predictions hold in direction, P1 without statistical resolution and P4 without the sharp threshold shape; these single-run checks are consistent with the framework as a design aid, not a quantitative physical theory of LLM channels.

\paragraph{Judge ablation.}
The judge-strength and judge-swap ablations, and the comparison of the 14B channels under two judges, are reported in \Cref{app:judge-validation}.

\section{Refinement-threshold evidence and limitations}
\label{app:threshold-details}
\label{app:limitations}

\paragraph{Why the statistic changed.} An earlier version of this analysis plotted the mean score of the tasks still being refined at each iteration (the ``survivor'' mean). That statistic depends on which tasks survive: tasks leave the pool when the critic stops finding corrections, and the leavers score higher (at 8B, early-exiting tasks finish at $0.591$ against $0.465$ for those that run the full length), so the remaining pool worsens in composition as iterations proceed, and the sign of its net change at 3B flips with the inclusion rule. We therefore report composition-free statistics computed per trajectory: the change from the initial answer $q_0$ to the final \emph{unguarded} iterate, its sign, and its mean with a bootstrap interval (\Cref{fig:threshold-3panel}).

\paragraph{Per-scale evidence.} Turbo on the three local configurations, five iterations at threshold $\tau{=}0.85$, judge Gemma-3 12B, trajectories with at least one refinement:
\begin{itemize}[nosep,leftmargin=*]
  \item \emph{3B}, Qwen-2.5 3B critiquing itself, three generation runs ($n{=}145$ trajectories on 57 tasks): $26\%$ end below their start; $+6.2$ pp [$+2.8$, $+9.8$], $p{=}0.007$ (interval and test at the task level, since a task's runs are not independent).
  \item \emph{8B}, Llama-3.1 8B with a Qwen-2.5 7B critic ($n{=}58$): $26\%$ end below; $+5.1$ pp [$+0.6$, $+9.7$], $p{=}0.039$.
  \item \emph{14B}, DeepSeek-R1 14B critiquing itself ($n{=}41$): $5\%$ end below; $+19.7$ pp [$+12.4$, $+27.3$], $p{=}2{\times}10^{-5}$.
\end{itemize}
Refinement helps significantly at every scale; its failure rate is about a quarter at 3B and 8B and falls to $5\%$ at 14B, where its mean benefit is three to four times larger. This is the graded form of \Cref{prop:refinement-threshold}. No scale is net-harmful on average, so a pair operating strictly below threshold has not been observed.

\paragraph{Dependence on initial quality.} Pooling turbo and HARQ-IR over ten configurations (20 configuration-technique cells), the mean unguarded change is $+13.2$ pp for $q_0 < 0.4$ ($n{=}640$), $+4.8$ pp for $0.4 \le q_0 < 0.6$ ($n{=}581$), $+3.0$ pp for $0.6 \le q_0 < 0.75$ ($n{=}397$), and $+1.1$ pp for $0.75 \le q_0 < 0.9$ ($n{=}683$); trajectories that start at $q_0 \ge 0.9$ are not refined, because the operators stop once $q \ge \tau$. Part of this decline is regression to the mean: in the four three-repeat caches, the uncoded answer's score changes between two independent generations of the same task by $+8.7$, $+1.2$, $+1.7$, and $-1.7$ pp in the same four bins. Refinement exceeds this placebo in every bin, by $1.3$ to $4.5$ pp, so the decline is real in direction but we do not claim a precise decay rate. At the top of the scale the gate matters: Self-Refine, which refines regardless of quality, delivers $2.8$ pp less than its own initial answer on trajectories that start at $q_0 \ge 0.9$ (seven configurations, $n{=}355$, $p{<}10^{-19}$), whereas every turbo and HARQ-IR trajectory that starts there ($n{=}377$ and $355$) stops after the initial answer.

\paragraph{Limitations (full discussion).}
\emph{Judge dependence.} Outside the judge-validation study, every reported quality is a judge score. Judge-free grading shows the judge is informative at the item level (AUC $0.77$) but ranks near-tied techniques weakly (Spearman $0.37$ with objective accuracy), and it over-ranks Self-Refine (\Cref{app:judge-validation}). No human evaluation was performed.
\emph{Residual CSI noise.} The binary checklist resolves score quantization (up to $2^{15}$ levels), but scorer noise remains: two judges disagree on the same output with standard deviation $0.19$. Differential scoring for refinement (candidate and incumbent scored against each other) cancels part of this common-mode noise.
\emph{Repeats.} The 300-task split and both 3B configurations have three generation runs; the other four columns of \Cref{tab:results-summary} are single runs, for which we report gaps below about $4$ pp as ties (\Cref{app:noise}). The 8B and 14B local configurations were not repeated because they run on one RTX~3090.
\emph{Router scope.} The router's margin depends on cache coverage, task distribution, judge, and candidate pool (\Cref{sec:exp-judge}); its transfer across model configurations, model updates, and distribution shift is untested, and it is close to non-parametric algorithm selection~\citep{rice1976algorithm}.
\emph{Analytical scope.} \Cref{prop:csi-crossover} is a first-order statement for a linear combiner, and the location of the synthesis crossover was found by comparing configurations; \Cref{prop:refinement-threshold} is supported in graded form only.
\emph{Cost.} Scoring adds cost even on easy tasks; at flat per-token rates the judge is about $83\%$ of per-task cost (\Cref{app:judge-validation}).
\emph{Non-stationarity.} Rankings from generation campaigns weeks apart on cloud models agree at Spearman $0.77$, against $0.91$ to $0.94$ within one campaign, consistent with model updates changing the channel: on unchanged local weights, 3B runs six months apart agree at $0.89$ to $0.91$ (\Cref{app:noise}).
\emph{Curated tasks.} The 69 curated tasks were written to expose technique differences, so gains on a random-difficulty set would be smaller; they carry only the qualitative claim that no fixed choice dominates, which the 300-task split reproduces.

\section{Broader impacts}
\label{app:broader-impacts}
The techniques studied here make LLM agents more reliable per unit cost and let a deployment trade quality for cost explicitly: one knob sets how much computation each task receives, so easy tasks need not consume the budget of hard ones. This can reduce both errors and energy use in agent pipelines, and the communication-theoretic framing gives practitioners a shared vocabulary for choosing among reliability techniques. Two considerations accompany these benefits. First, reliability measured by an LLM judge is only as trustworthy as the judge, a property of LLM-as-judge evaluation in general and not of the operators. Our validation shows that judge scores track correctness at the item level but can misorder techniques whose quality is nearly tied (\Cref{app:judge-validation}), so a deployment should check its scorer against verifiable outcomes on a sample of its own tasks. \agentcodec{} supports this directly: besides the judge, its scorer offers an exact mode, which replaces the judge by a deterministic check (answer matching, or execution against a task's tests), and a blended mode that combines the two (\Cref{app:exp-methodology}). Both can audit the judge and calibrate the router on historical tasks with verified outcomes, which keeps ground truth out of the decoding loop (\Cref{sec:system}). Second, as with any gain in agent reliability, more dependable agents also serve harmful uses. The operators do not change what the underlying models can do, only how often they succeed, so the safety measures of the underlying models remain the primary safeguard.

\section{Soft-Output Extensions via Token-Level Log-Probabilities}
\label{sec:soft}

The techniques in the main text treat each agent response as a \emph{hard} symbol: the combiner sees only the emitted text and a scalar judge score. Classical receivers, however, benefit enormously from \emph{soft} decisions: log-likelihood ratios (LLRs) carry not just the decoded bit but the receiver's confidence in it, and every competitive decoder (MAP, BCJR, belief propagation, turbo) is soft-in soft-out. Autoregressive LLMs expose an analogous quantity for free: the per-token log-probability $\log p(y_t \mid y_{<t}, x)$. We define a response-level confidence
\begin{equation}
  c(y) \;=\; \exp\!\Big(\tfrac{1}{N}\textstyle\sum_{t=1}^{N} \log p(y_t \mid y_{<t}, x)\Big),
  \label{eq:soft-conf}
\end{equation}
the geometric mean per-token probability, which plays the role of a channel LLR at the response level. This section introduces three \emph{optional} soft variants of our core techniques (\textsc{Soft-MRC}, \textsc{Soft-Fountain}, and \textsc{Soft-ACM}), added as additional modes without modifying the hard-decision pipeline used in~\Cref{sec:exp-judge}.

\paragraph{Soft MRC.} In classical MRC, branch $k$ is weighted by its instantaneous SNR $\gamma_k$; our hard variant uses the judge score as a proxy. Soft MRC replaces this proxy with the model's own confidence, weighting branch $k$ by $w_k = c(y_k)/\sum_j c(y_j)$ and passing weighted candidates to the synthesizer. Because synthesis is lossy (\Cref{app:synthesis-integrity}), we enforce the MRC guarantee $\mathrm{SNR}_{\text{MRC}}\!\geq\!\mathrm{SNR}_{\text{SC}}$ by scoring the combined output and falling back to selection combining (the single highest-confidence branch) whenever synthesis scores below the best individual, the direct analog of a CRC check triggering a retransmission.

\paragraph{Soft Fountain.} A rateless decoder stops as soon as it has collected enough reliable symbols. Our hard fountain variant uses an agreement heuristic; Soft Fountain uses a genuine confidence-based stopping criterion
\begin{equation}
  S \;=\; 0.6 \cdot \overline{c} \;+\; 0.4 \cdot \mathrm{agreement},
\end{equation}
and marks any sample with $c(y) < 0.5\,c_{\max}$ as an \emph{erasure}, mirroring erasure decoding in real fountain codes rather than quality gating. The ML-style synthesizer only decodes over non-erased samples.

\paragraph{Soft ACM.} Hard ACM estimates channel quality via a meta-judgment prompt (``rate this task's difficulty''), which we found to be bimodal and to leak uncoded outputs on catastrophic tasks. Soft ACM instead transmits a short \emph{pilot probe} (a 2--3 sentence brief answer) and uses $c(y_{\text{probe}})$ directly as the CQI, with zero extra judge calls, and grounded in the generator's own uncertainty rather than a separate meta-query. The MCS table maps confidence to coding mode with \emph{no} uncoded floor: $[0.70,1.0]\!\to\!$FEC $r{=}0.75$; $[0.50,0.70]\!\to\!$FEC $r{=}0.5$; $[0.30,0.50]\!\to\!$HARQ-IR; $[0,0.30]\!\to\!$MRC diversity.

\paragraph{Where the analogy tightens.} Soft extensions sharpen three gaps we observed in the hard pipeline. (i)~\emph{Combining weight fidelity}: judge scores are discretized and noisy; $c(y)$ is continuous and generated by the same process that produced $y$, matching the SISO recipe more faithfully. (ii)~\emph{Stopping criteria}: rateless codes are defined by a confidence threshold, not a quality poll. (iii)~\emph{CQI cost}: a pilot probe costs one short generation, whereas meta-judgment costs a full extra judge call and is an out-of-band oracle.

\paragraph{Limitations.} Soft outputs require backend logprob support: OpenAI, Ollama, and vLLM expose them, but the Anthropic API does not, so soft techniques raise an explicit error on unsupported backends rather than silently falling back. The geometric mean in~\eqref{eq:soft-conf} weights all tokens equally and is blind to \emph{which} tokens are uncertain (a single wrong fact buried in a confident paragraph can still dominate correctness); token-level weighting by salience is left to future work. Finally, $c(y)$ measures the model's internal confidence, not correctness; it is a calibrated LLR only to the extent that the generator itself is calibrated.

\section{A Larger Design Space}
\label{app:design-space}

The communication-theoretic view admits a substantially larger design space than the six techniques instantiated in this paper. Concrete mappings we have not built include: (i)~\emph{interleaving} for decorrelating burst errors in long reasoning chains, treating consecutive reasoning steps as a coded sequence and permuting the order in which a critic inspects them; (ii)~\emph{MIMO-style joint precoding} across multiple generator channels, where prompt design plays the role of a precoder steering each channel into a complementary error mode rather than running them independently; (iii)~\emph{LDPC- and polar-style structured FEC} over reasoning chains, replacing the uniform repetition that our FEC implementation currently uses with sparse parity structures that admit message-passing decoders; (iv)~\emph{coded computation for verification}, exchanging some of the redundant generator calls in HARQ/turbo for cheaper checksum-style verifier calls in the spirit of coded distributed computing; (v)~\emph{OFDM-style subcarrier allocation} as principled task decomposition, where independent reasoning subtasks are assigned to channels under a power-budget constraint analogous to bit-loading; and (vi)~\emph{joint source-channel coding} for compressing the prompt and its protection scheme together, rather than the prompt-then-redundancy decomposition implicit in our current pipelines. Each of these is a well-developed line in classical communications with a natural agentic counterpart that has not, to our knowledge, been instantiated in either form. The contribution of this paper is to commit to the channel abstraction concretely enough that those mappings become specifiable rather than evocative; the limitation is that we have implemented only a small subset and evaluated them on a constrained compute budget. We hope subsequent work treats this as a research program rather than a closed set of analogies.

\section{Synthesis integrity validation}
\label{app:synthesis-integrity}

For MRC and EGC combining, \emph{synthesis contamination} is a methodological concern: the synthesizer model may inject its own knowledge into the combined output rather than faithfully merging the input responses. Quality gains attributed to combining would then reflect the synthesizer solving the task independently, invalidating the diversity framework.

We validate synthesis integrity through five controlled tests using the 14B local configuration (DeepSeek-R1 14B + Phi-3 14B channels, Gemma-3 12B judge/synthesizer).

\paragraph{Test 1: Information preservation.}
Both channel models answer ``What is the GCD of 252 and 735?'' Both correctly produce 21 via the Euclidean algorithm. The synthesizer preserves the correct answer and complete reasoning chain without adding computation steps or alternative methods not present in either input. \textbf{Result: PASS.}

\paragraph{Test 2: No contamination (factual).}
We construct two deliberately sparse descriptions of Python: Input~A states ``created in the early 1990s, interpreted language''; Input~B states ``high-level, supports multiple paradigms.'' Neither mentions Guido van Rossum, Django, NumPy, indentation-based syntax, duck typing, or list comprehensions, all facts the synthesizer has seen in training. The synthesizer output contains \emph{zero} facts beyond those in the two inputs. \textbf{Result: PASS.}

\paragraph{Test 3: Faithful error preservation.}
We test whether the synthesizer corrects \emph{errors} present in inputs using its own knowledge. Input~A states the speed of light was ``first measured accurately by Michelson in 1887''; Input~B states relativity was ``published in 1915'' (the correct year for \emph{special} relativity is 1905). The synthesizer faithfully reproduces both claims without correction, confirming it does not silently fix factual errors using its own knowledge. It also does not add the exact value 299,792,458~m/s (absent from both inputs). \textbf{Result: PASS.}

\paragraph{Test 4: Conflict resolution without independent computation.}
Input~C states $\int_0^3 x^2\,dx = 9$ (correct); Input~D states $\int_0^3 x^2\,dx = 8$ (incorrect). Input~C is scored higher (0.80 vs.\ 0.60). The synthesizer selects the higher-scored answer (9) without showing its own derivation (e.g., $[x^3/3]_0^3$). \textbf{Result: PASS.}

\paragraph{Test 5: No prompt leakage.}
The synthesis prompt uses internal labels (\texttt{[BEST]}, \texttt{[ALT-1]}) and section headers. None of these labels appear in the synthesized output. \textbf{Result: PASS.}

These tests are consistent with the constrained maximal-ratio synthesis prompt (``Your output must contain ONLY information from the responses above'') and its low temperature ($T\!=\!0.1$) keeping the synthesizer from acting as an independent solver. Five probes cannot exclude contamination on every input, but they make it unlikely that the combining gains come mainly from the synthesizer's own knowledge.

\paragraph{Where the synthesizer does act as a solver.} The probes above used a synthesizer no stronger than its channels. In every configuration the synthesizer is the judge (\Cref{app:impl-diversity}), and the judge writes answer text only in the two synthesis-based combiners, Diversity-MRC and Diversity-EGC; every other technique's output is written by the channel models, with the judge only scoring. A direct test of contamination is whether synthesis turns two failed inputs (both scored below $0.5$) into a good answer (scored $0.8$ or more). Across all reported configurations but one, this happens in $5$ of $274$ such cases for Diversity-EGC ($1.8\%$) and in none of $278$ for Diversity-MRC. The exception is the 3B + cloud-judge configuration, whose synthesizer, the Gemma-4 31B judge, is ten times the size of the 3B channels. There Diversity-EGC, whose prompt asks the synthesizer to resolve disagreements and add missing details, rescues $56\%$ of $71$ such cases and scores above both of its inputs in $75\%$ of runs, so its $+49.2\%$ in \Cref{tab:results-summary} measures the synthesizer rather than equal-gain combining. Diversity-MRC in the same configuration is unaffected ($1\%$ of $83$ cases): its prompt forbids new information, and its final answer is the best input unchanged in $93\%$ of runs. Hand-coded ACM inherits part of the effect through the $12\%$ of runs it routes to Diversity-EGC; limiting those runs to their best input lowers its gain from $+27.2\%$ to $+23.7\%$. The issue is confined to these two entries of this one configuration. No other entry, figure, or claim depends on them; in particular, the router of \Cref{fig:acm-oracle-pareto} selects a synthesis-based combiner on $1\%$ of tasks at $\lambda{=}3$ and on none at $\lambda{=}38$.

\section{Formal Channel-Theoretic Foundations}
\label{app:formal-foundations}

This appendix provides the information-theoretic arguments underlying the agent channel model (\Cref{sec:model}), stating where established theorems apply and where the agent setting introduces deviations.

\subsection{The LLM as a discrete stochastic channel}

\begin{lemma}[Channel identity]
\label{lemma:refinement-threshold}
Let $\calA_\theta$ be an autoregressive language model with parameters $\theta$ and temperature $T > 0$.
For any input $x \in \calT$, the model defines a conditional distribution over output sequences:
\begin{equation}
  p_T(y \mid x) = \prod_{t=1}^{|y|} \frac{p_\theta(y_t \mid y_{<t}, x)^{1/T}}{\sum_{v \in \mathcal{V}} p_\theta(v \mid y_{<t}, x)^{1/T}}
\end{equation}
This is a valid transition probability satisfying $\sum_{y \in \mathcal{V}^*} p_T(y \mid x) = 1$ for all $x$, and thus constitutes a discrete stochastic channel in the sense of Shannon~\citep{shannon1948}.
\end{lemma}

\begin{proof}
The softmax function with temperature scaling produces a valid probability distribution at each token position: each factor is non-negative and sums to one over $\mathcal{V}$.
The product of valid conditional distributions over a finite sequence (terminated by an end-of-sequence token) yields a valid joint distribution over $\mathcal{V}^*$.
The mapping $x \mapsto p_T(\cdot \mid x)$ is therefore a well-defined channel transition probability.
\end{proof}

No approximation or limiting argument is involved.

\subsection{Conditional entropy and channel noise}

For a fixed input $x$ and temperature $T$, the per-token conditional entropy is:
\begin{equation}
  H_T(Y_t \mid y_{<t}, x) = -\sum_{v \in \mathcal{V}} p_T(v \mid y_{<t}, x) \log p_T(v \mid y_{<t}, x)
\end{equation}
The sequence-level conditional entropy, normalized by output length, is:
\begin{equation}
  \mathcal{H}_T(\calA_\theta, x) = \frac{1}{\E[|Y|]} \sum_{t=1}^{\E[|Y|]} H_T(Y_t \mid Y_{<t}, x)
\end{equation}
This quantity is the channel's \emph{intrinsic noise rate} in bits per token.
It decomposes the factors that determine channel quality:

\begin{enumerate}[nosep]
  \item \textbf{Model parameters $\theta$}: larger, better-trained models learn sharper conditional distributions (lower entropy at each position). Neural scaling laws~\citep{kaplan2020scaling, hoffmann2022training} show that cross-entropy loss (= our $\bar{\mathcal{H}}$) decreases as a power law $L \propto N^{-\alpha}$ with model parameter count $N$, providing a quantitative relationship between model size and channel noise.
  \item \textbf{Temperature $T$}: by \Cref{rem:temperature}, increasing $T$ flattens the distribution, monotonically increasing $H_T$. At $T = 0$, $H = 0$ (deterministic); at $T = \infty$, $H = \log|\mathcal{V}|$ (uniform noise).
  \item \textbf{Task difficulty}: for inputs $x$ where the model is uncertain (ambiguous, out-of-distribution, or requiring multi-step reasoning), the conditional distribution is naturally more diffuse, yielding higher $\mathcal{H}$ even at fixed $T$.
\end{enumerate}

\begin{remark}[Model-difficulty duality, an organizing observation]
\label{cor:duality}
For a given technique with decoding threshold $\mathcal{H}^*$, the technique is effective when $\mathcal{H}_T(\calA_\theta, x) < \mathcal{H}^*$.
Since $\mathcal{H}$ depends on both $\theta$ (model) and $x$ (task), a 70B model on a hard task and a 3B model on an easy task may occupy the same operating point.
Technique selection should therefore be conditioned on the \emph{operating point} $\mathcal{H}$, not on model size or task difficulty independently.
This is a way of organizing technique choice, not a tested claim.
\end{remark}

\subsection{Mutual information and diversity}

For $d$ independent channels $\calA_{\theta_1}, \ldots, \calA_{\theta_d}$ (independence holds when models have different architectures and training data, or when temperature variation provides sufficient randomization), the mutual information satisfies:
\begin{equation}
  I(X; Y_1, \ldots, Y_d) = I(X; Y_1) + \sum_{i=2}^{d} I(X; Y_i \mid Y_1, \ldots, Y_{i-1})
  \label{eq:chain-mi}
\end{equation}
Under independence, $I(X; Y_i \mid Y_1, \ldots, Y_{i-1}) = I(X; Y_i)$, giving:
\begin{equation}
  I(X; Y_1, \ldots, Y_d) = \sum_{i=1}^{d} I(X; Y_i)
\end{equation}
Each additional independent channel adds information about $X$, which is why diversity combining improves quality.
When channels are correlated (shared training data, overlapping architectures), $I(X; Y_i \mid Y_1, \ldots, Y_{i-1}) < I(X; Y_i)$, yielding diminishing returns.
The effective diversity order $d_{\text{eff}} < d$ captures this: it is the number of channels' worth of independent information in the ensemble.
\citet{turkmen2026ensemble} formalize this for LLM ensembles, showing that MI-based ensemble selection outperforms greedy accuracy-based selection by accounting for inter-model correlation.

\paragraph{Empirical branch correlation.} \Cref{tab:branch-correlation}
reports the Pearson correlation $r$ between the two channels'
\texttt{individual\_scores} on Diversity-SC, the only technique whose
per-branch quality scores are pure parallel single-shot draws, counting
every generation once. Aggregate $r$ ranges from $0.17$ at 14B local to
$0.68$ on the Ollama-cloud curated tasks, giving an empirical
$d_{\mathrm{eff}} \in [1.19, 1.71]$ from a nominal $d{=}2$: the two
channels fail together often enough to deliver well under two independent
looks at each task. On the 300-item hard split the aggregate $r{=}0.59$
hides a sharp split by task family (\Cref{tab:branch-correlation-cat}):
$r{=}-0.03$ on MMLU, where the two channels fail nearly independently,
against $r{=}+0.58$ on HumanEval. Correlation alone does not predict where
equal-gain combining overtakes maximal-ratio combining: the 14B pair has the
lowest $r$ yet the largest synthesis advantage for equal gain
(\Cref{tab:prop1-spaces}), consistent with \Cref{prop:csi-crossover},
in which the reversal is driven by noise in the channel-state estimates.

\begin{table}[h]
\centering
\caption{Inter-branch quality-score correlation on Diversity-SC, the only technique whose \texttt{individual\_scores} are pure parallel single-shot draws (no synthesis, no critic, no combining). Pearson $r$ between the $d{=}2$ per-branch judge scores, with every generation counted once ($n$ draws; configurations generated three times contribute all three); 95\% CIs from $4000$ bootstrap resamples of tasks; $d_{\mathrm{eff}} = d/(1{+}(d{-}1)r)$ clamped at $r\geq 0$. Correlated channels deliver well under the nominal diversity order $d{=}2$.}
\label{tab:branch-correlation}
\small
\setlength{\tabcolsep}{4pt}
\begin{tabular}{@{}lcccc@{}}
\toprule
\textbf{Configuration} & $\bm{n}$ & \textbf{Pearson $r$} & \textbf{95\% CI} & $\bm{d_{\mathrm{eff}}}$ \\
\midrule
3B local & 207 & +0.584 & [+0.446, +0.704] & 1.26 \\
3B local + cloud judge & 207 & +0.609 & [+0.461, +0.729] & 1.24 \\
8B local & 69 & +0.464 & [+0.222, +0.672] & 1.37 \\
14B local & 69 & +0.171 & [-0.080, +0.407] & 1.71 \\
Anthropic + OpenAI cloud & 69 & +0.575 & [+0.308, +0.759] & 1.27 \\
Ollama-cloud (n=69 curated) & 69 & +0.679 & [+0.516, +0.813] & 1.19 \\
Ollama-cloud (n=300 hard) & 900 & +0.594 & [+0.509, +0.672] & 1.25 \\
\bottomrule
\end{tabular}
\end{table}

\begin{table}[h]
\centering
\caption{Per-category inter-branch correlation on Diversity-SC. Each cell is the Pearson $r$ between the two branches' judge scores on tasks in that category, with the number of draws in parentheses; n/a marks cells with fewer than five draws. One configuration can combine nearly independent failures in one task family with strongly shared failures in another.}
\label{tab:branch-correlation-cat}
\footnotesize
\setlength{\tabcolsep}{4pt}
\begin{tabular}{@{}lccccc@{}}
\toprule
\textbf{Configuration} & \textbf{QA} & \textbf{Reasoning} & \textbf{Code} & \textbf{Creative} & \textbf{All} \\
\midrule
3B local & +0.69 (48) & +0.28 (90) & +0.51 (42) & +0.90 (27) & +0.58 (207) \\
3B local + cloud judge & +0.49 (48) & +0.52 (90) & +0.53 (42) & +0.46 (27) & +0.61 (207) \\
8B local & +0.46 (16) & +0.11 (30) & +0.48 (14) & +0.72 (9) & +0.46 (69) \\
14B local & +0.26 (16) & +0.30 (30) & +0.05 (14) & +0.66 (9) & +0.17 (69) \\
Anthropic + OpenAI cloud & +0.87 (16) & +0.61 (30) & -0.18 (14) & +0.58 (9) & +0.57 (69) \\
Ollama-cloud (n=69 curated) & +0.89 (16) & +0.74 (30) & +0.25 (14) & +0.59 (9) & +0.68 (69) \\
Ollama-cloud (n=300 hard) & -0.03 (300) & +0.29 (300) & +0.58 (300) & n/a & +0.59 (900) \\
\bottomrule
\end{tabular}
\end{table}

\subsection{Soft vs.\ hard channel state information}

In classical communications, the combining weights for MRC are derived from the channel's soft output.
For a Gaussian channel with known noise variance $\sigma_i^2$ on branch $i$, the optimal MRC weights are $w_i = 1/\sigma_i^2$ (inverse noise variance weighting)~\citep{proakis2008digital}.
When the noise variance must be estimated, any estimation error $\hat{\sigma}_i^2 \neq \sigma_i^2$ degrades MRC performance.
Under Gaussian weighting errors~\citep{gans1971effect,tomiuk1999general}, the first-order output SNR~\eqref{eq:approx-mrc-snr} decays from $\gamma_{\mathrm{MRC}}$ toward $\gamma_{\mathrm{MRC}}/d$ and crosses $\gamma_{\mathrm{EGC}}$ at the variance $\sigma_w^{*2}$ of \Cref{prop:csi-crossover}.

In our framework, the intrinsic CSI is the log-probability $\log p_\theta(y \mid x)$, which is a direct measure of the channel's confidence.
The extrinsic CSI is the LLM judge score.
Naive numeric-rating judges are severely quantized: our initial judge produced only $\sim$33 unique values across 516 measurements ($\sim$5-bit resolution), under which the theoretical MRC advantage over EGC vanishes~\citep{simon2005digital}.
We therefore adopt a \emph{sigma-delta} scoring design: the judge answers 15 weighted binary criteria rather than emitting a single number, and the weighted sum recovers a fine-grained score (up to $2^{15}$ distinct values). Binary decisions are far more reliable for small LLMs than numeric ratings, and the many-1-bit-measurements-to-high-resolution-output construction mirrors oversampled sigma-delta ADCs~\citep{norsworthy1997delta}.

\textbf{Practical implication.} Log-probability-weighted MRC would use the channel's own soft output, avoiding the estimation noise entirely.
However, sequence-level log-probability $\log p_\theta(y \mid x) = \sum_t \log p_\theta(y_t \mid y_{<t}, x)$ has known biases:
(i)~\emph{length bias}: shorter sequences tend to have higher log-probability, favoring terse but incomplete answers;
(ii)~\emph{repetition bias}: repeated tokens have artificially high conditional probability;
(iii)~\emph{calibration}: LLM probabilities are not perfectly calibrated~\citep{guo2017calibration, kadavath2022language}.
These can be mitigated by length normalization (\Cref{eq:intrinsic}) and temperature calibration, but require empirical validation.

\subsection{Iterative decoding threshold}

The convergence behavior of iterative decoders (turbo codes, LDPC) is governed by the decoding threshold~\citep{ten_brink2001convergence}: a channel noise level $\mathcal{H}^*$ such that:
\begin{itemize}[nosep]
  \item For $\mathcal{H} < \mathcal{H}^*$ (good channel): iterative decoding converges, and error probability decreases exponentially with iterations.
  \item For $\mathcal{H} > \mathcal{H}^*$ (bad channel): the decoder \emph{cannot converge}. Additional iterations may even degrade performance due to error propagation.
\end{itemize}
The threshold depends on the constituent decoder capacity, not on the interleaver design.

In the agent setting, the generator is the constituent decoder and the critic supplies the extrinsic information. Our multi-scale data are consistent with a threshold that depends on generator strength, in graded form: the share of turbo trajectories that end below their starting quality is $26\%$ at 3B (three runs) and at 8B and falls to $5\%$ at 14B, while the mean per-trajectory gain rises from $+6.2$ and $+5.1$ pp to $+19.7$ pp (\Cref{app:threshold-details}). The gaps between HARQ-IR and HARQ-CC at a single scale ($0.1$ and $2.3$ points of relative gain at 8B and 14B, under $2$ pp in absolute quality) lie below single-run resolution and carry no weight here. We stop short of claiming a waterfall: neither the number of iterations nor the number of tasks resolves a sharp transition of the kind seen in bit-error-rate curves, and no configuration in our data is net-harmful on average.

\paragraph{A fixed-point characterization of the threshold.}
The threshold can be made precise by modeling the per-iteration map on quality directly. Let $\mathcal{R}_x: \mathcal{Y} \to \mathcal{Y}$ be the refinement operator (generator + critic + extrinsic-information exchange) for prompt $x$, and define the \emph{quality-update map}
\begin{equation*}
  f(q_0) \;:=\; \E\!\left[\, q\big(\mathcal{R}_x(Y)\big) \;\middle|\; q(Y) = q_0 \,\right],
\end{equation*}
with $q_k := f^{(k)}(q_0)$ the $k$-fold deterministic iterate of expected quality, and $Q_k := \E\big[\max_{0 \leq j \leq k} q(Y_j)\big]$ the expected quality under the best-of-sequence guard (\Cref{app:impl-primitives}). The guarded delivered sequence is monotone independently of the dynamics of $f$:

\begin{remark}[Best-of-sequence monotonicity]
\label{rem:bos-guard}
For every refinement trajectory $(Y_0, Y_1, \dots)$, the running maximum $\max_{0 \leq j \leq k} q(Y_j)$ is non-decreasing in $k$ pathwise; taking expectations gives $Q_k \geq Q_{k-1} \geq \cdots \geq q_0$ for all $k \geq 1$.
\end{remark}

\noindent This is a one-line consequence of monotonicity of the maximum over a growing index set; it makes no use of the structure of $f$, and we use it throughout to read the empirical running-max curves of \Cref{fig:turbo-waterfall,fig:threshold-3panel} as predictions about the \emph{delivered} output even when the unguarded iterate $q_k$ behaves badly. The substantive content of the threshold therefore lies in the local stability of the high-quality fixed point of $f$.

\begin{proposition}[Refinement contraction and the iterative-decoding threshold]
\label{prop:refinement-threshold}
Suppose $f: [0, 1] \to [0, 1]$ is continuously differentiable with at least one fixed point $q^\infty \in (0, 1]$.
\begin{enumerate}[nosep,leftmargin=*]
  \item[(i)] \emph{Above threshold.} If $|f'(q^\infty)| < 1$, then for every $\kappa \in \big(|f'(q^\infty)|, 1\big)$ there exists a neighborhood $U \ni q^\infty$ with $f(U) \subseteq U$ such that
  \begin{equation*}
    |f(q) - q^\infty| \;\leq\; \kappa\,|q - q^\infty| \qquad \text{for all } q \in U,
  \end{equation*}
  and consequently, for every $q_0 \in U$,
  \begin{equation*}
    |q_k - q^\infty| \;\leq\; \kappa^{k}\,|q_0 - q^\infty| \;\xrightarrow{k \to \infty}\; 0.
  \end{equation*}
  Identifying $q_k$ with the conditional mean $\E[q(Y_k) \mid q_0]$ (the EXIT-recursion reading of $f$), this gives $\liminf_{k} Q_k \geq q^\infty$; the squeeze $q_k \leq Q_k \leq q^\infty$ forces $Q_k \to q^\infty$ whenever $q(Y_j) \leq q^\infty$ almost surely along the trajectory (in particular when $q^\infty = 1$).
  \item[(ii)] \emph{Below threshold.} If, in addition, $f$ is non-decreasing on $[0,1]$, $|f'(q^\infty)| > 1$, and $f$ admits a second stable fixed point $q^- < q^\infty$ with $|f'(q^-)| < 1$ and no other fixed point in $[q^-, q^\infty]$, then for every $q_0 \in (q^-, q^\infty)$ the unguarded iterate satisfies $q_k \searrow q^-$, while the guarded sequence still satisfies $Q_k \geq q_0$ for all $k$ by \Cref{rem:bos-guard}.
\end{enumerate}
\end{proposition}

\begin{proof}
\emph{(i) Local contraction.} Fix $\kappa \in (|f'(q^\infty)|, 1)$. Continuity of $f'$ at $q^\infty$ gives $\delta > 0$ such that $|f'(q)| \leq \kappa$ for all $q \in U := (q^\infty - \delta, q^\infty + \delta) \cap [0,1]$. The set $U$ is convex and contains $q^\infty$, so for any $q \in U$ the segment from $q$ to $q^\infty$ lies in $U$ and the mean-value theorem yields some $\xi \in U$ with $f(q) - q^\infty = f(q) - f(q^\infty) = f'(\xi)\,(q - q^\infty)$, hence
\begin{equation*}
  |f(q) - q^\infty| \;\leq\; \kappa\,|q - q^\infty| \;<\; \delta.
\end{equation*}
Combined with $f(q) \in [0,1]$, this gives $f(q) \in U$, so $f(U) \subseteq U$ and the iteration $q_{k+1} = f(q_k)$ remains in $U$ once $q_0 \in U$. Induction on $k$ now yields $|q_k - q^\infty| \leq \kappa\,|q_{k-1} - q^\infty| \leq \cdots \leq \kappa^k\,|q_0 - q^\infty|$, and $\kappa < 1$ forces $q_k \to q^\infty$.

\emph{Guarded sequence.} Pathwise, $\max_{0 \leq j \leq k} q(Y_j) \geq q(Y_k)$, so under the mean-field reading $\E[q(Y_k)] = q_k$ we have $Q_k \geq q_k$, hence $\liminf_k Q_k \geq q^\infty$. The reverse bound $Q_k \leq q^\infty$ holds whenever $q(Y_j) \leq q^\infty$ almost surely along the trajectory, and the two together give $Q_k \to q^\infty$.

\emph{(ii)} Set $g(q) := f(q) - q$. Then $g(q^-) = g(q^\infty) = 0$ and, since $|f'(q^-)| < 1$, $g'(q^-) = f'(q^-) - 1 < 0$; together with the assumption that $g$ has no other zero in $[q^-, q^\infty]$, this fixes the sign of $g$ as negative on the open interval, i.e.\ $f(q) < q$ on $(q^-, q^\infty)$. Monotonicity of $f$ gives $f(q) \geq f(q^-) = q^-$, so $f$ maps $(q^-, q^\infty)$ into $[q^-, q^\infty)$ and $q_k$ is well-defined and non-increasing there. Bounded below by $q^-$, $q_k$ converges by monotone convergence and continuity to a fixed point of $f$ in $[q^-, q^\infty)$, which can only be $q^-$. The guard inequality is \Cref{rem:bos-guard}.
\end{proof}

\paragraph{Attribution.} The mathematical content of \Cref{prop:refinement-threshold} is a specialization of the classical local-stability theory of one-dimensional smooth maps to the present setting (see, e.g., \citealp[Ch.~15]{hirsch_smale_devaney2013}; the contraction case is the one-dimensional Banach--Picard corollary, and the bistable case is the standard saddle-node picture). The interpretation of the dichotomy $|f'(q^\infty)| \lessgtr 1$ as a \emph{decoding threshold} is the EXIT-chart / density-evolution analysis of \citet{ten_brink2001convergence} and \citet{richardson2001capacity}, in which the threshold is the channel parameter at which the high-quality fixed point of a one-dimensional recursion loses local stability. Our contribution here is the modeling step, identifying the LLM refinement operator $\mathcal{R}_x$ as the constituent decoder, the critique channel as the extrinsic-information exchange, and the quality-update map $f(q_0) = \E[q(\mathcal{R}_x(Y)) \mid q(Y) = q_0]$ as the EXIT recursion; we restate the underlying dynamical-systems result here so that the threshold prediction in \Cref{fig:threshold-3panel} can be read off the same $|f'|$ condition that controls turbo and LDPC convergence.

\textbf{Interpretation.} The condition $|f'(q^\infty)| < 1$ at the high-quality fixed point is the agent analog of operating above the iterative-decoding threshold of a turbo code~\citep{ten_brink2001convergence}: refinement is locally contractive toward $q^\infty$. When the condition fails, unguarded trajectories can drift toward a lower-quality basin, and only the best-of-sequence guard keeps the delivered output from following them. Empirically (\Cref{fig:threshold-3panel}), the unguarded iterate degrades more often and gains less for weaker generators and for initial answers already close to the top of the scale, while the guarded running maximum is monotone by construction (\Cref{rem:bos-guard}).

\textbf{What the theorem does and does not give.} It does not estimate $|f'(q^\infty)|$ for any concrete generator-critic pair; a direct estimate would need many independent refinement trajectories per quality bin. Its testable content is local in initial quality and graded across generators, and both directions are observed (\Cref{app:threshold-details}); the full dichotomy, a pair operating strictly below threshold, is not.

\subsection{Bounds on agent channel capacity}

While the exact capacity $C = \max_{p(x)} I(X; Y)$ is intractable for LLM channels, useful bounds exist.

\textbf{Upper bound.} For any output distribution, $I(X; Y) \leq H(Y) \leq \E[|Y|] \cdot \log|\mathcal{V}|$.
A tighter bound uses the observation that LLM outputs occupy a tiny fraction of $\mathcal{V}^*$: natural language has approximately 1--1.5 bits/character of entropy~\citep{shannon1948, cover2006elements}, giving $H(Y) \approx 1.2 \cdot \E[|Y|]$ bits.

\textbf{Lower bound.} Any achievable quality $Q^*(C)$ under cost budget $C$ provides a lower bound, since the technique configuration that achieves $Q^*$ is a specific (suboptimal) coding scheme.
Our ACM experiments (\Cref{sec:exp-acm}) thus provide constructive lower bounds.

\textbf{Rate interpretation.} In our setting, the ``rate'' is the semantic information density: how many bits of task-relevant content are transmitted per token of model output.
FEC at code rate $r$ trades this density for reliability: lower $r$ adds more parity tokens, reducing information density but increasing the fraction of tasks solved correctly.
This directly parallels the rate-reliability tradeoff in Shannon's coding theorem.

\section{Hyperparameter Settings and Pseudocode}
\Cref{tab:hyperparams} lists every hyperparameter as recorded in the run configurations, and \Cref{alg:diversity} gives the diversity combiner in pseudocode.
\label{app:hyperparams}

\begin{algorithm}[h]
\caption{Diversity ensemble with MRC combining}
\label{alg:diversity}
\begin{algorithmic}[1]
\Require Task $X$, channels $\{\calA_i\}_{i=1}^{d}$, scorer $q$, synthesizer $\calA_s$
\For{$i = 1$ to $d$} \Comment{Parallel}
  \State $Y_i \gets \calA_i(X)$; \quad $q_i \gets q(X, Y_i)$
\EndFor
\State $w_i \gets q_i / \sum_j q_j$ \quad $\forall i$ \Comment{MRC weights}
\State \Return $\calA_s(\text{Synthesize}(\{(Y_i, w_i)\}))$
\end{algorithmic}
\end{algorithm}

\begin{table}[h]
\centering
\scriptsize
\caption{Hyperparameters as recorded in the run configurations; the pseudocode of \Cref{alg:diversity} and the descriptions of \Cref{app:impl-details} use these values. ``Generator'' means the critic is the generating model itself (self-critique). Two runs predate the turbo damping fields (8B and the cloud pair), so their values for those three fields are not recorded.}
\label{tab:hyperparams}
\begin{tabular}{@{}p{0.36\linewidth}p{0.58\linewidth}@{}}
\toprule
\textbf{Parameter} & \textbf{Value} \\
\midrule
\multicolumn{2}{@{}l}{\emph{Sampling}} \\
\quad Channel temperature; judge temperature & 0.7 unless a technique sets its own; 0.1 \\
\midrule
\multicolumn{2}{@{}l}{\emph{Diversity (SC / MRC / EGC)}} \\
\quad Branches $d$; prompt variants; temperature spread & 2 (one per channel model); 1; none \\
\quad Wider pools (Diversity-SC-$N$, -MRC-Discrete-$N$) & $N{=}5$, channels in round-robin, temperature $0.7$ \\
\midrule
\multicolumn{2}{@{}l}{\emph{HARQ-CC and HARQ-IR}} \\
\quad Max rounds $K$; threshold $\tau$; early exit & 5; 0.85; off (all rounds run unless $q \geq \tau$) \\
\quad Critic & Generator (Gemma-3 12B at 8B) \\
\midrule
\multicolumn{2}{@{}l}{\emph{Turbo}} \\
\quad Max iterations; threshold $\tau$ & 5; 0.85 (0.90 on the cloud pair) \\
\quad Critic & Generator (Qwen-2.5 7B at 8B) \\
\quad Damping $\alpha_0$; severity floor; corrections per round & 0.5; major; at most 2 \\
\midrule
\multicolumn{2}{@{}l}{\emph{Fountain}} \\
\quad Samples $[N_{\min}, N_{\max}]$; confidence threshold & $[2, 8]$; 0.85 (both channels in round-robin; Haiku only on the cloud pair) \\
\midrule
\multicolumn{2}{@{}l}{\emph{FEC}} \\
\quad Code rates & 0.75 (reasoning), 0.50 (+ verification), 0.33 (+ alternative solution) \\
\midrule
\multicolumn{2}{@{}l}{\emph{Prior methods}} \\
\quad Self-Consistency, Best-of-$N$, Weighted Best-of-$N$ & $N{=}5$ \\
\quad CISC & $N{=}5$; self-reported confidence (0 to 100), softmax temperature 8 \\
\quad Mixture-of-Agents & 5 proposals over 2 channels, aggregated by the first channel model \\
\quad Self-Refine; Chain-of-Verification & 3 rounds; 3 verification questions \\
\midrule
\multicolumn{2}{@{}l}{\emph{Routers}} \\
\quad Hand-coded ACM & 5 levels (MCS-0 to MCS-4) \\
\quad Semantic-KNN router & BAAI/bge-large-en-v1.5 embeddings, $k{=}20$, $\lambda \in \{1,3,4,5,7,10,15,20,35,38,40,50,100,1000\}$ \\
\bottomrule
\end{tabular}
\end{table}

\section{Detailed implementation of each technique}
\label{app:impl-details}

This appendix documents the concrete implementation of every technique benchmarked in this paper, at the level of detail required for an independent reimplementation. Together with \Cref{app:data,app:hyperparams} and the curated tasks released as supplementary material, it suffices to reproduce every result in this paper; an extended version of \agentcodec{} is commercially available at \url{https://intellerce.com/products/agentcodec}. All techniques are built from two shared primitives: the \emph{agent channel} $\calA$ and the \emph{quality scorer} $q$. Before describing individual techniques we fix notation and document these primitives, since the guards and guarantees of each technique rely on them.

\subsection{Shared primitives}
\label{app:impl-primitives}

\paragraph{Agent channel $\calA$.} A channel is a callable $\calA \colon \text{prompt} \times \text{temperature} \to \text{output}$ that wraps an LLM endpoint (OpenAI-compatible, Anthropic native, or a local Ollama/vLLM server). A call returns an \texttt{AgentOutput} record with fields
\texttt{text}, \texttt{model}, \texttt{temperature}, \texttt{token\_count}, \texttt{cost\_usd}, \texttt{latency\_s}, and, optionally, \texttt{token\_logprobs} and \texttt{mean\_logprob} when the backend supports them. Reasoning-model thinking blocks (\texttt{<think>\ldots</think>} for DeepSeek-R1 / Qwen3; Anthropic \texttt{ThinkingBlock} for Claude) are stripped from the visible text before downstream use, but the thinking tokens remain in the cost accounting. Whenever possible we disable thinking.

\paragraph{Quality scorer $q$.} The scorer is a \emph{binary checklist judge}: given (prompt, candidate, optional reference) it asks a dedicated LLM (the judge of \Cref{tab:model-configs}) to evaluate 15 weighted yes/no criteria and returns the weighted true-rate in $[0, 1]$. The reference is supplied only when final outputs are scored for evaluation; calls made by techniques during decoding omit it. Criteria are fixed per task type (with / without reference) and weights sum to $1$; a ``good but not perfect'' answer is calibrated to pass 10--12 of 15 checks. This produces $2^{15}=32{,}768$ possible score values, avoiding the severe quantization of numeric-rating judges (\Cref{app:exp-methodology}). Blending with objective checks is implemented but disabled in all reported runs, so every score a technique or the router uses is judge-only (\Cref{app:exp-methodology}).

The scorer additionally exposes a differential variant $q_\Delta$: to compare a candidate against a baseline with known score $q_0$, the judge scores both in the same call pattern, and the candidate score is reported as $q_0 + (q_{\text{cand}} - q_{\text{base-rescore}})$, clamped to $[0, 1]$. This cancels common-mode judge noise and is used by every technique that refines a previous answer (HARQ-IR, turbo, FEC post-decode, diversity synthesis).

\paragraph{Output contract.} Every technique returns a \texttt{ReliabilityRun} record containing: \texttt{task\_id}, \texttt{technique}, \texttt{individual\_outputs} (all generator calls), \texttt{overhead\_outputs} (synthesizer/critic/decoder/judge calls not directly part of the answer), \texttt{combined\_output} (the final answer text), \texttt{final\_quality} (the judge score of the combined output), and \texttt{rounds} (iterations used). Cost accounting sums all \texttt{cost\_usd} fields across \texttt{individual\_outputs $\cup$ overhead\_outputs $\cup$ judge calls}. Techniques that synthesize or refine (diversity~MRC/EGC, fountain, HARQ, turbo, FEC) implement a \emph{best-of-sequence guard}: if the post-synthesis score is lower than the best individual score seen, the final output reverts to the best individual output. This guard is what makes the running-max curves in \Cref{fig:threshold-3panel} monotonic regardless of refinement quality.

\subsection{Baseline (uncoded transmission)}
\label{app:impl-baseline}

A single channel call with the task prompt. No parity, no diversity, no feedback:
$Y \gets \calA(X)$, $q_0 \gets q(X, Y)$, return $Y$ with score $q_0$. This is the ``uncoded'' reference point; all gains are reported relative to it after paired matching on \texttt{task\_id} (and, when \texttt{repeat\_runs}${>}1$, on \texttt{(task\_id, repeat\_idx)}).

\subsection{Diversity ensemble (SC / MRC / EGC)}
\label{app:impl-diversity}

\paragraph{Analog.} Multipath diversity. Generate $d$ noisy copies and combine: SC picks the best, MRC weights by quality, EGC weights equally.

\paragraph{Algorithm.} Let $\calA_1, \ldots, \calA_d$ be channels (spatial diversity: different model families; frequency diversity: different prompt variants of the same question; time diversity: different temperature draws of the same model). Generate $Y_i \gets \calA_i(X)$ and score all in parallel: $q_i \gets q(X, Y_i)$.
\begin{itemize}[nosep,leftmargin=*]
  \item \textbf{SC}: return $Y_{i^*}$ where $i^* = \argmax_i q_i$. No synthesizer call.
  \item \textbf{MRC / EGC}: pass the full list $\{(Y_i, q_i, \text{model}_i)\}$ to a synthesizer LLM $\calA_s$ (the judge model, which has already seen the candidates during scoring).
\end{itemize}
The MRC synthesis prompt marks the top-scoring response as \texttt{[BEST]} and others as \texttt{[ALT-i]} with their quality scores, and instructs: start from \texttt{[BEST]}; add only details from \texttt{[ALT-i]} that do not contradict \texttt{[BEST]}; prefer higher-quality responses on conflicts; do not introduce external facts. The EGC prompt is similar but presents all responses as ``equally weighted'' and uses majority reasoning for conflicts. Synthesis is called with temperature $0.1$ (MRC) or $0.2$ (EGC).

\paragraph{Guards.}
\emph{(i)~MRC dominance fast path.} If all non-best scores satisfy $q_i < 0.5 \cdot q_{\max}$, synthesis is skipped and SC is returned, analogous to one branch dominating in AWGN MRC.
\emph{(ii)~Best-of-sequence.} The synthesis output is rescored via $q_\Delta$ against the best individual. If the result does not improve over the best individual, SC output is returned instead. This enforces the combining guarantee $\text{MRC}\geq\text{SC}$ empirically, even when the synthesizer has regressions.
\emph{(iii)~Identity detection.} If the synthesizer returns the best individual text verbatim, no rescore is performed (the best-individual score is used directly, saving a judge call).

\paragraph{Defaults.} $d=2$ spatial branches (two different model families); prompt variants $=1$ (default wording only); temperature spread unused in the main tables (all branches at each model's configured default temperature).

\paragraph{Wider-pool variants: \texttt{diversity\_sc\_N} and \texttt{diversity\_mrc\_discrete\_N}.}
The two wider-pool operators referenced in \Cref{tab:operator-view} share the same multi-model sampling rule and differ only in the combiner. \emph{Sampling:} both draw $N$ candidates by cycling through the configured channel pool in round-robin order: sample $i$ is generated by channel index $i \bmod |\{\calA_1,\ldots,\calA_d\}|$ at temperature $0.7$, so an $N{=}5$ run with $d{=}2$ channels uses $\{\calA_1,\calA_2,\calA_1,\calA_2,\calA_1\}$. Each candidate is independently scored by the judge to obtain $q_i$.
\emph{\texttt{diversity\_sc\_N}:} return the argmax-$q_i$ sample. No synthesizer or voter call; total cost is $N$ generator calls $+ N$ judge calls.
\emph{\texttt{diversity\_mrc\_discrete\_N}:} cluster the $N$ samples into semantic equivalence classes via one extra \emph{voter} LLM call (defaults to $\calA_1$) at temperature $0.0$; the voter receives the task plus all $N$ samples and returns a JSON array $[\ell_1, \ldots, \ell_N]$ of cluster IDs. We sum $q_i$ within each cluster, pick the cluster with the largest sum, and return its top-$q_i$ member as the combined output. Failure handling: if voter output cannot be parsed as a length-$N$ JSON integer array, every sample is assigned to its own cluster and the operator degenerates to \texttt{diversity\_sc\_N} on that task. Total cost is $N$ generator calls $+ N$ judge calls $+ 1$ voter call; the voter call is recorded in the run's overhead bucket. \textbf{Defaults.} $N=5$; voter $= \calA_1$ (the first configured channel). The operator's continuous-MRC counterpart (\texttt{diversity\_mrc} with $d=N$) trades the voter call for an LLM synthesis call, so the per-task call counts are equal up to that one substitution; the empirical comparison in \Cref{sec:exp-baselines} attributes the discrete-MRC win at 8B and its 3B-QA reversal to voter-LLM clustering quality rather than to a difference in budget.

\paragraph{Discrete-MRC vs.\ CISC: orthogonal generalizations.}
CISC~\citep{taubenfeld2025confidence} and \texttt{diversity\_mrc\_discrete\_N} both apply discrete MRC (cluster the samples, sum the CSI per cluster, return the top-CSI member of the winning cluster) but specialize different axes of the same operator. Our CISC implementation fixes the pool to one policy ($N{=}5$ temperature draws) and uses the model's self-reported confidence, a number from 0 to 100 that it appends to each answer, passed through a softmax with temperature $8$, as the per-sample weight; clusters are formed by the same voter call. \texttt{diversity\_mrc\_discrete\_N} instead uses the judge score as CSI and samples a multi-model pool in round-robin. A third variant, \emph{Diversity-MRC-Discrete-$N$-Soft}, uses the multi-model pool with token log-probability CSI (\Cref{app:impl-soft}); it runs on the local configurations only, since the Anthropic API does not expose log-probabilities. Log-probabilities from different model families are not on a common scale, so summing them within a cluster is a heuristic rather than a calibrated weighting.

\subsection{Hybrid ARQ with Chase Combining (HARQ-CC)}
\label{app:impl-harq-cc}

\paragraph{Analog.} Chase combining: same codeword retransmitted, receiver soft-combines all copies.

\paragraph{Algorithm.} Up to $K=5$ retransmissions of the \emph{same} prompt to the \emph{same} channel (with independent stochastic sampling each call). After each retransmission $k$, if $q_k \geq \tau=0.85$, stop and return $Y_k$. Otherwise, at the end, pass all $\{Y_1, \ldots, Y_K\}$ to a \emph{chase-combining synthesizer}: a prompt that shows every attempt with its quality score and instructs to (1)~detect consensus (high-confidence signal), (2)~harvest unique details from individual attempts, (3)~resolve conflicts in favor of higher-quality attempts. The synthesized output is rescored via $q_\Delta$ against the best individual, with the best-of-sequence guard applied.

\paragraph{Defaults.} $K=5$, $\tau=0.85$, synthesizer = the critic channel (same model family as the generator, communication-faithful) at temperature $0.3$.

\subsection{Hybrid ARQ with Incremental Redundancy (HARQ-IR)}
\label{app:impl-harq-ir}

\paragraph{Analog.} Incremental redundancy: each retransmission carries \emph{new parity bits}. In our analog, the ``new parity'' is structured critic feedback that points to specific errors in the previous attempt.

\paragraph{Algorithm.} Let $\calA_g$ be the generator channel and $\calA_c$ be the critic channel (default: $\calA_c = \calA_g$, matching real HARQ decoders).
\begin{enumerate}[nosep,leftmargin=*]
  \item \emph{Round 1.} $Y_1 \gets \calA_g(X)$; $q_1 \gets q(X, Y_1)$ (absolute scoring). Initialize $Y^{\star} \gets Y_1$, $q^{\star} \gets q_1$, score history $\calS \gets [q_1]$, correction buffer $\calC \gets \emptyset$. If $q_1 \geq \tau=0.85$, return.
  \item \emph{Rounds $k = 2, \ldots, K=5$.}
  \begin{enumerate}[nosep,leftmargin=*]
    \item \emph{Critic pass.} The critic is shown the task, the current best $Y^{\star}$, the current score $q^{\star}$, and a deduplication list of already-addressed quoted issues from $\calC$. It is instructed to return a JSON array of issues, each a record $\{$\texttt{quote}, \texttt{type}$\in\{$factual\_error, missing\_content, reasoning\_gap, unclear$\}$, \texttt{correction} or \texttt{detail}, \texttt{severity}$\in\{$critical, major, minor$\}\}$. Temperature $0.2$.
    \item \emph{Parse.} A multi-stage JSON parser (raw $\to$ stripped-backtick $\to$ fenced-code-block $\to$ inline-array) tolerates weak critics. Explicit \texttt{[]} or \texttt{PASS} text yields an empty issue list. Parse failures degrade to a single unstructured issue with the raw text.
    \item \emph{Early-stop checks (only when \texttt{early\_exit}=True).}
      If the issue list is empty and $q^{\star} \geq 0.9\tau$, break (genuine convergence).
      If the score history has plateaued over a window of $2$ with spread $< 0.015$, break (EXIT-chart flatness).
    \item \emph{Generator pass.} If the issue list is structured (at least one item has a \texttt{quote}), the refinement prompt becomes a \emph{correction list}: for each issue, show the quoted text plus the requested fix / insertion, and explicitly instruct the generator to preserve every other part of the answer verbatim. If the issues are unstructured, fall back to a full-rewrite prompt. $Y_k \gets \calA_g(\text{refine-prompt})$; $q_k \gets q_\Delta(X, Y_k, Y^{\star}, q^{\star})$ (differential scoring against the current best).
    \item \emph{Accept / reject.} If $q_k \geq q^{\star}$, set $Y^{\star} \gets Y_k$, $q^{\star} \gets q_k$, and append applied corrections to $\calC$. Else, keep $Y^{\star}$ unchanged (regression protection) and do not enlarge $\calC$ (so the same issues remain available to the critic next round).
    \item \emph{Score-history update.} $\calS \gets \calS \cup \{q_k\}$. If $q^{\star} \geq \tau$, break.
  \end{enumerate}
\end{enumerate}
Return $Y^{\star}$ with score $q^{\star}$.

\paragraph{Communication-faithful defaults.} \texttt{early\_exit}=False (all $K$ rounds run unless $q^{\star} \geq \tau$, matching real decoders); $\calA_c = \calA_g$ (same-complexity decoder), except at 8B, where the critic was Gemma-3 12B.

\paragraph{Key design choices.}
\emph{(i)~Correction-based refinement} (non-destructive accumulation): the generator is prompted to \emph{apply specific corrections while preserving correct content}, not to rewrite the answer from scratch. This is the agent analog of the soft-buffer accumulation in real HARQ-IR: the previous answer is the ``received signal'' and corrections are accumulated parity bits.
\emph{(ii)~Structured JSON critique} forces the critic to quote specific text and tag severity, enabling deduplication across rounds. Vague feedback is the analog of retransmitting identical parity, with zero incremental information.
\emph{(iii)~Differential scoring per round} reduces judge noise: scoring candidate and baseline in the same judge call cancels common-mode variance.

\subsection{Turbo decoder}
\label{app:impl-turbo}

\paragraph{Analog.} Two SISO decoders exchanging extrinsic information through an interleaver; each iteration one decoder updates its belief using the other's extrinsic LLRs, with a damping factor $\alpha \in (0, 1]$ to prevent oscillation~\citep{vogt2000improving}.

\paragraph{Algorithm.} Let $\calA_g$ be the generator, $\calA_c$ the critic (default: $\calA_c = \calA_g$, ``same'' mode), $q$ the judge. Maintain best $Y^{\star}$, best score $q^{\star}$, correction buffer $\calC$, score history $\calS$, current damping $\alpha \gets \alpha_0 = 0.5$, and consecutive-regression counter $r$.
\begin{enumerate}[nosep,leftmargin=*]
  \item \emph{Iteration 0.} $Y_0 \gets \calA_g(X)$; $q_0 \gets q(X, Y_0)$; initialize. If $q_0 \geq \tau$, return ($\tau{=}0.85$; $0.9$ on the cloud pair).
  \item \emph{Iterations $k = 1, \ldots, K-1=5$} (total max iterations including iter 0 is $6$, but the default configuration caps at $K=5$):
  \begin{enumerate}[nosep,leftmargin=*]
    \item \emph{Interleaver.} Select a lens from $\{$\texttt{correctness}, \texttt{completeness}, \texttt{reasoning}, \texttt{clarity}$\}$ by index $k \bmod 4$. The lens is appended to the critic's system instruction (``pay special attention to X this round, but flag any problem''). This decorrelates critic observations across iterations, the agent analog of the bit-interleaver in real turbo codes. Without the interleaver the critic tends to produce correlated feedback (same issues flagged every round), reducing extrinsic information to near-zero.
    \item \emph{Critic pass.} Same structured-JSON critique as HARQ-IR, but with the lens hint and the dedup list of already-applied quoted issues from $\calC$.
    \item \emph{Extrinsic scaling (the $\alpha$ step).} Apply three filters in sequence to the critic's issue list:
      \begin{itemize}[nosep]
        \item \emph{Severity floor.} Drop issues below \texttt{severity\_floor}$=$\texttt{major} (i.e.\ only \texttt{critical} and \texttt{major} survive). Minor issues on already-polished answers are the critic fabricating problems (the agent analog of LLR magnitude thresholding).
        \item \emph{$\alpha$ scaling.} Keep $\lceil n \cdot \alpha \rceil$ issues (at least one, if any remain), sorted by severity ascending (most severe first). This is the direct agent analog of extrinsic-LLR scaling in turbo decoders: limit how aggressively each iteration can correct.
        \item \emph{Hard cap.} \texttt{max\_corrections\_per\_round}$=2$ issues total (LLR clipping).
      \end{itemize}
    \item \emph{Generator pass.} Apply the $\leq 2$ surviving corrections via the same correction-list prompt as HARQ-IR (preserve-except-corrections contract). Score via $q_\Delta$ against $Y^{\star}$. Append $q_k$ to $\calS$.
    \item \emph{Accept / reject with adaptive damping.}
      If $q_k \geq q^{\star}$: accept, append applied corrections to $\calC$, reset $r \gets 0$, and \emph{relax} $\alpha \gets \min(\alpha_0, 1.2\alpha)$.
      Else: reject (regression protection), increment $r$, and \emph{damp} $\alpha \gets \max(0.1, 0.5\alpha)$.
    \item \emph{Divergence break.} If $r \geq 2$ (two consecutive regressions), break. The EXIT trajectory is diverging and further iterations will consume tokens without recovery. The best-of-sequence guard ensures $Y^{\star}$ is still the best seen.
    \item \emph{Convergence break.} If $q^{\star} \geq \tau$, break.
  \end{enumerate}
\end{enumerate}
Return $Y^{\star}$ with score $q^{\star}$.

\paragraph{Key design choices.}
\emph{(i)~Same-complexity decoder} ($\calA_c = \calA_g$) matches real turbo codes, where both component decoders operate on the same trellis. Allowing a stronger critic is supported (\texttt{critic\_model}="judge") but breaks the analogy.
\emph{(ii)~Communication-faithful fixed iterations}: \texttt{early\_exit}=False runs all $K$ rounds. The divergence break on $r \geq 2$ is activity-dependent; the plateau-detection break is gated by \texttt{early\_exit}=True.
\emph{(iii)~Interleaver lens rotation} is what distinguishes turbo from HARQ-IR in our implementation. With the same critic and no interleaver, the two techniques become identical.

\paragraph{Why $\alpha_0 = 0.5$, severity floor = major.} Empirical analysis on 14B models showed the first refinement round at $\alpha=0.7$ produced the largest regression (hard-task mean $0.627 \to 0.592$), consistent with over-correction even at moderate damping. Lowering $\alpha_0$ to $0.5$ and gating minor issues eliminated the first-round regression without harming overall gain. At 3B, even $\alpha_0=0.5$ leaves $30\%$ of trajectories ending below their start (\Cref{fig:threshold-3panel}); there the best-of-sequence guard protects the delivered answer.

\subsection{Fountain (Rateless) Decoder}
\label{app:impl-fountain}

\paragraph{Analog.} LT / Raptor codes~\citep{shokrollahi2006raptor}: generate an unbounded stream of encoded symbols and stop when the receiver can decode. No fixed code rate.

\paragraph{Algorithm.} Parameters: $N_{\max}=8$, $N_{\min}=2$, confidence threshold $\gamma = 0.85$.
\begin{enumerate}[nosep,leftmargin=*]
  \item For $n = 1, \ldots, N_{\max}$, round-robin across channels $\calA_{1+(n-1)\bmod d}$; temperature perturbed by $0.5 + 0.1(n \bmod 5)$ for stochastic encoding. $Y_n \gets \calA(X, \text{temp})$; $q_n \gets q(X, Y_n)$.
  \item After each sample, if $n \geq N_{\min}$, compute confidence $\gamma_n = 0.6 \cdot \overline{q} + 0.4 \cdot \text{agreement}$, where $\overline{q}$ is the mean score so far and $\text{agreement} = 1 - (\max - \min)$ over the top-$\lceil n/2\rceil$ scores. If $\gamma_n \geq \gamma$, stop.
  \item \emph{ML decode.} Sort samples by quality descending. If the best sample dominates the runner-up by $> 0.20$, return it directly (ML-decode fast path, no synthesis). Else, apply \emph{erasure marking}: drop samples whose score is $> 0.10$ below the best. The surviving set is passed to a weighted-synthesis prompt with weights $w_i = q_i / \sum q_j$ and explicit instructions ``trust the higher-weighted sample unconditionally on conflicts; add details from lower-weighted samples only when they do not contradict''. The synthesized output is scored independently and subjected to the best-of-sequence guard.
\end{enumerate}

\paragraph{Key design choices.} \emph{Erasure marking} (drop samples $> 0.10$ below best) is what distinguishes this from naive ensembling: earlier versions that merged all top-half samples leaked wrong facts from weak samples into the answer (measured $-0.024$ vs best individual across $69$ tasks). The ML decoder now either picks the single best (fast path) or synthesizes only from a tight quality band.

\subsection{Forward Error Correction (FEC)}
\label{app:impl-fec}

\paragraph{Analog.} Structured redundancy: generate systematic bits plus parity sections. Code rate $r = k/n$ controls the redundancy level.

\paragraph{Algorithm.} The systematic bits are the \emph{main answer}: $Y_0 \gets \calA(X)$, $q_0 \gets q(X, Y_0)$. Each parity section is a \emph{separate LLM call} with the main answer in context but a distinct instruction:
\begin{itemize}[nosep,leftmargin=*]
  \item \texttt{reasoning}: re-derive the answer from scratch via step-by-step reasoning; flag any discrepancy with the main answer.
  \item \texttt{verification}: check each claim in the main answer for correctness, logical validity, completeness, and internal consistency; state each issue with a quote, a reason, and a correction.
  \item \texttt{alternative}: solve the task using a \emph{different} approach / strategy without reference to the main answer. Cross-check.
  \item \texttt{confidence}: rate each claim HIGH / MEDIUM / LOW confidence; flag LOW items as potential errors.
\end{itemize}
Code rates map to parity-section lists: $r=1.0 \mapsto []$ (uncoded), $r=0.75 \mapsto [\texttt{reasoning}]$, $r=0.50 \mapsto [\texttt{reasoning}, \texttt{verification}]$, $r=0.33 \mapsto [\texttt{reasoning}, \texttt{verification}, \texttt{alternative}]$, $r=0.25 \mapsto$ all four. The decoder sees the main answer and \emph{all} parity sections as separate numbered blocks and runs a \emph{syndrome-decoding prompt}: for each present parity section, include a bullet describing what to cross-check (``does step-by-step reach the same conclusion?'', ``did verification find errors?'', ``does the independent solution agree?'', ``are LOW-confidence items actually wrong?''); detect inconsistencies; correct; output the final answer. Decoder temperature $0.2$. The decoded output is rescored via $q_\Delta$ against $Y_0$ with the best-of-sequence guard.

\paragraph{Key design choice.} Each parity section is a \emph{separate LLM call}, not a single prompt that asks for everything. Cramming multiple checks into one call produces attention dilution and correlated ``parity'' (all sections restating the main answer). Separate calls give each check the model's full attention and produce genuinely independent redundancy, matching how real FEC computes each parity symbol independently.

\subsection{Adaptive Coding \& Modulation (ACM)}
\label{app:impl-acm}

\paragraph{Analog.} LTE / 5G ACM: a modulation-and-coding scheme (MCS) table is indexed by measured channel quality (CQI); good channels get high-rate / low-protection MCS, bad channels get low-rate / high-protection MCS.

\paragraph{Algorithm.} Profiles (MCS levels) are specified as a list of $\{$\texttt{name}, \texttt{difficulty\_range}, \texttt{model}, \texttt{technique}, \texttt{code\_rate} (for FEC), \texttt{num\_branches} (for diversity), \texttt{max\_rounds} (for HARQ), \texttt{cost\_multiplier}$\}$.
\begin{enumerate}[nosep,leftmargin=*]
  \item \emph{Difficulty estimation (CQI).} The \emph{pilot probe}: a short 2--3 sentence generation on the task, requested with token-level log-probabilities. The probe's mean logprob $\overline{\ell}$ is the CQI; difficulty is $d = 1 - \exp(\overline{\ell})$, clamped to $[0, 1]$. This is a direct analog of channel estimation from pilot symbols in physical-layer ACM. When the backend does not expose logprobs (e.g., Anthropic channels), ACM falls back to an LLM self-rating call (``rate this task's difficulty $[0, 1]$'' at $T{=}0.1$). The probe is billed as overhead.
  \item \emph{Profile selection.} Find the first MCS whose \texttt{difficulty\_range} contains the estimated difficulty; fall through to the maximum-protection MCS on miss.
  \item \emph{Execute.} Instantiate the selected technique (HARQ-IR / turbo / fountain / Diversity-MRC / Diversity-EGC / FEC / uncoded) with the profile's parameters, using the profile's \texttt{model} (if available in the channel pool; else the first channel).
\end{enumerate}
Routing tables are calibrated per model tier. The 3B table (qwen2.5:3b + gemma3:4b) routes MCS-0/1/2 to fountain with $n{=}2$ spatial branches and MCS-3 to Diversity-EGC; fountain dominates every bucket at this scale. The 14B table (deepseek-r1:14b + phi3:14b) uses HARQ-IR (\texttt{max\_rounds}${=}2{\to}3$) for MCS-0/1 and turbo (\texttt{max\_rounds}${=}4{\to}6$) for MCS-2/3/4, reflecting the 14B leaderboard where iterative refinement beats spatial diversity. The Haiku + GPT-5-mini table routes MCS-0/1/2 to HARQ-IR, MCS-3 to Diversity-MRC, and MCS-4 to turbo.

\subsection{Soft-Output Variants}
\label{app:impl-soft}

When the backend exposes token-level log-probabilities (Ollama, vLLM, OpenAI), three techniques have soft-decision counterparts that replace judge-based weights with intrinsic model confidence $c = \exp(\overline{\ell})$ where $\overline{\ell}$ is the mean token log-probability of the generation. Anthropic models do not currently expose logprobs; these modes fail fast rather than silently degrade.

\paragraph{SoftDiversityMRC.} Identical structure to MRC, but the combining weights are $w_i = c_i / \sum_j c_j$ rather than $q_i / \sum_j q_j$. Eliminates the judge call that would otherwise be needed to weight branches, at the cost of a weaker signal (logprob reflects the model's own uncertainty, not output correctness). The SC fallback on synthesis regression is unchanged.

\paragraph{SoftFountainDecoder.} Replaces the quality-based confidence estimator with a soft one: $\gamma_n = 0.6 \cdot \overline{c} + 0.4 \cdot (1 - \text{spread}(c))$, and erasure-marks samples with $c_i < 0.5 c_{\max}$. Low-logprob samples are dropped from synthesis, mirroring how real fountain decoders discard erasure-flagged symbols.

\paragraph{SoftACMRouter.} Replaces the explicit difficulty-estimation LLM call with a \emph{pilot probe}: a short generation (2--3 sentences) on the task, requested with logprobs. The probe's mean logprob \emph{is} the CQI: high logprob $\to$ easy channel $\to$ light protection; low logprob $\to$ hard channel $\to$ heavy protection. Profiles are keyed by confidence range rather than difficulty range.

\subsection{Where Each Technique Can Fail (and Why That Is a Feature)}
\label{app:impl-failure-modes}

Because all techniques share the best-of-sequence guard, their \emph{final} quality can never drop below the best individual call. But their \emph{per-iteration} or \emph{per-combining-step} dynamics differ sharply across the operating regime, and these differences are what the framework predicts:
\begin{itemize}[nosep,leftmargin=*]
  \item \emph{Diversity} (SC/MRC/EGC) degrades only if branches are strongly correlated (shared training data), reducing the effective diversity order. The MRC~$\geq$~EGC prediction assumes near-perfect CSI; in our data (\Cref{sec:exp-diversity}, \Cref{app:mrc-proof}) CSI is noisy enough on occasions that EGC $\geq$ MRC reverses.
  \item \emph{HARQ-CC} fails when all attempts share the same failure mode; synthesis cannot correct errors that every attempt makes.
  \item \emph{HARQ-IR / Turbo} can degrade the unguarded iterate when the generator is weak or the initial answer is already good: at 3B, $26\%$ of turbo trajectories end below their start (\Cref{fig:threshold-3panel}). The $\alpha$ damping, the divergence break, the stopping threshold, and the best-of-sequence guard contain the damage to the delivered answer but do not create gain.
  \item \emph{Fountain} can over-generate on adversarial tasks where the confidence metric is miscalibrated (many high-quality-looking but wrong samples), reaching $N_{\max}$ without stopping; the synthesized output is still best-of-sequence-guarded, so cost is wasted but quality is not harmed.
  \item \emph{FEC} fails when the parity sections themselves share the main answer's errors (the critic/verifier is misled by the main answer's framing). Separate LLM calls with different instructions (verification vs.\ independent re-derivation) mitigate but do not eliminate this.
  \item \emph{ACM} fails when the difficulty estimator is systematically miscalibrated, routing hard tasks to MCS-0 or easy tasks to MCS-4. Soft ACM replaces the meta-judgment with a direct logprob measurement, reducing this risk at the cost of requiring logprob support.
\end{itemize}

% ===========================================================================
% Per-configuration figure gallery
% ===========================================================================
\section{Per-configuration figure gallery}
\label{app:figure-gallery}

The same plotting procedure is applied to every channel configuration of \Cref{tab:model-configs}, producing a per-configuration set of plots from the single-call baseline and the per-technique results. We reproduce the five most informative plots per configuration here. Throughout this gallery, ``$\rho$'' denotes per-task cost overhead relative to the uncoded baseline (\Cref{eq:cost-overhead}); all error bars are 95\% bootstrap confidence intervals; all significance stars are paired Wilcoxon signed-rank vs.\ baseline ($^{*}p{<}0.05$, $^{**}p{<}0.01$, $^{***}p{<}10^{-3}$). Communication-theoretic technique names referenced in legends are: \textsc{Diversity-SC}, \textsc{Diversity-MRC}, \textsc{Diversity-EGC} (selection / maximal-ratio / equal-gain combining over a multi-channel pool); \textsc{HARQ-CC} and \textsc{HARQ-IR} (hybrid retransmission with chase combining vs.\ incremental redundancy); \textsc{Turbo} (extrinsic-information-exchange decoder); \textsc{Fountain} (rateless sampling with confidence-based stopping); \textsc{FEC}-$r$ (forward error correction at code rate $r$); \textsc{ACM} (the hand-coded adaptive-coding-and-modulation router); \textsc{Diversity-SC-$N$} and \textsc{Diversity-MRC-Discrete-$N$} (wider-pool $N$-sample variants of the same selection / discrete-MRC operators).

\subsection{14B local: DeepSeek-R1 14B + Phi-3 14B, judge Gemma-3 12B}
\label{app:gallery-14b}

This subsection collects the 14B-local diagnostics. \Cref{fig:gallery-14b-qvc} plots quality versus cost on three axes (raw cost, normalized overhead $\rho$, and quality gain over baseline). \Cref{fig:gallery-14b-iter} shows the iterative-refinement diagnostics, with the HARQ-CC-versus-HARQ-IR convergence in \Cref{fig:gallery-14b-harq} and the per-iteration turbo trajectory in \Cref{fig:gallery-14b-turbo}. \Cref{fig:gallery-14b-acm} compares the hand-coded router with the per-task oracle on this configuration, and \Cref{fig:gallery-14b-cat} reports per-category quality on the four curated categories.

\begin{figure}[h]
  \centering
  \includegraphics[width=\linewidth]{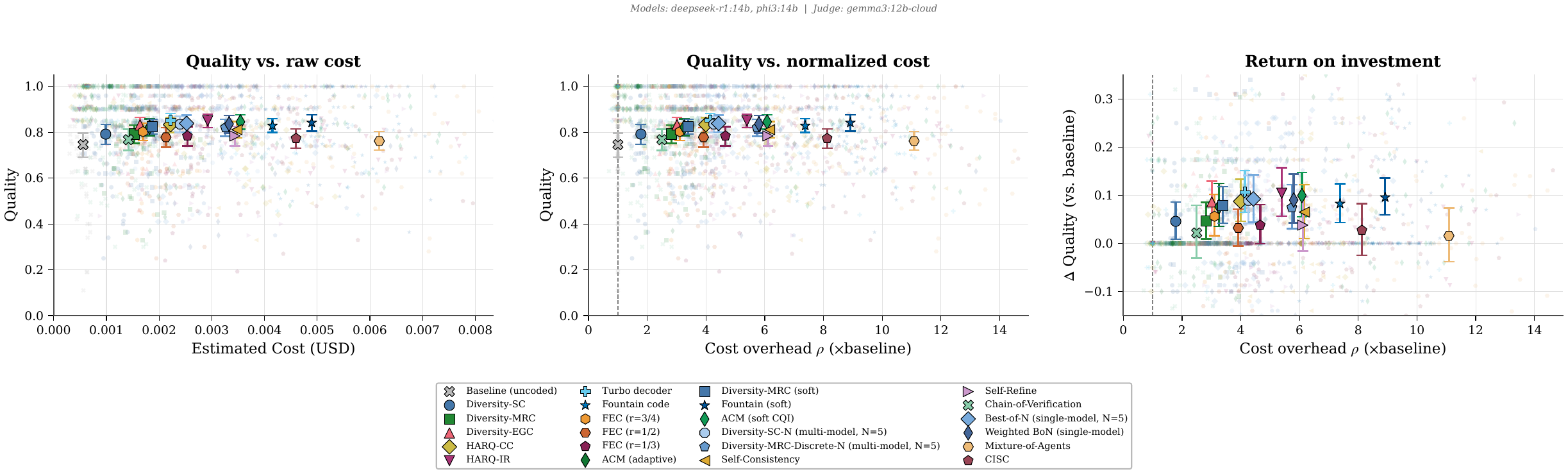}
  \caption{14B local: quality versus cost. \emph{Left}: raw cost (US dollars) per task; \emph{center}: quality versus normalized cost overhead $\rho$; \emph{right}: quality gain over the single-call baseline versus $\rho$. Large markers are per-technique means with 95\% bootstrap confidence intervals; small markers are individual task runs. Techniques in the upper-left of the right panel give the best quality return per cost-dollar invested.}
  \label{fig:gallery-14b-qvc}
\end{figure}

\begin{figure}[h]
  \centering
  \begin{subfigure}[t]{0.49\linewidth}
    \includegraphics[width=\linewidth]{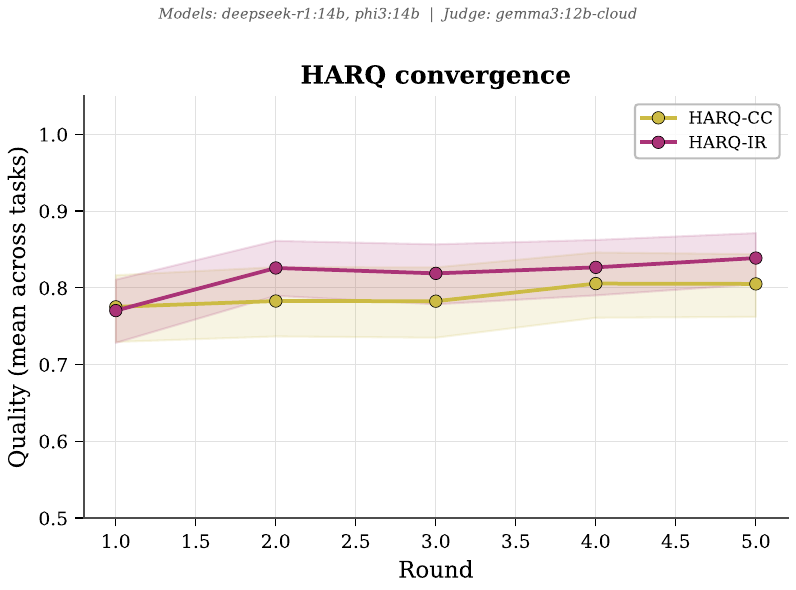}
    \caption{HARQ convergence.}
    \label{fig:gallery-14b-harq}
  \end{subfigure}\hfill
  \begin{subfigure}[t]{0.49\linewidth}
    \includegraphics[width=\linewidth]{figures/goods/14B_deepseek_phi3_gemma3/turbo_waterfall.pdf}
    \caption{Turbo iteration trajectory.}
    \label{fig:gallery-14b-turbo}
  \end{subfigure}
  \caption{14B local: iterative-refinement diagnostics. (\subref*{fig:gallery-14b-harq}) Mean quality versus retransmission round for chase combining (\textsc{HARQ-CC}) and incremental redundancy (\textsc{HARQ-IR}); incremental redundancy reaches its plateau in fewer rounds, consistent with the ``new parity per round'' interpretation. (\subref*{fig:gallery-14b-turbo}) Per-task turbo trajectories (thin lines) plus the all-tasks running-max mean (thick); iteration~0 is the initial generation before any refinement. The running-max curve is monotonic by the best-of-sequence guard (\Cref{rem:bos-guard}).}
  \label{fig:gallery-14b-iter}
\end{figure}

\begin{figure}[h]
  \centering
  \includegraphics[width=\linewidth]{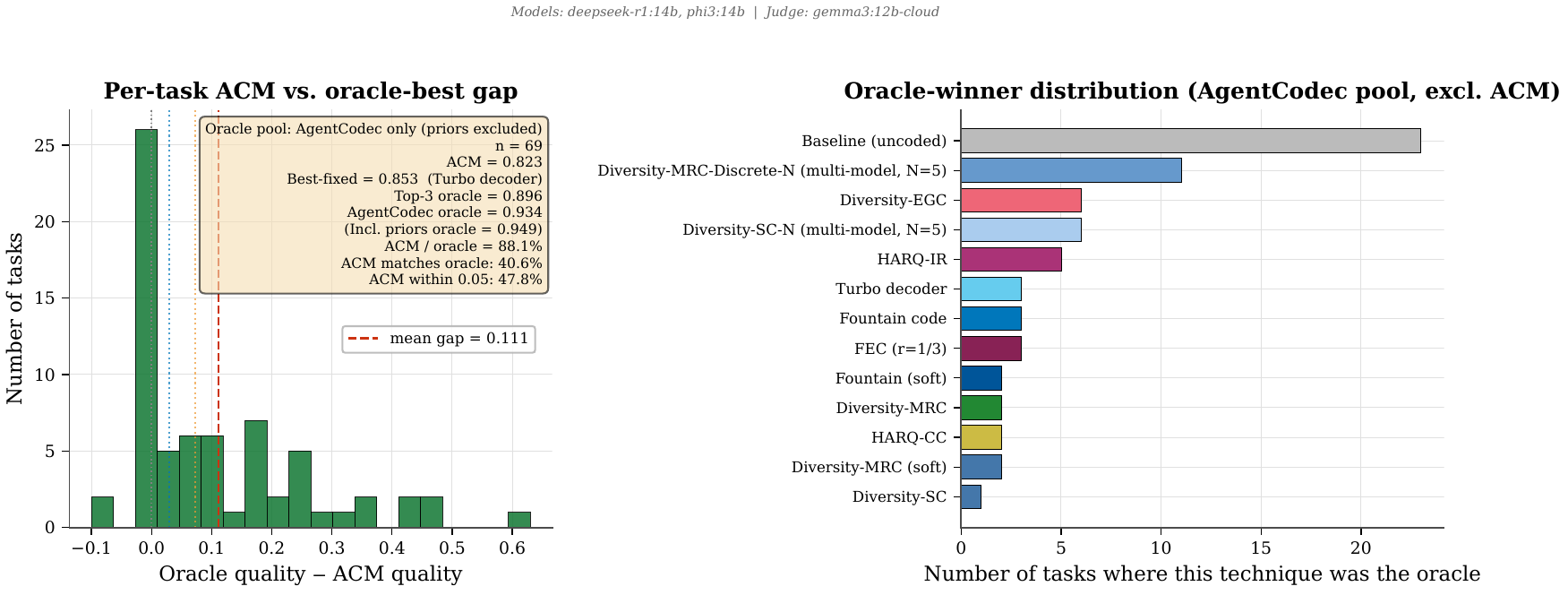}
  \caption{14B local: hand-coded \textsc{ACM} router versus the per-task oracle. \emph{Left}: histogram over tasks of the oracle's quality minus the hand-coded router's (oracle taken over the \agentcodec{} techniques, prior methods excluded), with summary statistics in the inset. \emph{Right}: number of tasks on which each technique attains the per-task oracle.}
  \label{fig:gallery-14b-acm}
\end{figure}

\begin{figure}[h]
  \centering
  \includegraphics[width=0.85\linewidth]{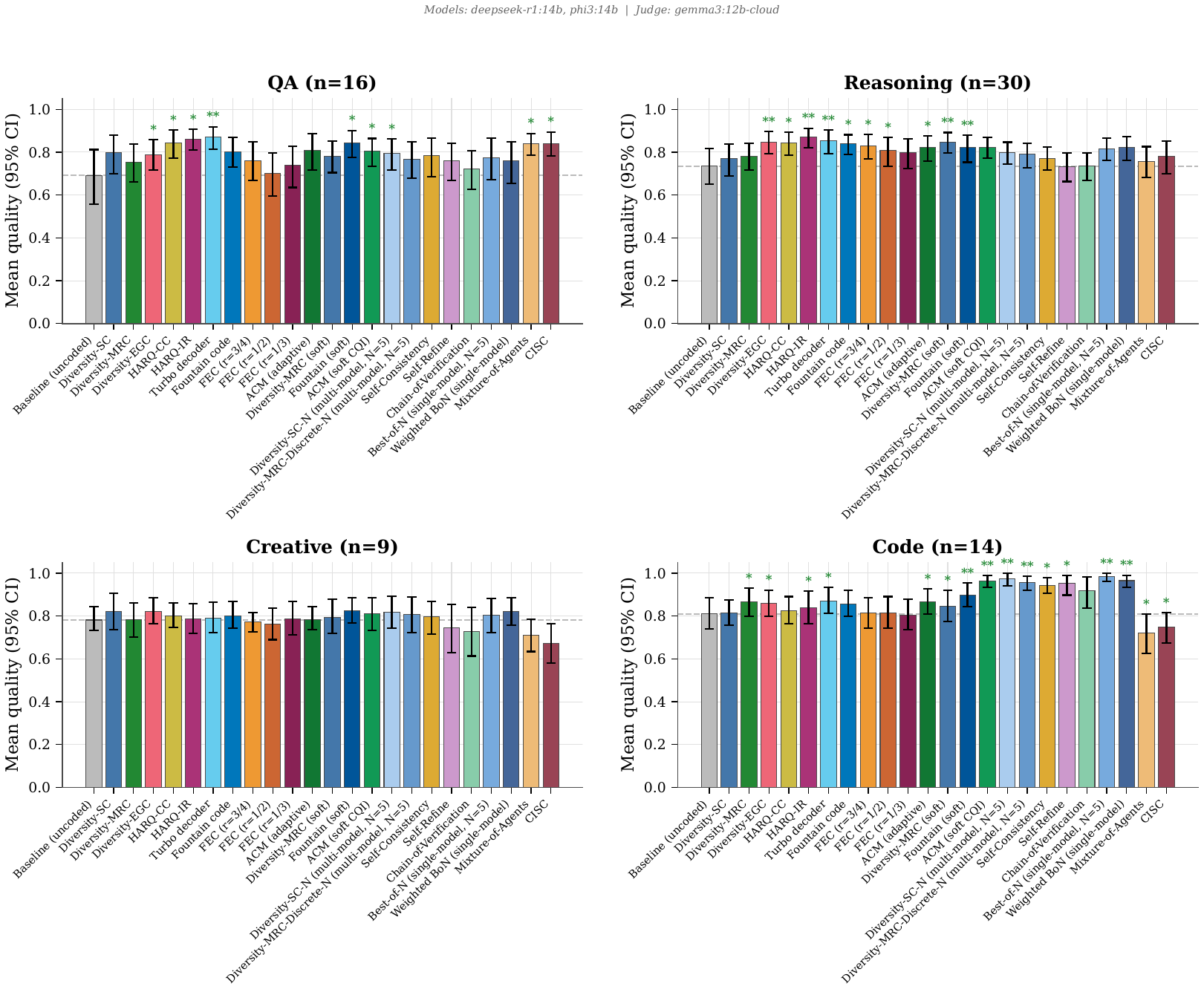}
  \caption{14B local: per-category mean quality on the 69 curated tasks; one panel per task category (question answering, reasoning, creative, code). Dashed line is the per-category baseline mean; stars mark a two-sided paired Wilcoxon difference from the baseline in either direction ($^{*}p{<}0.05$, $^{**}p{<}0.01$, $^{***}p{<}0.001$). Iterative techniques (\textsc{HARQ-IR}, \textsc{Turbo}) lead on question answering and reasoning; selection and weighted voting over five samples (Best-of-$N$, Diversity-SC-$N$, Weighted Best-of-$N$) lead on code; on creative, no technique differs significantly from the baseline.}
  \label{fig:gallery-14b-cat}
\end{figure}

\subsection{8B local: Llama-3.1 8B + Qwen-2.5 7B, judge Gemma-3 12B}
\label{app:gallery-8b}

This subsection mirrors the layout of \Cref{app:gallery-14b} for the 8B local configuration. The quality-cost panels are in \Cref{fig:gallery-8b-qvc}; the iterative-refinement diagnostics in \Cref{fig:gallery-8b-iter}; the comparison of the hand-coded router with the oracle in \Cref{fig:gallery-8b-acm}; and the per-category quality in \Cref{fig:gallery-8b-cat}.

\begin{figure}[h]
  \centering
  \includegraphics[width=\linewidth]{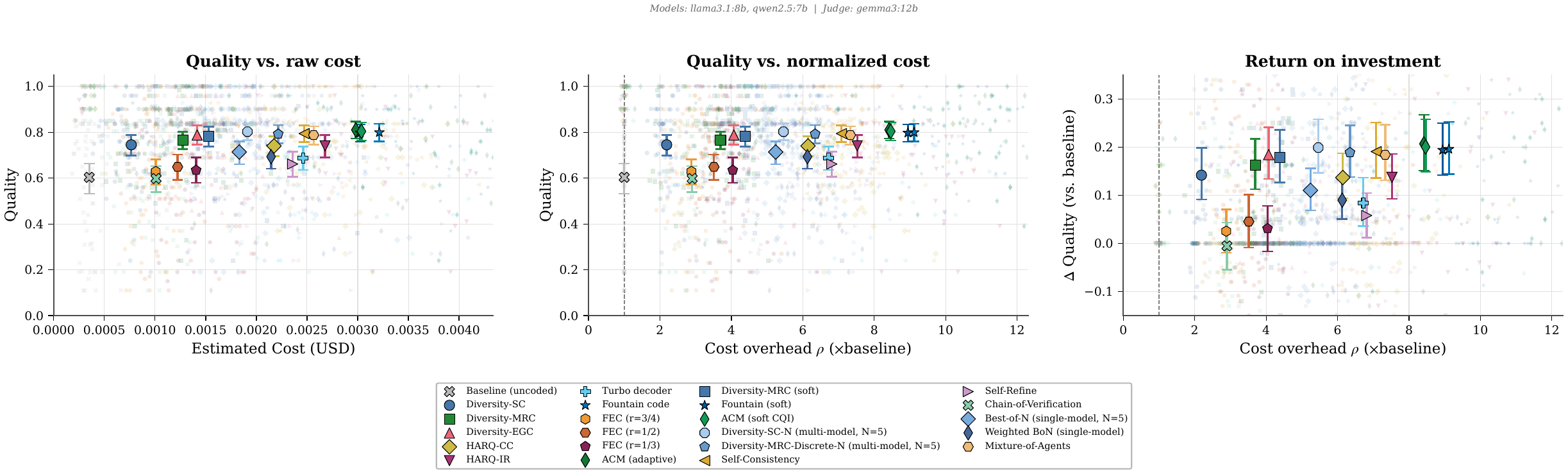}
  \caption{8B local: quality versus cost (panels as in \Cref{fig:gallery-14b-qvc}). The frontier shifts down in absolute quality compared with 14B but the relative ordering of comm-theoretic techniques persists.}
  \label{fig:gallery-8b-qvc}
\end{figure}

\begin{figure}[h]
  \centering
  \begin{subfigure}[t]{0.49\linewidth}
    \includegraphics[width=\linewidth]{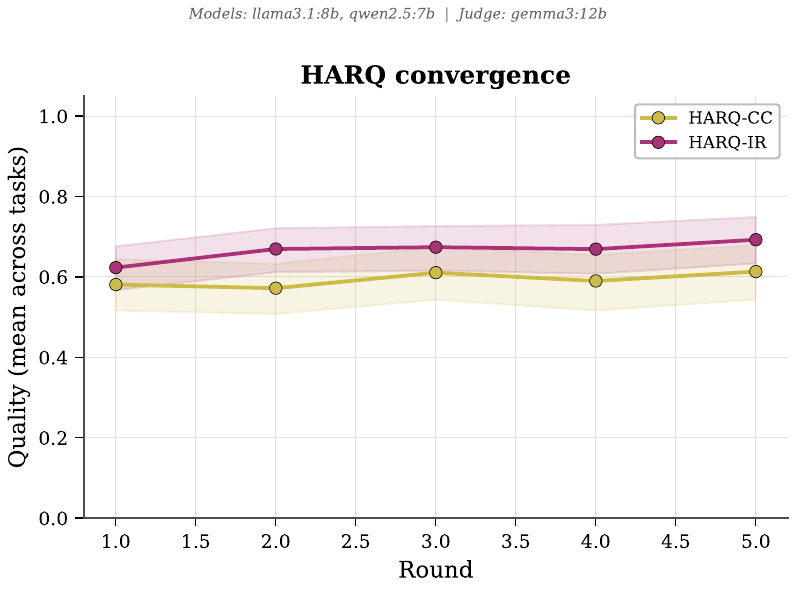}
    \caption{HARQ convergence.}
  \end{subfigure}\hfill
  \begin{subfigure}[t]{0.49\linewidth}
    \includegraphics[width=\linewidth]{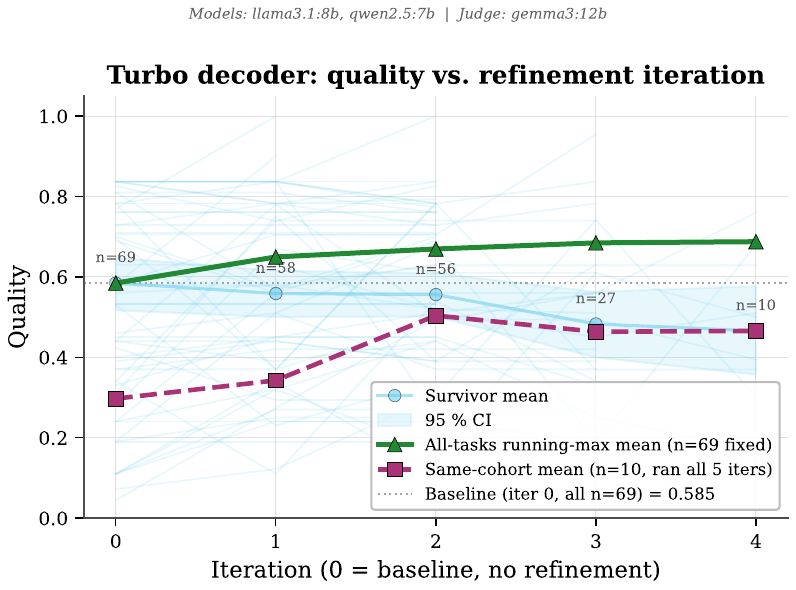}
    \caption{Turbo iteration trajectory.}
  \end{subfigure}
  \caption{8B local: iterative-refinement diagnostics. The running-max curve climbs because the best-of-sequence guard keeps the delivered output monotone. The mean over tasks still being refined changes composition as tasks exit, so it is not a test of \Cref{prop:refinement-threshold}; per-trajectory statistics are in \Cref{fig:threshold-3panel}.}
  \label{fig:gallery-8b-iter}
\end{figure}

\begin{figure}[h]
  \centering
  \includegraphics[width=\linewidth]{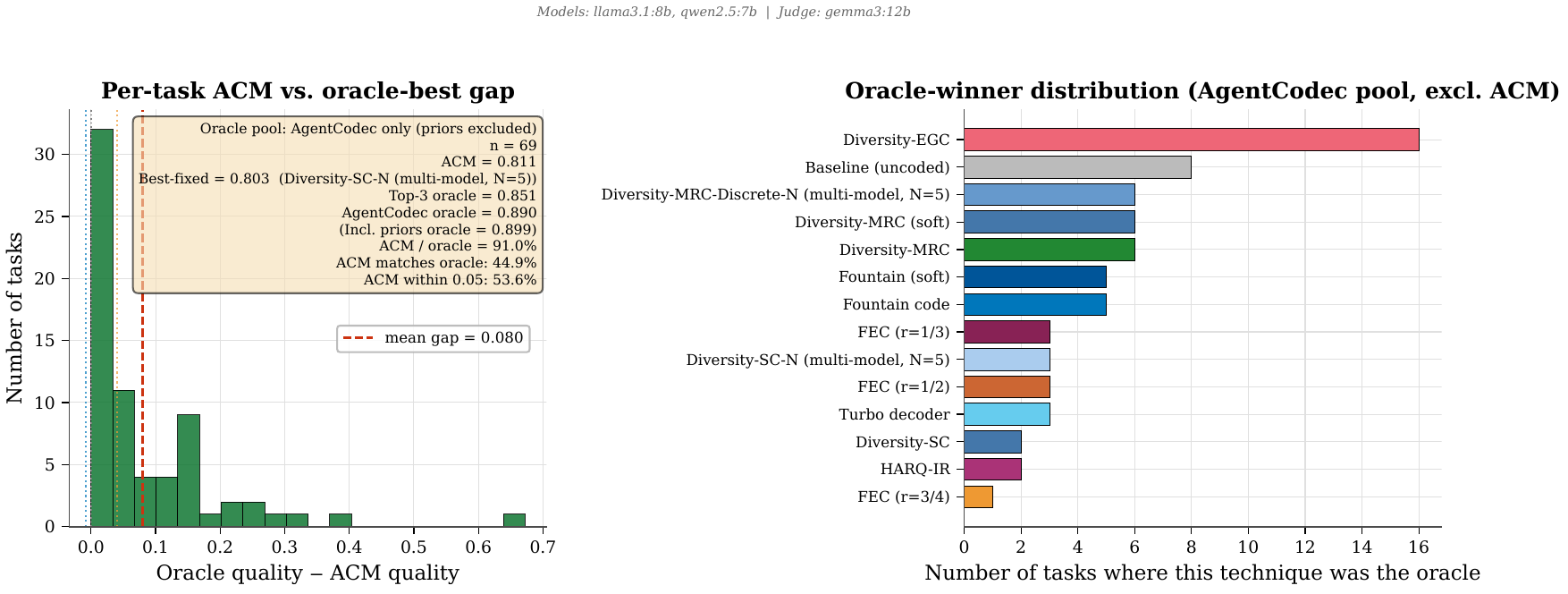}
  \caption{8B local: hand-coded \textsc{ACM} router versus the per-task oracle (panels as in \Cref{fig:gallery-14b-acm}).}
  \label{fig:gallery-8b-acm}
\end{figure}

\begin{figure}[h]
  \centering
  \includegraphics[width=0.85\linewidth]{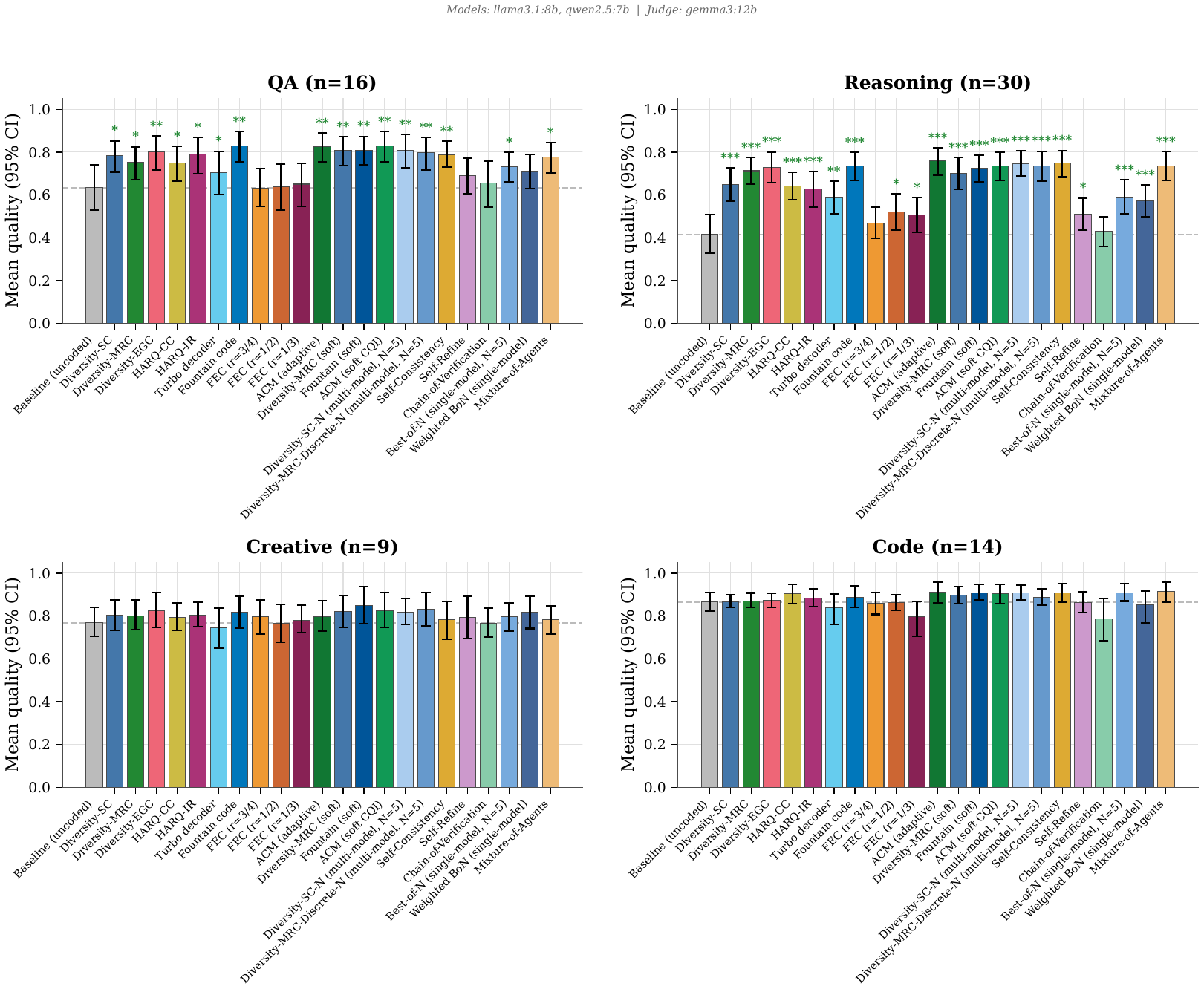}
  \caption{8B local: per-category mean quality (panels as in \Cref{fig:gallery-14b-cat}).}
  \label{fig:gallery-8b-cat}
\end{figure}

\subsection{3B local: Qwen-2.5 3B + Gemma-3 4B, judge Gemma-3 12B}
\label{app:gallery-3b}

This subsection mirrors the layout of \Cref{app:gallery-14b}. \Cref{fig:gallery-3b-qvc} plots quality versus cost; \Cref{fig:gallery-3b-iter} the iterative-refinement diagnostics; \Cref{fig:gallery-3b-acm} the hand-coded router against the oracle; and \Cref{fig:gallery-3b-cat} the per-category quality. The figures pool the three generation runs: each technique's runs are averaged per task in the quality, router, and category panels, as in \Cref{tab:results-summary}, and the refinement diagnostics show the trajectories of all three runs.

\begin{figure}[h]
  \centering
  \includegraphics[width=\linewidth]{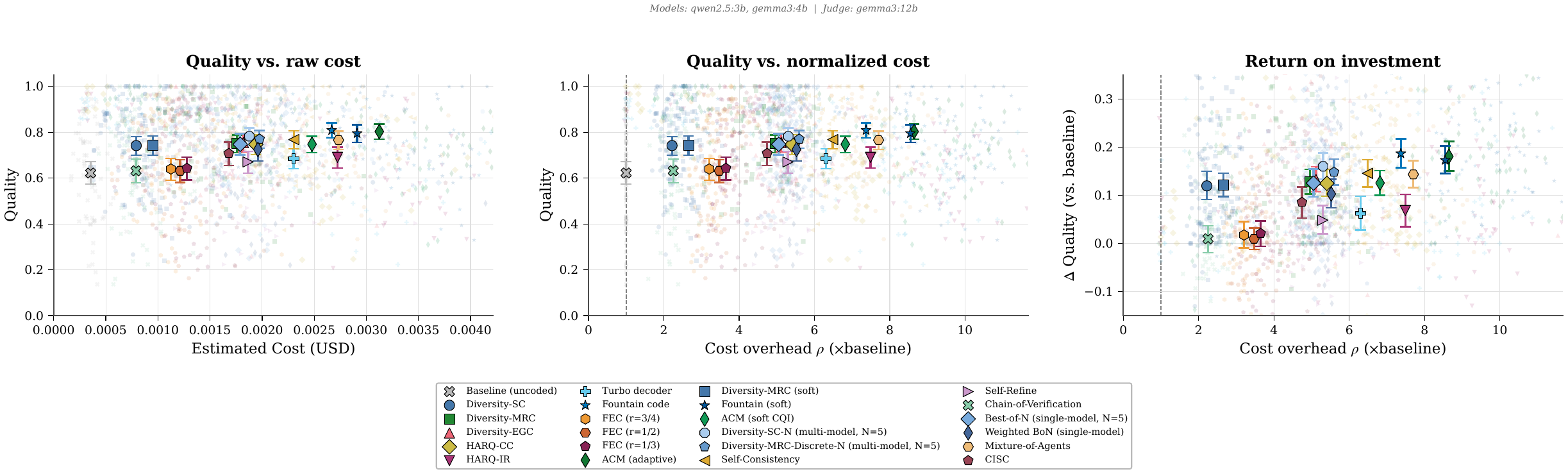}
  \caption{3B local: quality versus cost (panels as in \Cref{fig:gallery-14b-qvc}).}
  \label{fig:gallery-3b-qvc}
\end{figure}

\begin{figure}[h]
  \centering
  \begin{subfigure}[t]{0.49\linewidth}
    \includegraphics[width=\linewidth]{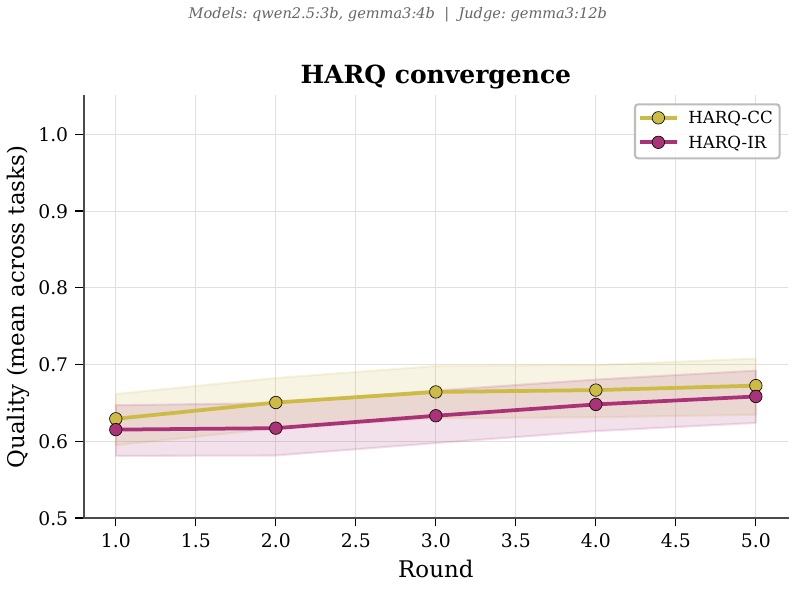}
    \caption{HARQ convergence.}
  \end{subfigure}\hfill
  \begin{subfigure}[t]{0.49\linewidth}
    \includegraphics[width=\linewidth]{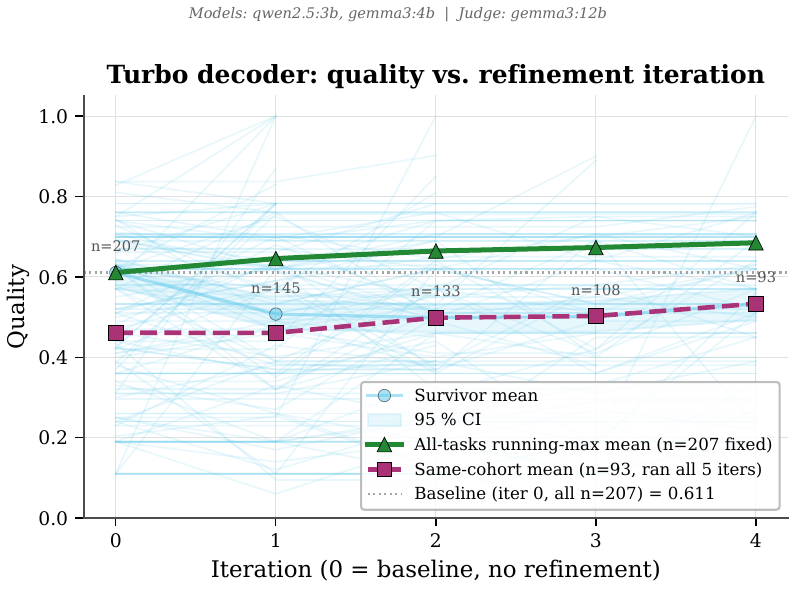}
    \caption{Turbo iteration trajectory.}
  \end{subfigure}
  \caption{3B local: iterative-refinement diagnostics. The best-of-sequence guard keeps the delivered output monotone. The mean over tasks still being refined changes composition as tasks exit, so it is not a test of \Cref{prop:refinement-threshold}; per trajectory and over the three runs, $26\%$ of 3B turbo refinements end below their start (\Cref{fig:threshold-3panel}).}
  \label{fig:gallery-3b-iter}
\end{figure}

\begin{figure}[h]
  \centering
  \includegraphics[width=\linewidth]{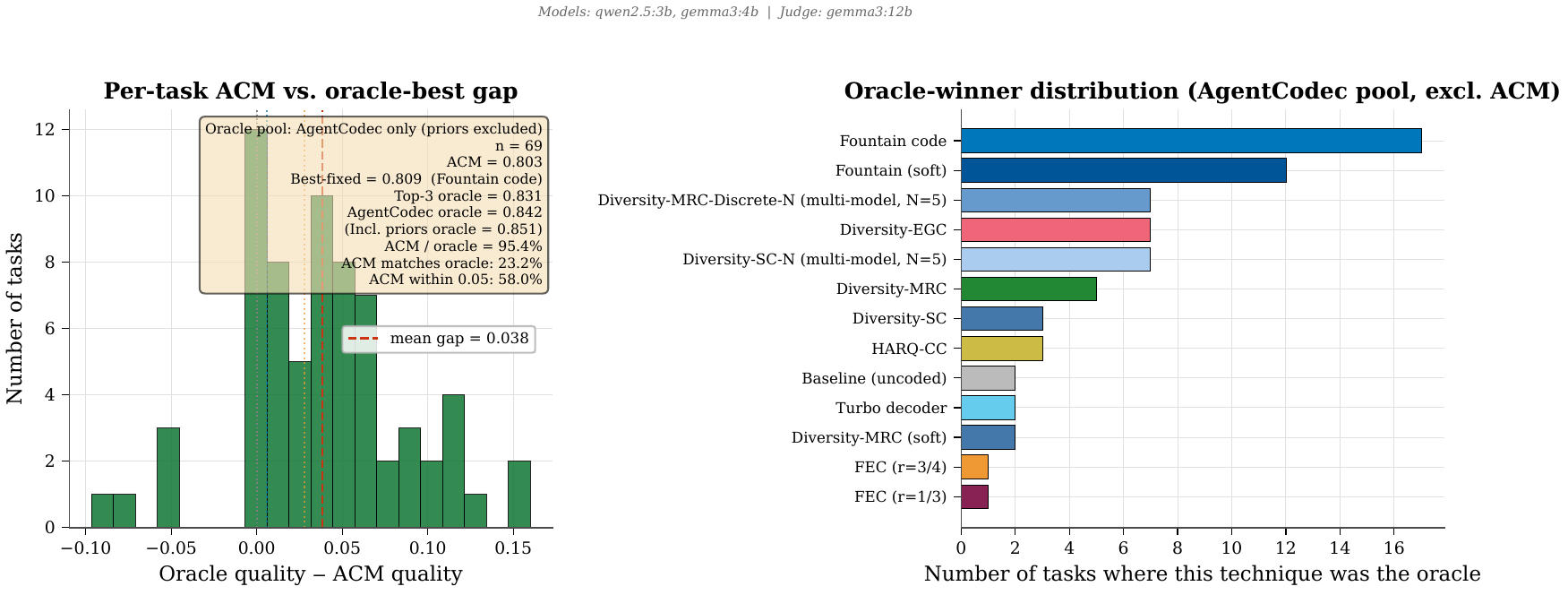}
  \caption{3B local: hand-coded \textsc{ACM} router versus the per-task oracle (panels as in \Cref{fig:gallery-14b-acm}).}
  \label{fig:gallery-3b-acm}
\end{figure}

\begin{figure}[h]
  \centering
  \includegraphics[width=0.85\linewidth]{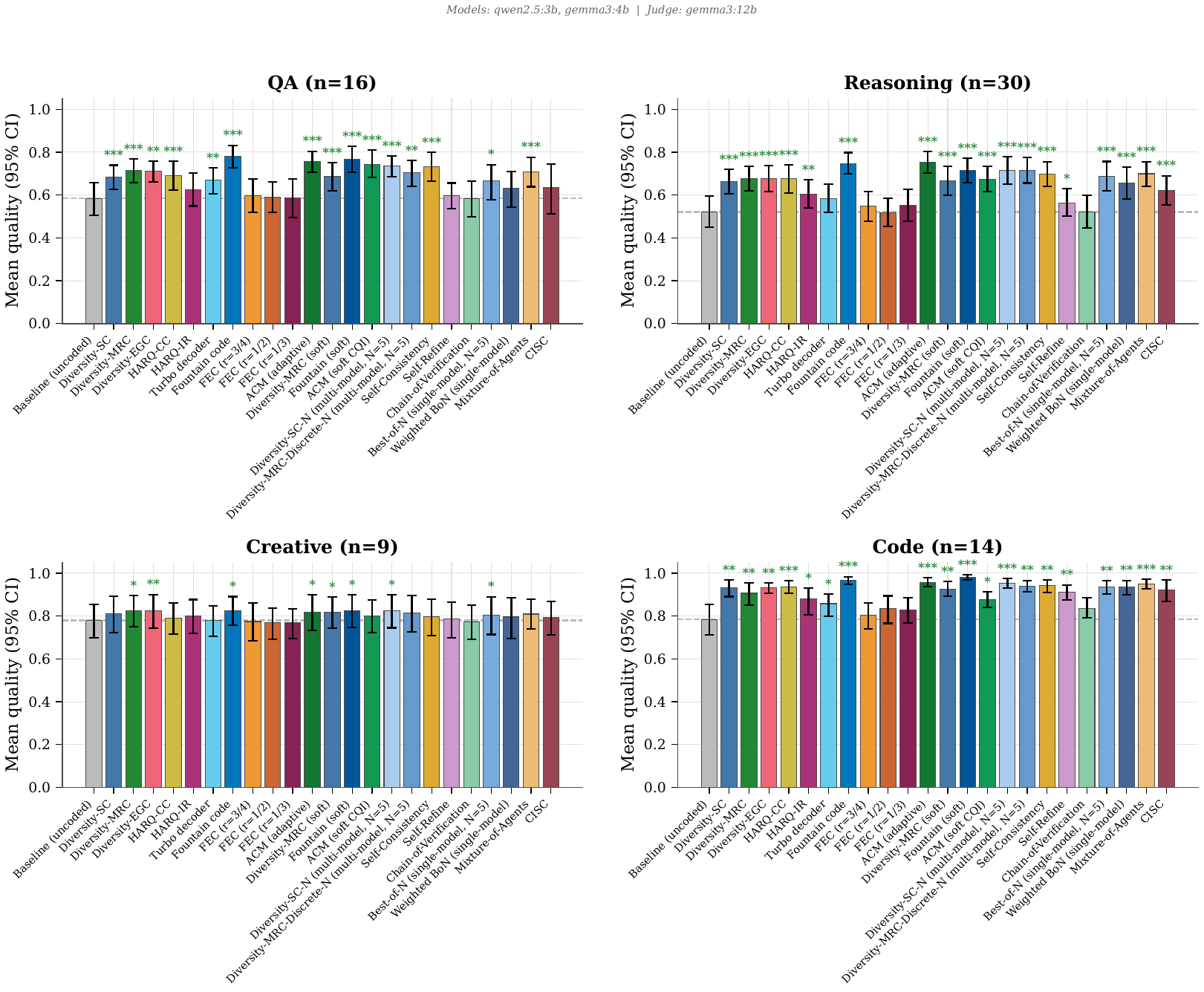}
  \caption{3B local: per-category mean quality (panels as in \Cref{fig:gallery-14b-cat}).}
  \label{fig:gallery-3b-cat}
\end{figure}

\subsection{3B local with cloud judge: Qwen-2.5 3B + Llama-3.2 3B, judge Gemma-4 31B}
\label{app:gallery-3b-cloud}

This subsection follows the same layout. \Cref{fig:gallery-3bcloud-qvc} plots quality versus cost; \Cref{fig:gallery-3bcloud-iter} the iterative-refinement diagnostics; and \Cref{fig:gallery-3bcloud-acm} the hand-coded router against the oracle.

\begin{figure}[h]
  \centering
  \includegraphics[width=\linewidth]{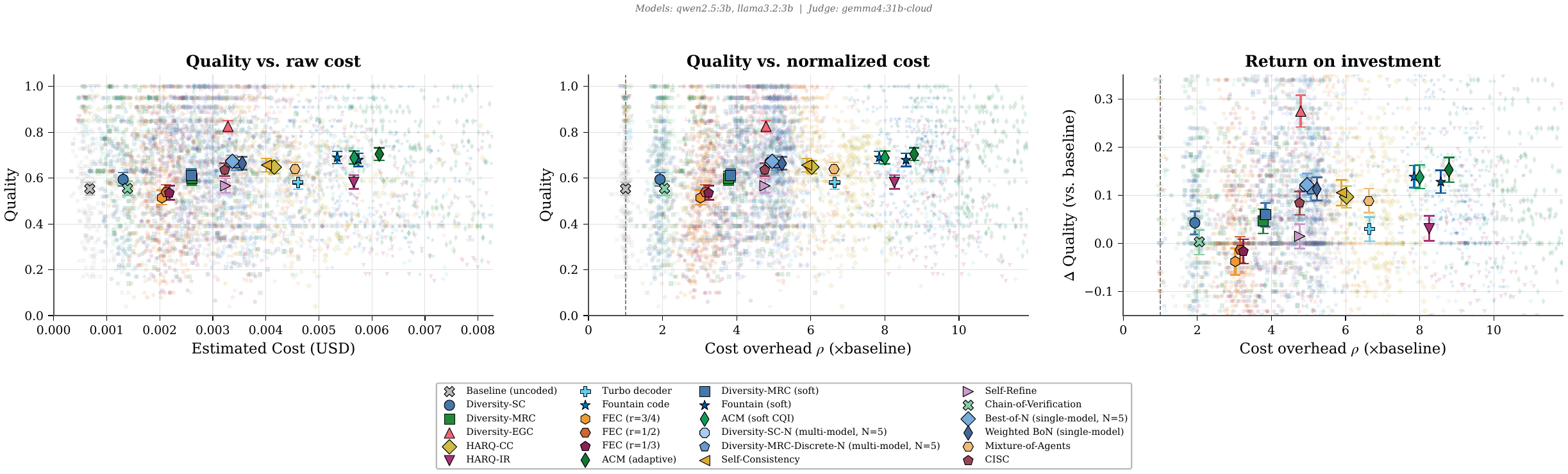}
  \caption{3B local with cloud judge: quality versus cost (panels as in \Cref{fig:gallery-14b-qvc}). Replacing the 12B local judge with a 31B cloud judge resolves judge-side score quantization at the price of an extra cost component on every scoring call; the technique ordering is qualitatively preserved. Diversity-EGC's position here reflects its 31B synthesizer answering tasks itself (\Cref{app:synthesis-integrity}).}
  \label{fig:gallery-3bcloud-qvc}
\end{figure}

\begin{figure}[h]
  \centering
  \begin{subfigure}[t]{0.49\linewidth}
    \includegraphics[width=\linewidth]{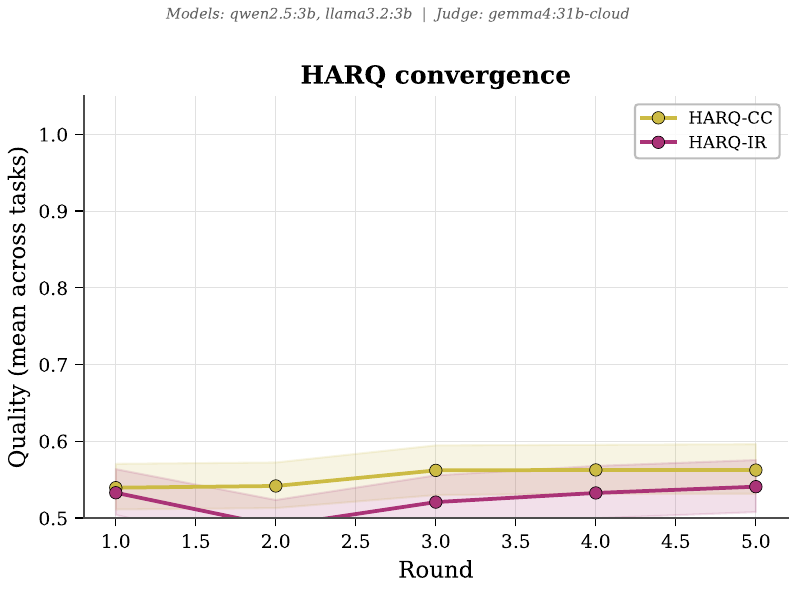}
    \caption{HARQ convergence.}
  \end{subfigure}\hfill
  \begin{subfigure}[t]{0.49\linewidth}
    \includegraphics[width=\linewidth]{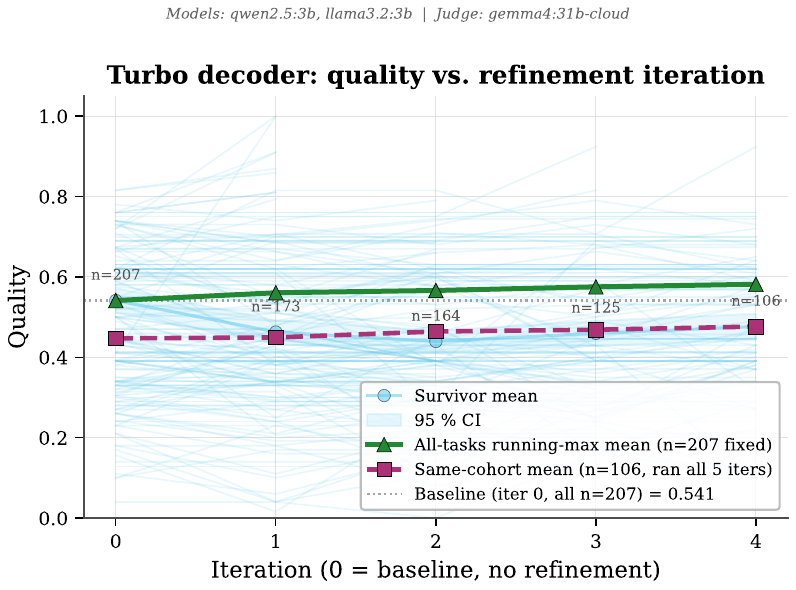}
    \caption{Turbo iteration trajectory.}
  \end{subfigure}
  \caption{3B local with cloud judge: iterative-refinement diagnostics.}
  \label{fig:gallery-3bcloud-iter}
\end{figure}

\begin{figure}[h]
  \centering
  \includegraphics[width=\linewidth]{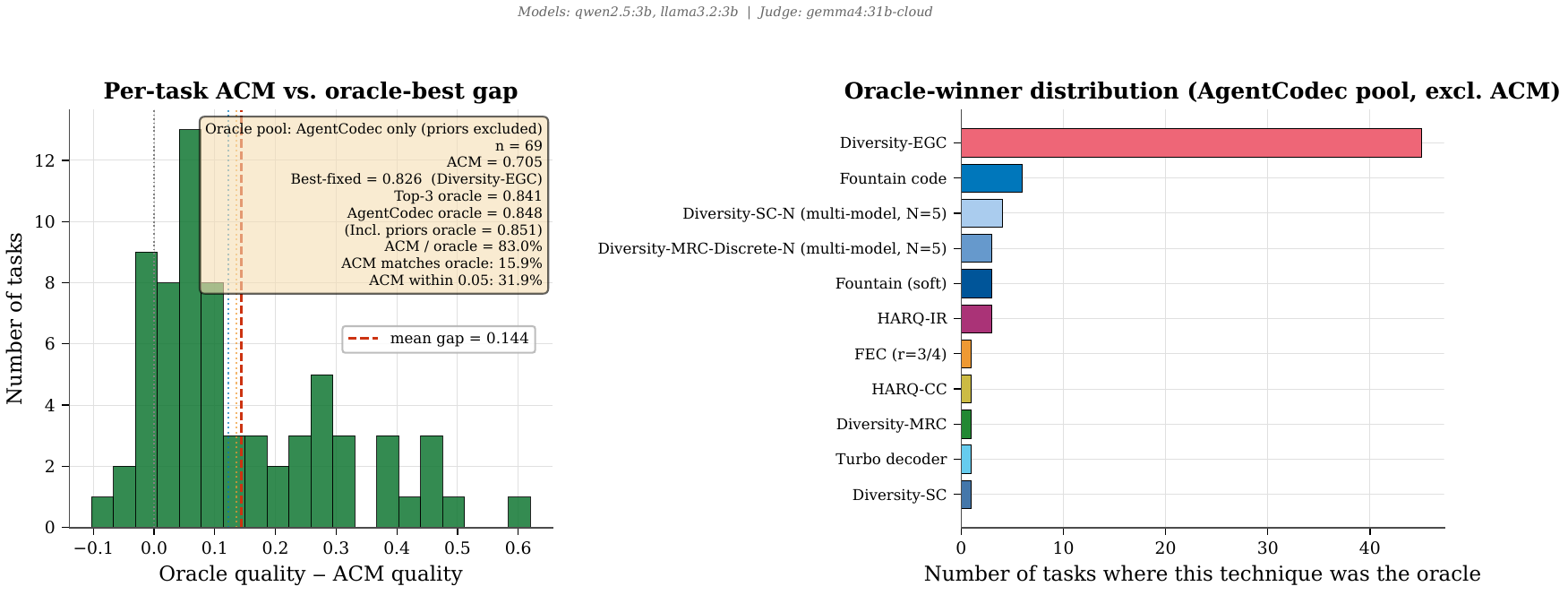}
  \caption{3B local with cloud judge: hand-coded \textsc{ACM} router versus the per-task oracle (panels as in \Cref{fig:gallery-14b-acm}). Diversity-EGC's frequent oracle wins here reflect its 31B synthesizer answering tasks itself, which is specific to this configuration (\Cref{app:synthesis-integrity}).}
  \label{fig:gallery-3bcloud-acm}
\end{figure}

\subsection{Anthropic + OpenAI cloud: Claude Haiku 4.5 + GPT-5-mini, judge Claude Haiku 4.5}
\label{app:gallery-cloud}

This subsection follows the same layout. \Cref{fig:gallery-cloud-qvc} plots quality versus cost; \Cref{fig:gallery-cloud-iter} the iterative-refinement diagnostics (both the chase-combining and turbo decoders sit above threshold on this configuration); and \Cref{fig:gallery-cloud-cat} the per-category quality.

\begin{figure}[h]
  \centering
  \includegraphics[width=\linewidth]{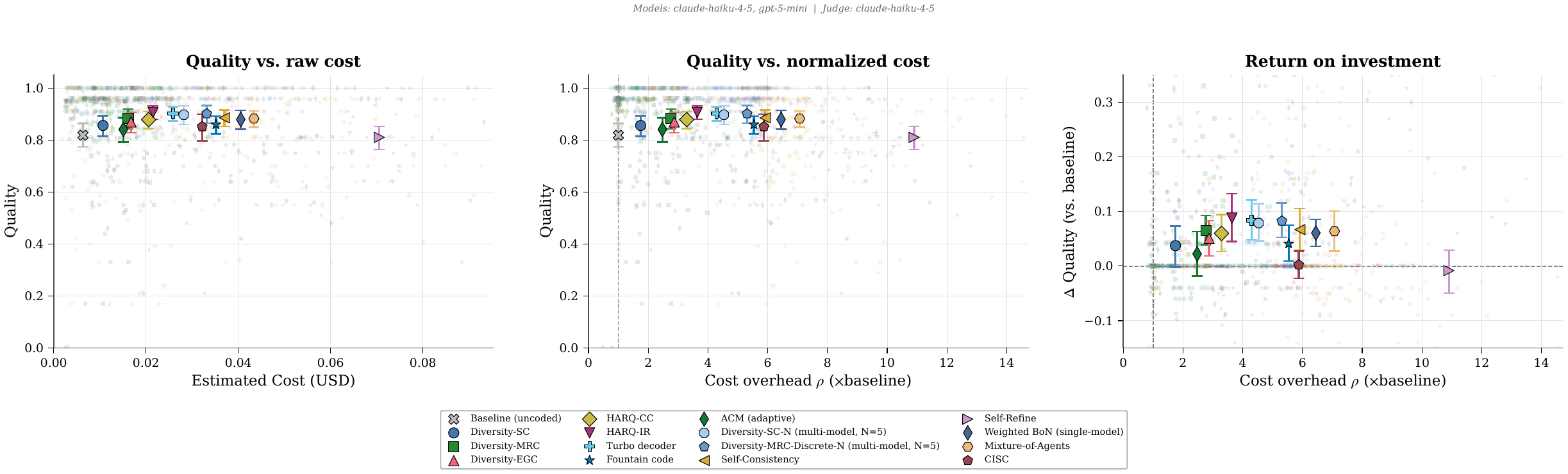}
  \caption{Anthropic + OpenAI cloud: quality versus cost (panels as in \Cref{fig:gallery-14b-qvc}). Absolute cost is one to two orders of magnitude above the local configurations; absolute quality is also higher, so quality gains per added cost-dollar are more conservative.}
  \label{fig:gallery-cloud-qvc}
\end{figure}

\begin{figure}[h]
  \centering
  \begin{subfigure}[t]{0.49\linewidth}
    \includegraphics[width=\linewidth]{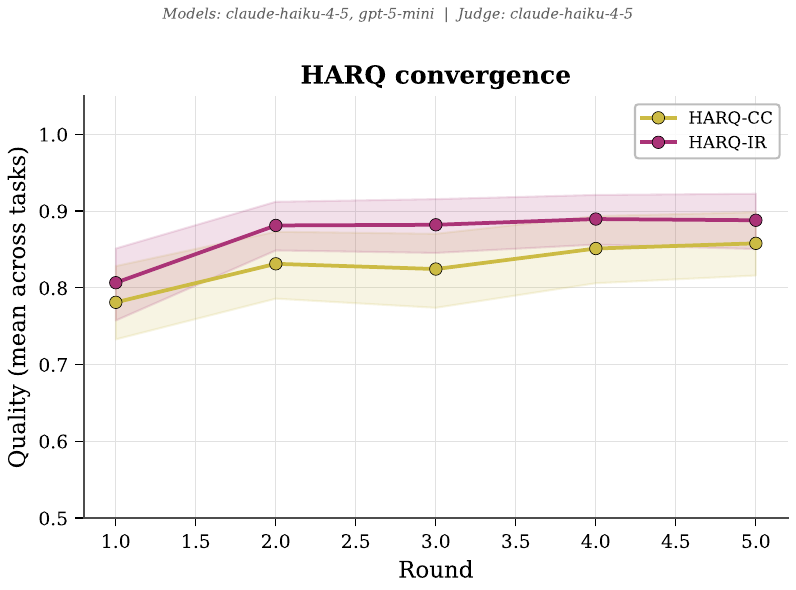}
    \caption{HARQ convergence.}
  \end{subfigure}\hfill
  \begin{subfigure}[t]{0.49\linewidth}
    \includegraphics[width=\linewidth]{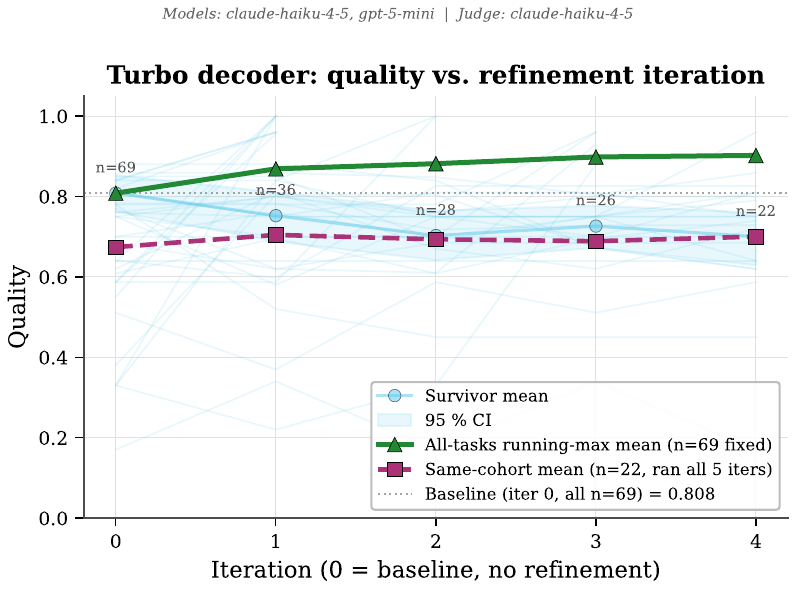}
    \caption{Turbo iteration trajectory.}
  \end{subfigure}
  \caption{Anthropic + OpenAI cloud: iterative-refinement diagnostics. Both refinement decoders sit firmly above the threshold of \Cref{prop:refinement-threshold} on this configuration.}
  \label{fig:gallery-cloud-iter}
\end{figure}

\begin{figure}[h]
  \centering
  \includegraphics[width=0.85\linewidth]{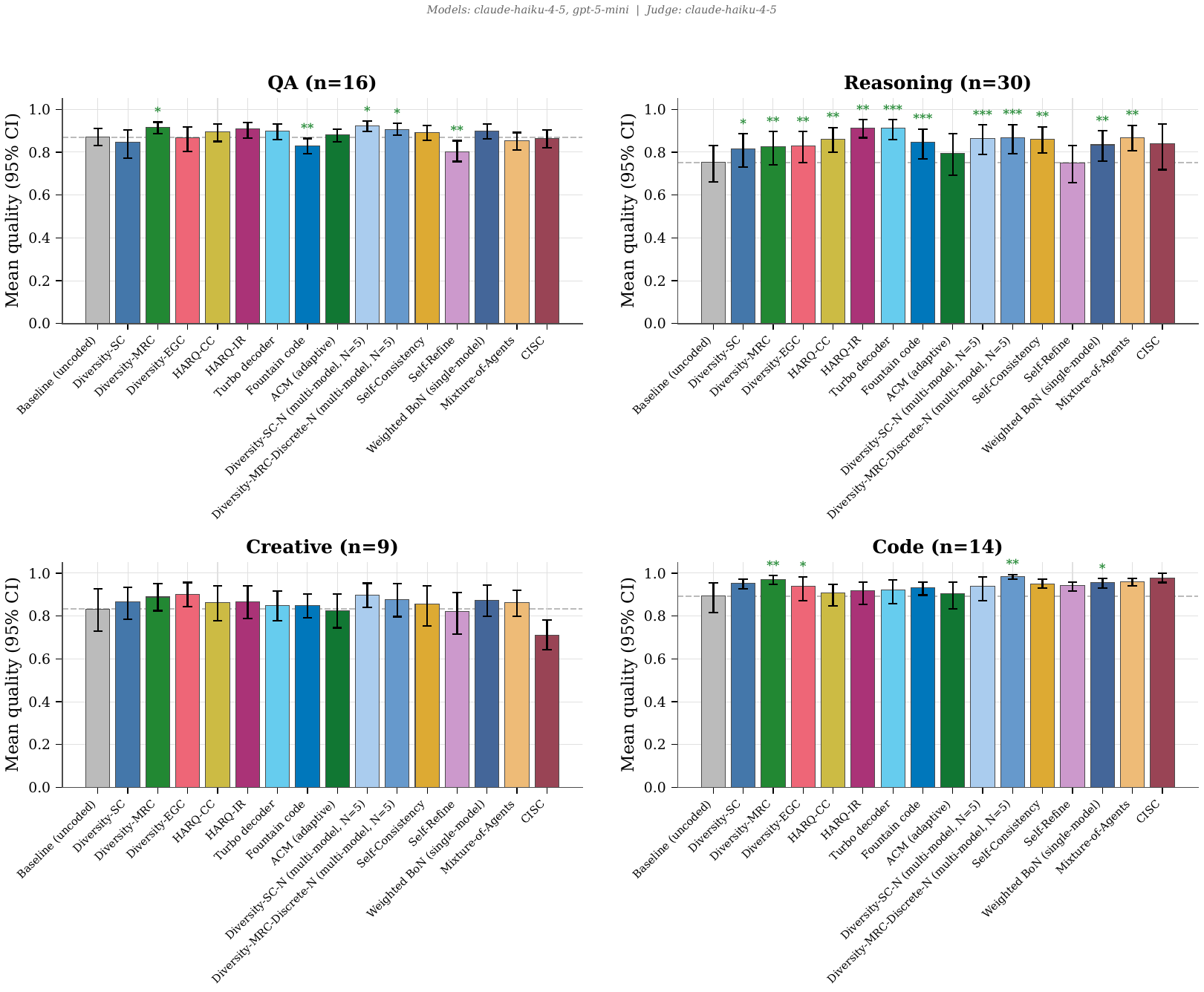}
  \caption{Anthropic + OpenAI cloud: per-category mean quality (panels as in \Cref{fig:gallery-14b-cat}).}
  \label{fig:gallery-cloud-cat}
\end{figure}

\subsection{Ollama-cloud trio on curated tasks: Nemotron-Nano-3 30B + Devstral-Small-2 24B, judge GLM-5.1}
\label{app:gallery-ollama69}

This subsection covers the Ollama-cloud trio on the 69 curated tasks. \Cref{fig:gallery-ollama69-qvc} plots quality versus cost; \Cref{fig:gallery-ollama69-iter} the iterative-refinement diagnostics; \Cref{fig:gallery-ollama69-acm} the hand-coded router against the oracle; and \Cref{fig:gallery-ollama69-cat} the per-category quality.

\begin{figure}[h]
  \centering
  \includegraphics[width=\linewidth]{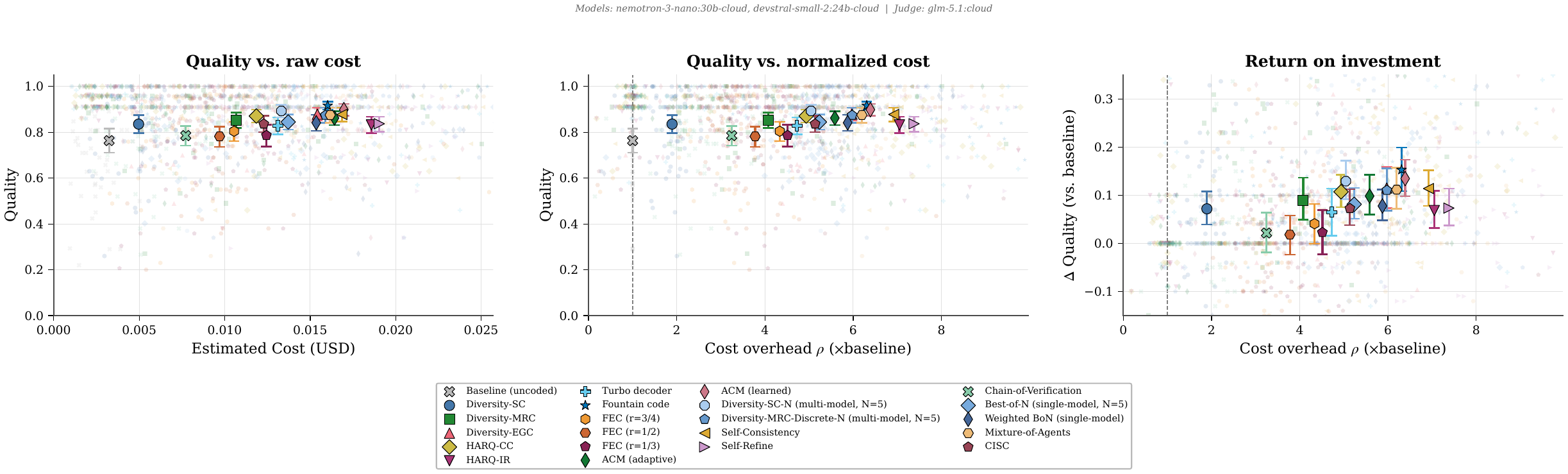}
  \caption{Ollama-cloud trio on the 69 curated tasks: quality versus cost (panels as in \Cref{fig:gallery-14b-qvc}).}
  \label{fig:gallery-ollama69-qvc}
\end{figure}

\begin{figure}[h]
  \centering
  \begin{subfigure}[t]{0.49\linewidth}
    \includegraphics[width=\linewidth]{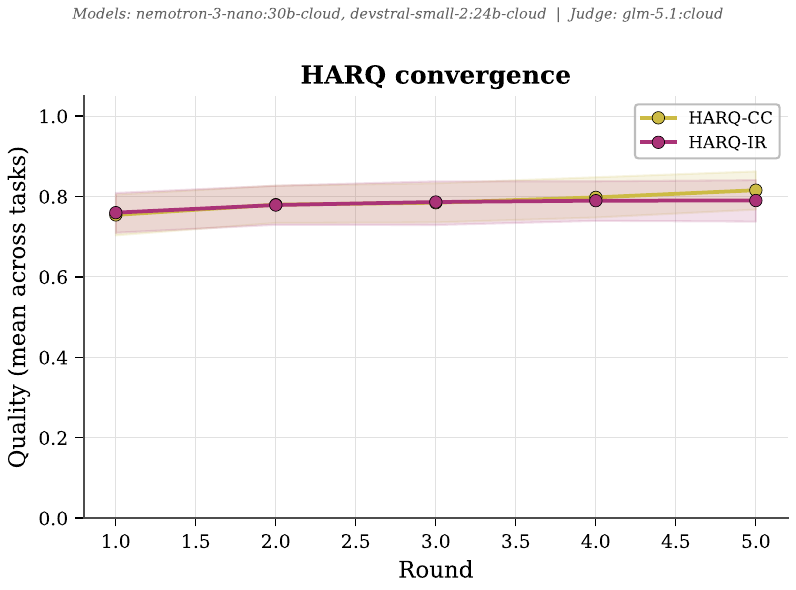}
    \caption{HARQ convergence.}
  \end{subfigure}\hfill
  \begin{subfigure}[t]{0.49\linewidth}
    \includegraphics[width=\linewidth]{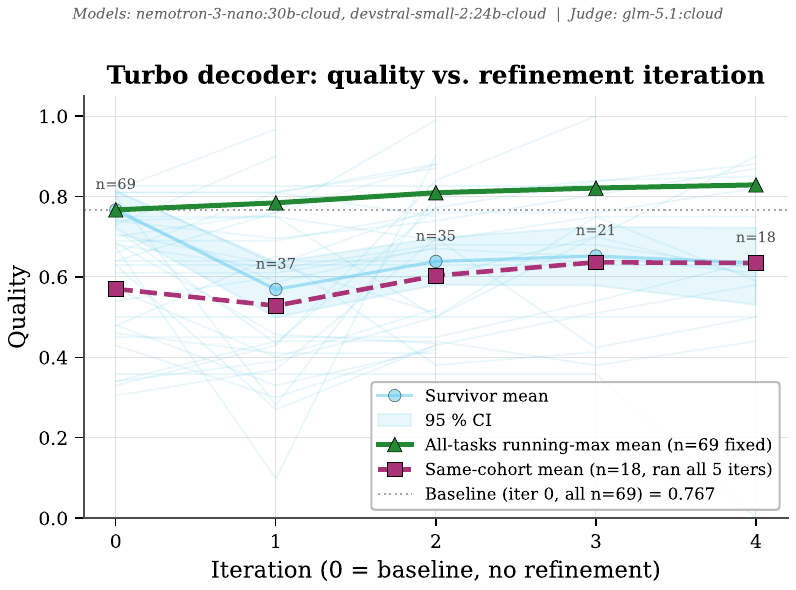}
    \caption{Turbo iteration trajectory.}
  \end{subfigure}
  \caption{Ollama-cloud trio on the curated split: iterative-refinement diagnostics.}
  \label{fig:gallery-ollama69-iter}
\end{figure}

\begin{figure}[h]
  \centering
  \includegraphics[width=\linewidth]{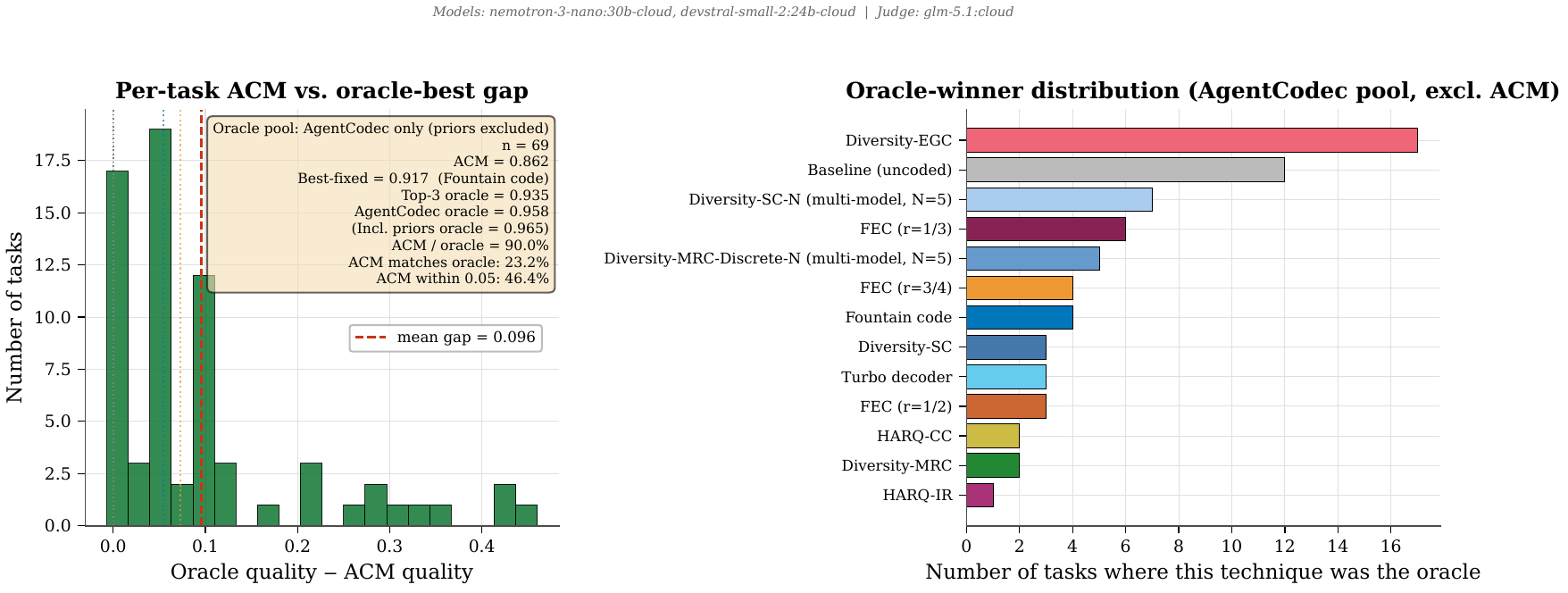}
  \caption{Ollama-cloud trio on the 69 curated tasks: hand-coded \textsc{ACM} router versus the per-task oracle (panels as in \Cref{fig:gallery-14b-acm}).}
  \label{fig:gallery-ollama69-acm}
\end{figure}

\begin{figure}[h]
  \centering
  \includegraphics[width=0.85\linewidth]{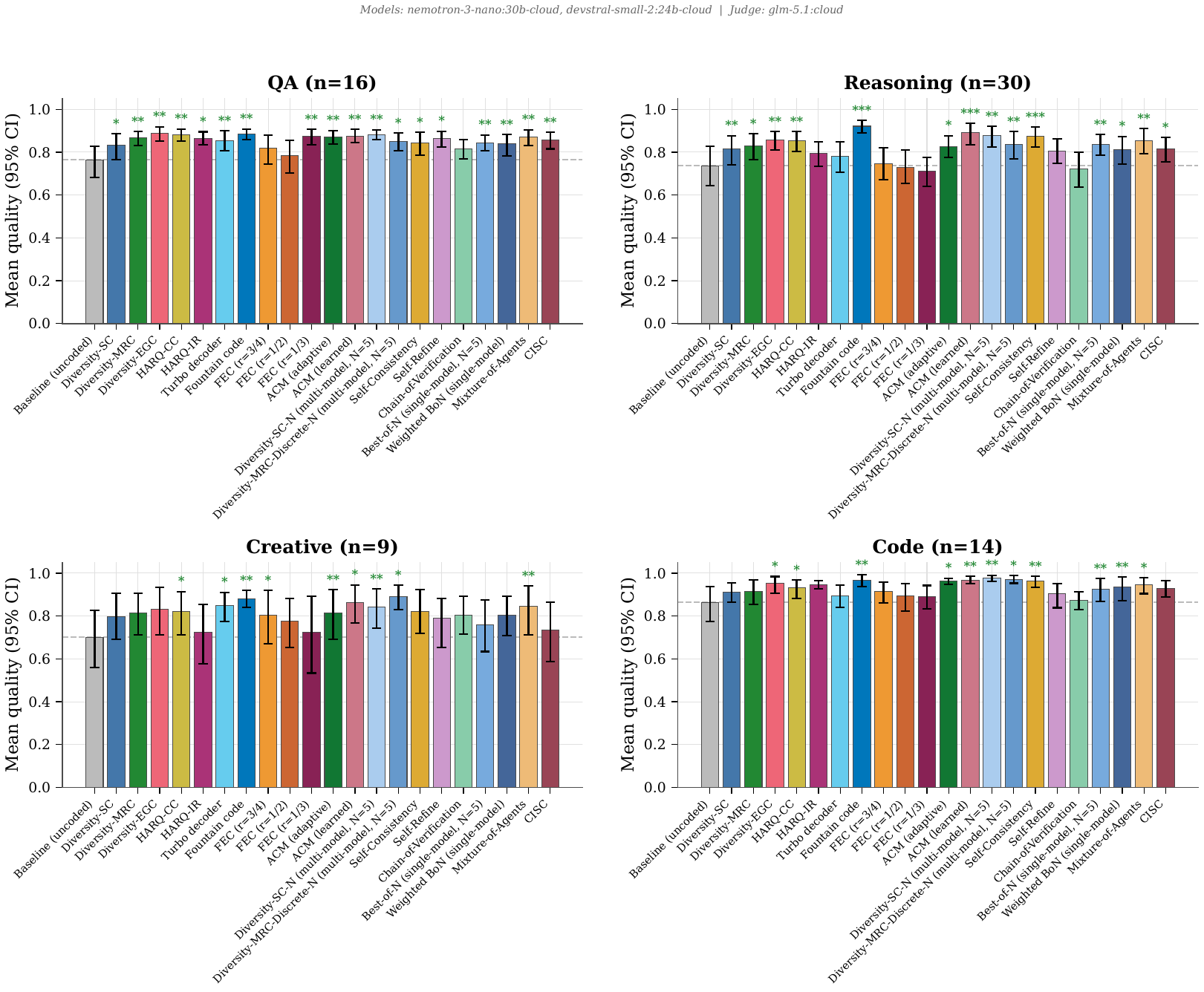}
  \caption{Ollama-cloud trio on the curated split: per-category mean quality (panels as in \Cref{fig:gallery-14b-cat}).}
  \label{fig:gallery-ollama69-cat}
\end{figure}

\subsection{Ollama-cloud trio on the 300-task hard-benchmark split (MMLU/GSM8K/HumanEval)}
\label{app:gallery-ollama300}

This subsection covers the headline 300-task hard-benchmark split. \Cref{fig:gallery-ollama300-qvc} plots quality versus cost on the same axes as the curated subsections; \Cref{fig:gallery-ollama300-bars} reports per-dataset mean quality (one panel per dataset); \Cref{fig:gallery-ollama300-pareto} reports the per-dataset quality-cost Pareto scatters; \Cref{fig:gallery-ollama300-acm} compares the hand-coded router with the oracle on the single-generation run of this configuration; and \Cref{fig:gallery-ollama300-iter} reports the iterative-refinement diagnostics.

\begin{figure}[h]
  \centering
  \includegraphics[width=\linewidth]{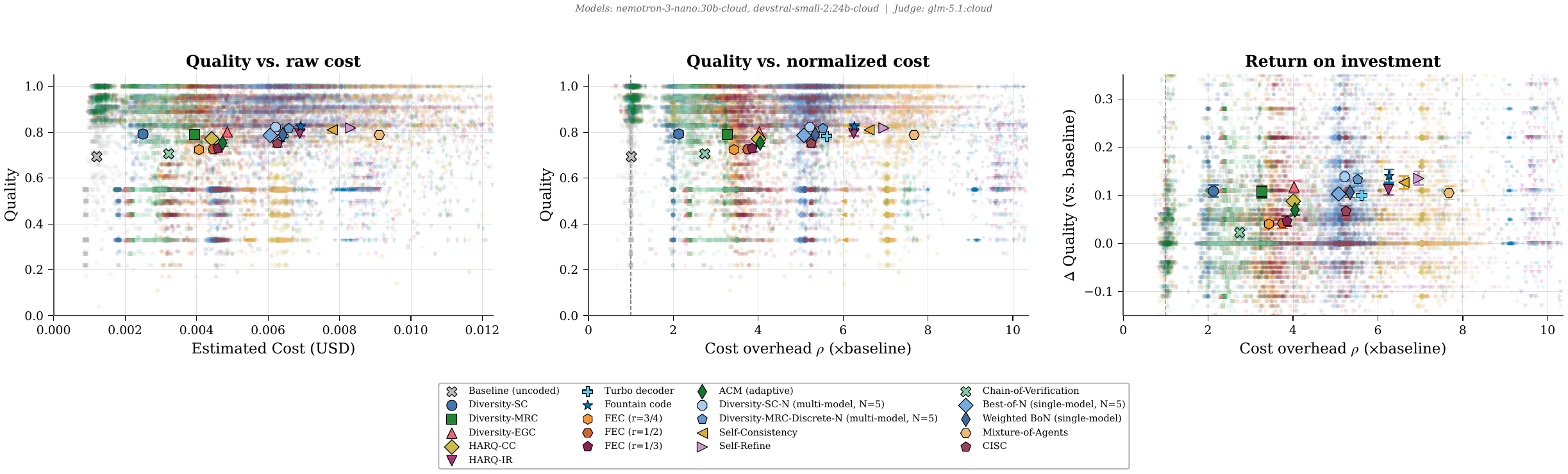}
  \caption{Ollama-cloud trio on the 300-task hard-benchmark split: quality versus cost (panels as in \Cref{fig:gallery-14b-qvc}). This is the configuration on which the headline cost-aware semantic-nearest-neighbor router (\Cref{tab:semknn-pareto}) is evaluated.}
  \label{fig:gallery-ollama300-qvc}
\end{figure}

\begin{figure}[h]
  \centering
  \includegraphics[width=\linewidth]{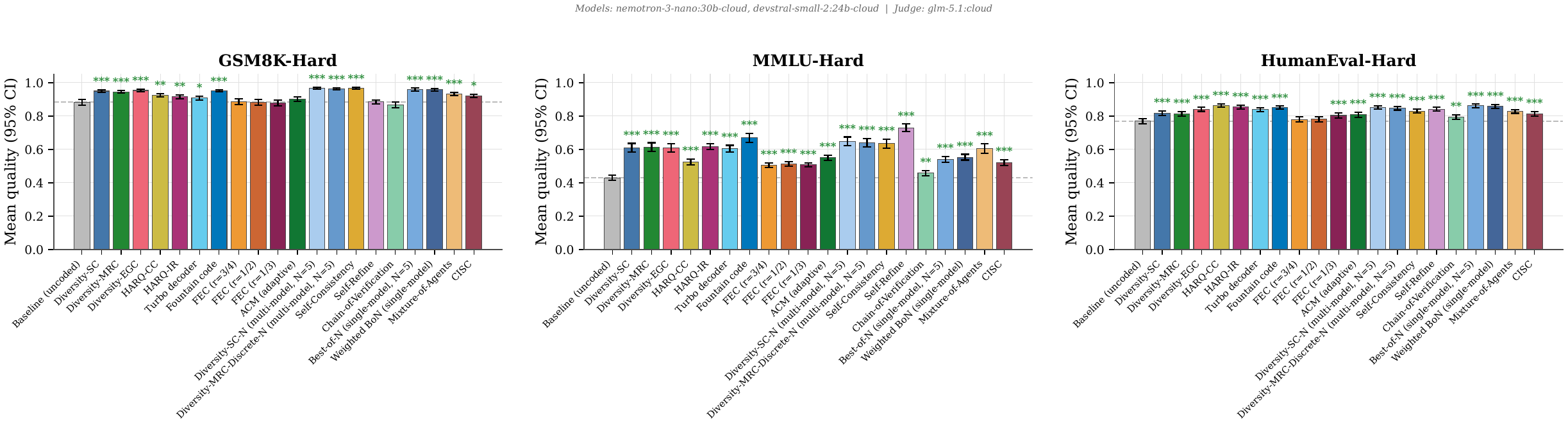}
  \caption{Per-dataset mean quality on the 300-task hard-benchmark split, one panel per dataset (GSM8K hard, MMLU hard, HumanEval hard). Dashed line is the per-dataset baseline mean; significance stars are paired Wilcoxon vs.\ baseline.}
  \label{fig:gallery-ollama300-bars}
\end{figure}

\begin{figure}[h]
  \centering
  \includegraphics[width=\linewidth]{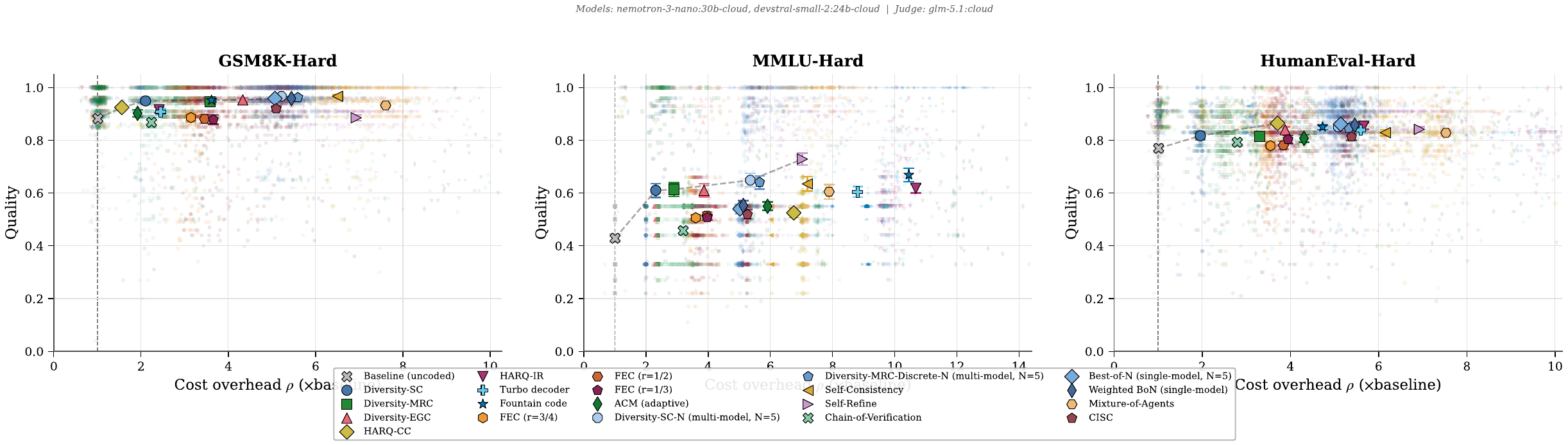}
  \caption{Per-dataset quality versus normalized cost overhead $\rho$ on the 300-task hard-benchmark split. Dashed line is the empirical Pareto frontier; the ACM marker is the hand-coded router (same data as \Cref{fig:hard-benchmark-pareto}).}
  \label{fig:gallery-ollama300-pareto}
\end{figure}

\begin{figure}[h]
  \centering
  \includegraphics[width=\linewidth]{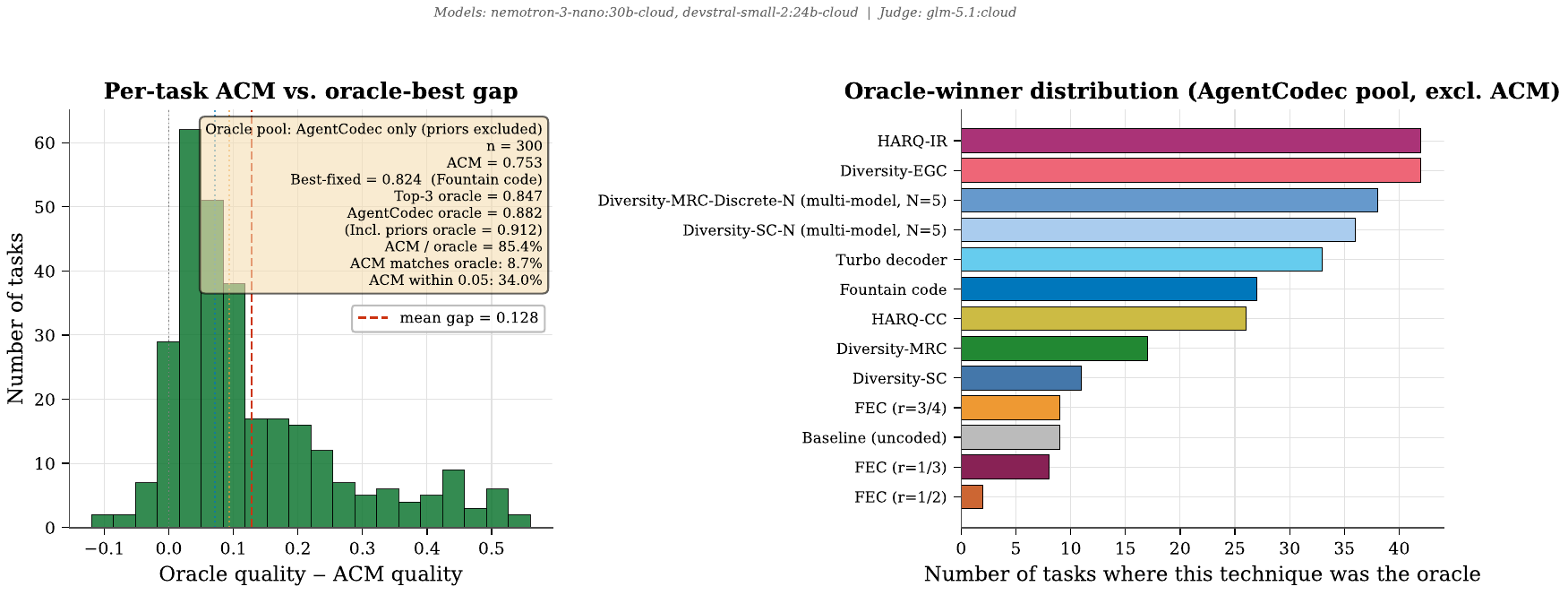}
  \caption{Ollama-cloud trio on the 300-task hard-benchmark split (single generation): hand-coded \textsc{ACM} router versus the per-task oracle (panels as in \Cref{fig:gallery-14b-acm}). The semantic-nearest-neighbor router's decomposition on the three-repeat split is in \Cref{sec:exp-acm-decomp}.}
  \label{fig:gallery-ollama300-acm}
\end{figure}

\begin{figure}[h]
  \centering
  \begin{subfigure}[t]{0.49\linewidth}
    \includegraphics[width=\linewidth]{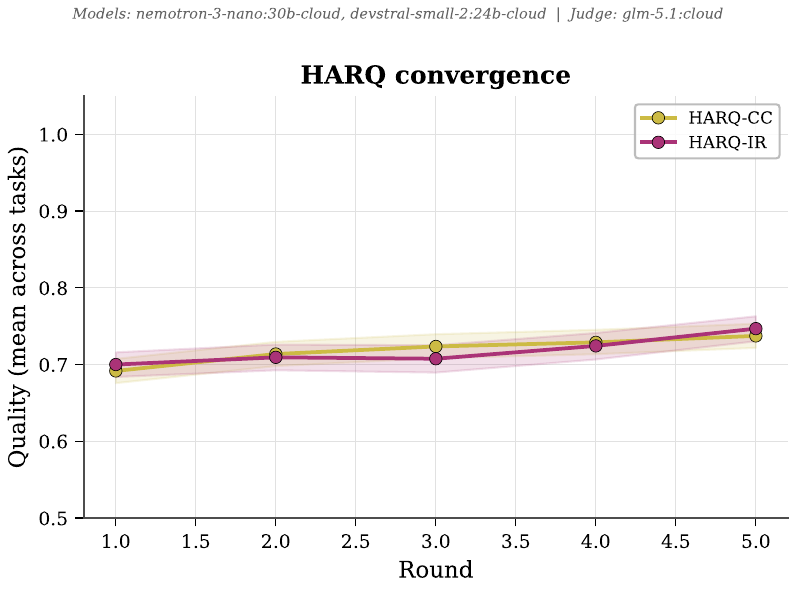}
    \caption{HARQ convergence.}
  \end{subfigure}\hfill
  \begin{subfigure}[t]{0.49\linewidth}
    \includegraphics[width=\linewidth]{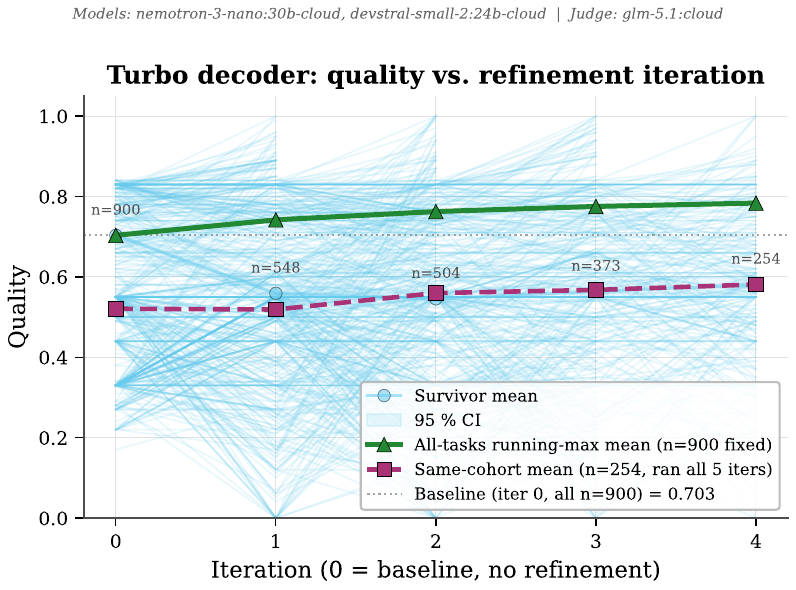}
    \caption{Turbo iteration trajectory.}
  \end{subfigure}
  \caption{Ollama-cloud trio on the 300-task hard-benchmark split: iterative-refinement diagnostics.}
  \label{fig:gallery-ollama300-iter}
\end{figure}

%%%%%%%%%%%%%%%%%%%%%%%%%%%%%%%%%%%%%%%%%%%%%%%%%%%%%%%%%%%%
\clearpage
\newpage

\setcounter{page}{1}
\renewcommand{\thepage}{\arabic{page}}

\newpage
\section*{NeurIPS Paper Checklist}

\begin{enumerate}

\item {\bf Claims}
    \item[] Question: Do the main claims made in the abstract and introduction accurately reflect the paper's contributions and scope?
    \item[] Answer: \answerYes{}
    \item[] Justification: The abstract and introduction (\Cref{sec:intro}) state four contributions, each supported by a section: the communication-theoretic framework with its prior-art mapping (\Cref{sec:model,sec:system,tab:operator-view,tab:baseline-reduction-app}); two analytical results in support of the framework, stated with what each predicts and what it only explains (\Cref{sec:analytical-predictions}); the \agentcodec{} operators and their evaluation across six channel configurations, including judge-free grading of $6{,}900$ outputs and a reasoning-enabled ablation (\Cref{sec:experiments,sec:discussion}); and the cost-aware router, with the gains stated in the abstract (\Cref{sec:router,sec:exp-acm}). Claims are stated at the strength the data support: the refinement result is graded rather than a dichotomy, the router's advantage under judge-free grading holds at most cost settings and depends on the candidate pool, and small quality margins on single-run columns are treated as ties (\Cref{app:noise}).

\item {\bf Limitations}
    \item[] Question: Does the paper discuss the limitations of the work performed by the authors?
    \item[] Answer: \answerYes{}
    \item[] Justification: \Cref{sec:discussion} has a dedicated limitations paragraph (judge dependence and no human evaluation, first-order and partly post hoc analysis, the router's dependence on cache, task distribution, judge and pool and its closeness to algorithm selection, single-run columns, judge cost), and \Cref{app:limitations} gives the full discussion, including non-stationarity and curated-task selection. Computational scaling is discussed in \Cref{app:exp-methodology}.

\item {\bf Theory assumptions and proofs}
    \item[] Question: For each theoretical result, does the paper provide the full set of assumptions and a complete (and correct) proof?
    \item[] Answer: \answerYes{}
    \item[] Justification: Both propositions are stated with their assumptions and proved in the appendix (\Cref{prop:csi-crossover} in \Cref{app:mrc-egc-crossover}, under a linear combiner with Gaussian weighting errors and a first-order ratio-of-expectations approximation stated as such; \Cref{prop:refinement-threshold} in \Cref{app:formal-foundations}, with attribution to the classical results it specializes). The MRC weighting (\Cref{prop:mrc}) is flagged as a heuristic, and the model-difficulty duality (\Cref{cor:duality}) is marked as an organizing observation rather than a result.

    \item {\bf Experimental result reproducibility}
    \item[] Question: Does the paper fully disclose all the information needed to reproduce the main experimental results of the paper to the extent that it affects the main claims and/or conclusions of the paper (regardless of whether the code and data are provided or not)?
    \item[] Answer: \answerYes{}
    \item[] Justification: The code is not released (see the next item), so the paper and its supplementary material carry everything needed to reproduce the results: the 69 curated tasks are released as supplementary material, and \Cref{app:data} specifies how to rebuild the 300-item split, item for item and in order, from public Hugging Face datasets, together with the judge-free grader; \Cref{tab:hyperparams} and \Cref{app:impl-details} specify every technique, including prompts' structure, thresholds, budgets, and guards; \Cref{app:exp-methodology} and \Cref{tab:judge-rubric} specify the judge; \Cref{tab:model-configs,app:extra-configs} list all models; and \Cref{sec:router,app:router-extended} specify the router and its evaluation protocol.

\item {\bf Open access to data and code}
    \item[] Question: Does the paper provide open access to the data and code, with sufficient instructions to faithfully reproduce the main experimental results, as described in supplemental material?
    \item[] Answer: \answerNo{}
    \item[] Justification: The data are open: the 69 curated tasks are released as supplementary material under the CC BY 4.0 license, and \Cref{app:data} gives a deterministic procedure that rebuilds the 300-item split from public datasets. During review, an anonymized version of the code was provided to the reviewers so that they could verify experimental reproducibility. Because of proprietary company policy, the code cannot be made publicly available: the techniques are implemented in a commercial platform of the authors' company (disclosed in the acknowledgments). To ensure reproducibility without the code, the paper gives step-by-step descriptions of every technique and its guards (\Cref{app:impl-details}), pseudocode for the diversity combiner (\Cref{alg:diversity}), all hyperparameters (\Cref{tab:hyperparams}), the judge's full criteria and weights (\Cref{tab:judge-rubric}), and the evaluation data above. All channel and judge models are publicly accessible.

\item {\bf Experimental setting/details}
    \item[] Question: Does the paper specify all the training and test details (e.g., data splits, hyperparameters, how they were chosen, type of optimizer) necessary to understand the results?
    \item[] Answer: \answerYes{}
    \item[] Justification: \Cref{sec:experiments} describes the task sets and which claim each supports, the six configurations (\Cref{tab:model-configs}), scoring, and statistics; \Cref{tab:hyperparams} gives all hyperparameters. No model is fine-tuned. The only fitted components are the routing policies, fit per fold under 5-fold stratified cross-validation (\Cref{sec:exp-acm-decomp,app:router-extended}); the fold construction is in \Cref{app:data-hard}.

\item {\bf Experiment statistical significance}
    \item[] Question: Does the paper report error bars suitably and correctly defined or other appropriate information about the statistical significance of the experiments?
    \item[] Answer: \answerYes{}
    \item[] Justification: Means carry $95\%$ bootstrap confidence intervals over tasks, and comparisons use paired Wilcoxon signed-rank tests on per-task differences with repeats averaged per task (\Cref{sec:experiments}). The 300-task split and both 3B configurations have three generation runs, and on the 300-task split they are propagated into the router's intervals by seed-aware bootstrap; single-run noise is measured directly and its resolution stated (\Cref{app:noise}), and single-run columns are marked in \Cref{tab:results-summary}.

\item {\bf Experiments compute resources}
    \item[] Question: For each experiment, does the paper provide sufficient information on the computer resources (type of compute workers, memory, time of execution) needed to reproduce the experiments?
    \item[] Answer: \answerYes{}
    \item[] Justification: Local configurations ran on one workstation (Intel Core i9-13900K, one NVIDIA RTX~3090 with 24\,GB, 32\,GB RAM); cloud configurations used provider APIs (\Cref{app:exp-methodology}). Per-task costs of every technique are reported (\Cref{tab:results-summary,tab:semknn-pareto}). The project also ran validation runs beyond the main configurations, listed in \Cref{app:extra-configs}, and preliminary runs that are not reported.

\item {\bf Code of ethics}
    \item[] Question: Does the research conducted in the paper conform, in every respect, with the NeurIPS Code of Ethics \url{https://neurips.cc/public/EthicsGuidelines}?
    \item[] Answer: \answerYes{}
    \item[] Justification: The work studies reliability techniques for LLM agents on public benchmarks and author-written tasks with no personal data and no human subjects. Open-weight and commercial models were used through their official interfaces under their terms of use.

\item {\bf Broader impacts}
    \item[] Question: Does the paper discuss both potential positive societal impacts and negative societal impacts of the work performed?
    \item[] Answer: \answerYes{}
    \item[] Justification: \Cref{app:broader-impacts} discusses positive impacts (more reliable and cheaper agent pipelines) and risks (over-trust in judge-scored reliability, which our judge validation shows can mislead between near-tied techniques, and the general dual use of more capable agents), with mitigations: validating the scorer against verifiable outcomes, which \agentcodec{}'s exact and blended scoring modes support, and relying on the safety measures of the underlying models.

\item {\bf Safeguards}
    \item[] Question: Does the paper describe safeguards that have been put in place for responsible release of data or models that have a high risk for misuse (e.g., pre-trained language models, image generators, or scraped datasets)?
    \item[] Answer: \answerNA{}
    \item[] Justification: The only released artifact is the set of 69 author-written benchmark tasks (prompts, references, and evaluation criteria), which poses no meaningful risk of misuse; no models or scraped data are released.

\item {\bf Licenses for existing assets}
    \item[] Question: Are the creators or original owners of assets (e.g., code, data, models), used in the paper, properly credited and are the license and terms of use explicitly mentioned and properly respected?
    \item[] Answer: \answerYes{}
    \item[] Justification: MMLU~\citep{hendrycks2021mmlu}, GSM8K~\citep{cobbe2021verifiers}, and HumanEval~\citep{chen2021humaneval} are cited and used from their Hugging Face releases under the MIT license (\Cref{app:data}). All channel and judge models are named in \Cref{tab:model-configs,app:extra-configs} and used under the license or terms of service published by their providers; the sentence encoder is cited~\citep{xiao2024cpack}.

\item {\bf New assets}
    \item[] Question: Are new assets introduced in the paper well documented and is the documentation provided alongside the assets?
    \item[] Answer: \answerYes{}
    \item[] Justification: The 69 curated tasks are released as supplementary material under the CC BY 4.0 license, with a README that documents how they were selected, their composition, and their provenance; \Cref{app:data-curated} describes them.

\item {\bf Crowdsourcing and research with human subjects}
    \item[] Question: For crowdsourcing experiments and research with human subjects, does the paper include the full text of instructions given to participants and screenshots, if applicable, as well as details about compensation (if any)?
    \item[] Answer: \answerNA{}
    \item[] Justification: No crowdsourcing or human-subjects research is involved; all evaluation is automated (LLM judge and judge-free grading).

\item {\bf Institutional review board (IRB) approvals or equivalent for research with human subjects}
    \item[] Question: Does the paper describe potential risks incurred by study participants, whether such risks were disclosed to the subjects, and whether Institutional Review Board (IRB) approvals (or an equivalent approval/review based on the requirements of your country or institution) were obtained?
    \item[] Answer: \answerNA{}
    \item[] Justification: No human-subjects research is involved.

\item {\bf Declaration of LLM usage}
    \item[] Question: Does the paper describe the usage of LLMs if it is an important, original, or non-standard component of the core methods in this research? Note that if the LLM is used only for writing, editing, or formatting purposes and does \emph{not} impact the core methodology, scientific rigor, or originality of the research, declaration is not required.
    \item[] Answer: \answerYes{}
    \item[] Justification: LLMs are the object of study and are used in three roles: as channels that generate outputs, as the judge (channel estimator) that scores them, and as synthesizer, critic, or voter inside techniques. All models and their roles are listed in \Cref{tab:model-configs,app:extra-configs}, and the judge-validation study adds DeepSeek-V4-Flash and Gemma-3 12B as second judges (\Cref{app:judge-validation}).

\end{enumerate}

\end{document}